\documentclass[twoside,11pt]{article}

\usepackage{jmlr2e} 

\usepackage[utf8]{inputenc} 
\usepackage[T1]{fontenc}    
\usepackage{hyperref}       
\usepackage{url}            
\usepackage{booktabs}       
\usepackage{nicefrac}       
\usepackage{microtype}      
\usepackage{lipsum}

\usepackage{verbatim}
\usepackage{url}
\usepackage{hyphenat}

\usepackage{latexsym}
\usepackage{array}
\usepackage{tocbibind}
\usepackage{listings}
\usepackage[retainorgcmds]{IEEEtrantools}

\usepackage{algorithmic}
\usepackage{algorithm2e}
\usepackage{color}
\usepackage{textcomp}

\usepackage{physics}
\usepackage{bm}
\usepackage{float}
\usepackage{scalerel}
\usepackage{tikz}
\usetikzlibrary{positioning}
\usetikzlibrary{arrows}
\usetikzlibrary{arrows.meta}
\usetikzlibrary{fit}
\usetikzlibrary{backgrounds}
\usetikzlibrary{matrix}
\usetikzlibrary{shapes.multipart}
\usetikzlibrary{decorations.pathreplacing}
\usetikzlibrary{calc}

\def\kenn{KENN}
\def\Reals{\mathbb{R}}
\def\Naturals{\mathbb{N}}
\def\boost{TBF}

\def\bdelta{\bm\delta}
\def\clauses{\mathcal{K}}
\def\L{\mathcal{L}}

\def\C{\mathcal{C}}
\def\P{\mathcal{P}}

\def\I{\mathcal{I}}
\def\G{\mathcal{G}}
\DeclareMathOperator*{\argmax}{\arg\!max}

\DeclareMathOperator*{\softmax}{softmax}
\def\lit{t}
\def\blit{\mathbf{t}}

\def\bv{\mathbf{v}}

\def\bx{\mathbf{x}}
\def\by{\mathbf{y}}
\def\bz{\mathbf{z}}

\def\U{\mathbf{U}}
\def\B{\mathbf{B}}
\def\M{\mathbf{M}}

\newenvironment{TODO LIST}[1]{\color{red} #1:\begin{itemize}}{\end{itemize}}
\newcommand{\DONE}[1]{#1}

\DeclareFixedFont{\ttb}{T1}{txtt}{bx}{n}{9} 
\DeclareFixedFont{\ttm}{T1}{txtt}{m}{n}{9}  
 
\definecolor{deepblue}{rgb}{0,0,0.5}
\definecolor{deepred}{rgb}{0.6,0,0}
\definecolor{deepgreen}{rgb}{0,0.5,0}

\newcommand\pythonstyle{\lstset{
  language=Python,
  backgroundcolor=\color{white}, 
  basicstyle=\ttm,
  otherkeywords={self},            
  keywordstyle=\ttb\color{deepblue},
  emph={MyClass,__init__},          
  emphstyle=\ttb\color{deepred},    
  stringstyle=\color{deepgreen},
  commentstyle=\color{red},  
  frame=tb,                         
  showstringspaces=false            
}}

\lstnewenvironment{python}[1][]
{
\pythonstyle
\lstset{#1}
}
{}

\title{Neural Networks Enhancement\\with Logical Knowledge}

\author{\name Alessandro Daniele \email daniele@fbk.eu \\
       \addr Fondazione Bruno Kessler\\ 
       Trento, Italy\\
       \AND
       \name Luciano Serafini \email serafini@fbk.eu \\
       \addr Fondazione Bruno Kessler\\ 
       Trento, Italy }

\begin{document}

\keywords{Neural-Symbolic Integration, Neural Networks, Fuzzy Logic, Relational Learning, Collective Classification}

\editor{}

\maketitle

\begin{abstract}
  
  In the recent past, there has been a growing interest in Neural-Symbolic Integration frameworks, i.e., hybrid systems that integrate connectionist and symbolic approaches to obtain the best of both worlds. In a previous work, we proposed KENN
  (Knowledge Enhanced Neural Networks), a Neural-Symbolic architecture that injects prior logical knowledge into a neural network by adding a new final layer which modifies the initial predictions accordingly to the knowledge. Among the advantages of this strategy, there is the inclusion of clause weights, learnable parameters that represent the strength of the clauses, meaning that the model can learn the impact of each clause on the final predictions. As a special case, if the training data contradicts a constraint, KENN learns to ignore it, making the system robust to the presence of wrong knowledge. In this paper, we propose an extension of KENN for relational data. To evaluate this new extension, we tested it with different learning configurations on Citeseer, a standard dataset for Collective Classification. The results show that KENN is capable of increasing the performances of the underlying neural network even in the presence relational data, outperforming other two notable methods that combine learning with logic.

\end{abstract}

\section{Introduction}
\label{sec:introduction}


In the last decade, Deep Learning approaches gained a lot of interest in the AI community, becoming the state of the art on many fields, such as Computer Vision~\citep{image_classification}, Machine Translation~\citep{machine_translation}, Speech Recognition~\citep{speech_recognition} and so forth. However, the main downside of such methods is that they are demanding in terms of training data.

On the other hand, human beings are capable of learning new concepts with few examples (Few Shot Learning) and even in some cases with zero experience (Zero Shot Learning). One of the main reasons is due to their ability to make use of previously acquired knowledge. In other words, the human brain is not just a randomly initialized model that learns from data, on the contrary, it often contains some sort of prior knowledge when approaching a new task.

Such knowledge could come for experience on different tasks. For instance, when learning to move around and avoid obstacles, humans can effectively learn the three-dimensional structure of the world by constructing an internal representation of it. When approaching a new task where visual data must be processed, they can make use of the previously learned representation, transferring the acquired knowledge into a new domain.

Another type of knowledge is the one provided by other human beings. For example, suppose we tell a kid: ``A unicorn is a horse with a horn in its forehead''. Although the kid has no previous experience of unicorns, he will be able to recognize them, for example when watching a movie. This ability to exploit provided knowledge is crucial for effectively learn new concepts when training data is scarce.

In this paper, we are going to focus on this second type of knowledge. While in the case of humans such knowledge is often in the form of natural language definitions, in the proposed method it is provided as a set of logical formulas.

This work extends our previous research on \kenn (\emph{Knowledge Enhanced Neural Network}), a method to inject knowledge into models for multi-label classification \citep{kenn}. While multi-label classification is an important topic in machine learning, the usage of knowledge is particularly relevant in the context of relational domains.
For this reason, we propose an updated version of \kenn\ which can deal with these kinds of domains.

Suppose we have a Neural Network for a classification task, which is called \emph{base NN}, that takes as input the feature vector $\bx \in \Reals^n$ and returns an output $\by \in [0,1]^m$ which contains the predictions for $m$ classes. A background knowledge $\clauses$ is also provided. $\clauses$ is defined as a set of clauses, i.e. disjunctions of literals, that represent constraints on the $m$ classes to be predicted. For instance, in an image classification task, the clause $\forall x \ (\neg Dog(x) \vee Animal(x))$, stating that dogs are animals, can be used by \kenn, in conjunction with the base NN, to predict the two labels $Dog$ and $Animal$. 

 Figure~\ref{fig:KENN_overview} shows a high-level overview of \kenn.

\begin{figure}[h]
	\centering
	\definecolor{predicates}{rgb}{0.9,0.2,0.2}
\definecolor{clauses}{rgb}{0.3,0.3,0.7}
\definecolor{inputs}{rgb}{0.1,0.6,0.2}
\definecolor{outputs}{rgb}{0.9,0.2,0.2}

\tikzstyle{block} = [rectangle, draw, fill=green!8, text centered, rounded corners, minimum height=2cm, minimum width=3cm, align=center, font=\fontsize{16}{0}\selectfont]
\tikzstyle{block2} = [rectangle, draw, fill=green!8, text centered, rounded corners, minimum height=2.4cm, minimum width=4cm, align=center, font=\fontsize{8}{0}\selectfont, node distance=1cm]
\scalebox{0.6}{
		\begin{tikzpicture}[scale=0.2, node distance=5cm, align=left]
		
		
		\node[
		draw=none,
		ultra thick,
		rounded corners,
		minimum height=2cm,
		minimum width=3cm,
		black!35,
		fill=black!5] (FI) at (-8,0){};
		\node[text centered, font=\fontsize{16}{0}\selectfont, align=center] at (-8, 0) {NN};

		\node[
		draw=none,
		ultra thick,
		rounded corners,
		minimum height=2cm,
		minimum width=3cm,
		inputs,
		fill=green!15] (LE) at (24,0){};
		\node[text centered, font=\fontsize{16}{0}\selectfont, align=center] at (24, 0) {KE};
		

			\node[font=\fontsize{22}{0}\selectfont] (x) at (-23,0) {$\bx$};
			
			\node[font=\fontsize{18}{0}\selectfont] (k) at (8,-12) {$\clauses$};
			
			\path[draw, ->] (x) -- (FI);
			
			\path[draw, ->] (FI) -- (LE)node[midway, fill=white, font=\fontsize{22}{0}\selectfont] {$\by$};
			
			\node[align=left, font=\fontsize{22}{0}\selectfont] (y) at (39,0) {$\by'$};
			\path[draw, ->, dashed] (k)  .. controls (18, -11) and (22, -8.5) ..  (LE);
			\path[draw, ->] (LE) -- (y);
		\end{tikzpicture}}

%
%
%
%
%
	\caption{KENN model: features are given as input to a neural network (NN) and predictions on predicates values are returned. Knowledge Enhancer modifies the predictions based on logical constraints ($\clauses$)}
	\label{fig:KENN_overview}
\end{figure}
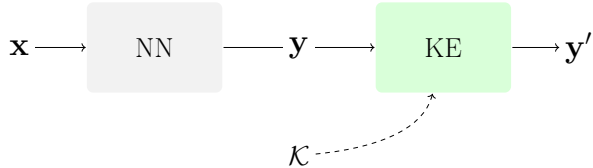

In KENN, the predictions $\by$ of the base NN are revised by a differentiable function, called \emph{Knowledge Enhancer} (KE), which updates $\by$ into $\by'$ to increase the truth value of each clause $c \in \clauses$. Since both base NN and KE are differentiable, the entire architecture is differentiable end-to-end, making it possible to apply back-propagation algorithm directly on the whole model. 

KE increases the satisfaction of each clause separately, obtaining $|\clauses|$ different vectors, each of which represents the changes to be applied on $\by$ to improve the satisfaction of a specific clause $c$. These changes are combined linearly to obtain the final change to be applied on the base NN's predictions.  
 More in details, for each clause $c\in\clauses$, the KE computes a soft differentiable approximation of a function called \emph{t-conorm boost function} (TBF). Intuitively, a TBF is a function \mbox{$\delta: [0,1]^k \to [0,1]^k$} that proposes the changes to be applied on a set of $k$ truth values to increase the value of the \mbox{t-conorm} applied on them: $\bot(\blit + \delta(\blit)) \geq \bot(\blit)$. Moreover, the KE contains additional learnable parameters that can be learned as well: for each clause $c \in \clauses$ there is an associated \emph{clause weight} $w_c$  which determines the strength of $c$, i.e., it defines the influence that the clause has on the final predictions.  Differently from other Neural-Symbolic integration approaches, clauses weights are not given, but they are learned. 
Moreover, by assigning zero to clause weights, \kenn\ can learn to ignore clauses in the Prior Knowledge that are not fully satisfied inside training data. In Section~\ref{sec:scatter} we will further analyze this ability of \kenn\ to learn the clause weights by inspecting their value after training.

In our previous work~\citep{kenn}, we tested \kenn\ on 
the Predicate Detection task of Visual Relationship Detection Dataset (VRD Dataset)~\citep{visual2} using a manually curated prior knowledge proposed by~\citep{IvanThesis}, outperforming previous state of the art results, with really good performances on the \emph{Zero Shot Learning} sub-task.
Moreover, \kenn\ outperformed Logic Tensor Networks, one of its major competitors, using the same knowledge.

However, notice that previous experimets have been applied only on multilabel classification problems with no relational data. Here we provide new experiments applied to test the efficacy of \kenn\ even in the context of relational domains. For this purpose, we applied \kenn\ on Citeseer~\citep{citeseer}, a standard dataset for Collective Classification~\citep{collective_classification}. The experiments on this dataset are particularly relevant since they provide some insight on the usage of \kenn\ on relational domains and they give us a comparison between \kenn\ and two other approaches: Semantic Based Regularization (SBR)~\citep{SBR} and Relational Neural Machines (RNM)~\citep{RNM}. The evaluation was conducted with two different learning paradigms of Collective Classification, namely Inductive and Transductive Learning. Moreover, the experiments were applied to different splits of the dataset to evaluate the three methods performances at the varying of the number of training samples.




\section{State of the Art}
\label{ch:rel_works}

Traditional supervised learning methods assume samples to be independent and identically distributed. However, in relational domains, the independence assumption is often violated since the relations encode the dependencies among different domain elements. 

For instance, we could be interested in finding out whether a person has smoking habits. In a classical machine learning approach, we could use only persons' features, such as age, geographical location, and so forth. However, if we know that his friends smoke, there could be a higher chance that he smokes as well. More generally, relational data can be described
as a graph composed of a set of nodes and edges, which represent the relations between nodes.

Exploiting relations between different entities has been the main focus of Statistical Relational Learning (SRL), a subfield of Machine Learning which aims at applying statistical methods in domains that exhibit both uncertainty and relational structure \citep{introduction_SRL}.





Central to SRL is the integration of logical knowledge in the learning framework. For instance, the previously mentioned friends-smoke relation could be provided by a human being in the form of a First Order Logic (FOL) rule:
$$
\forall x \ \forall y \ Smoker(x) \land Friends(x,y) \to Smoker(y)
$$
This kind of knowledge could act as a sort of supervision for the learning process. 
Additionally, the usage of knowledge could help to represent relational data inside a neural network, which is not a straightforward task since neural networks inputs are represented using tabular data (they can be vectors, matrices or in general tensors) \citep{NN_relational}.

Many previous works attempt to combine learning models with logical knowledge. Among them, there is Markov Logic Networks~\citep{MLN}, which uses weighted FOL rules as a template for building Markov Random Fields \citep{MRF}. The network represents grounded atoms as random variables and it defines a joint distribution over possible worlds, i.e. the possible assignments to all the grounded atoms. 
Such distribution is calculated as the product of a set of potential functions, each of which is derived from a logical formula in the knowledge base. 
In practice, worlds that satisfy more the knowledge (according to the weights of the rules) are associated with higher probability. A similar framework is \emph{Probabilistic Soft Logic} (PSL)~\citep{PSL} which uses a continuous relaxation of the variables to gain efficiency. 

Neither MLN, nor PSL can deal with real values features. To obviate at this limitation, an extension of MLN, called Hybrid Markov Logic Networks (HMLN), has been introduced in \citep{HMLN}. HMLN differs from MLN for its ability to represent continuous variables.

Generally speaking, SRL deals with the knowledge either by combining logic rules with probabilistic graphical models (as in the previously mentioned work) or by extending logic programming languages to handle uncertainty, like in the case of ProbLog \citep{problog}.


The recent achievements of Deep Learning methods lead to a renewed interest in another line of research, called Neural-Symbolic Integration, which focuses on combining neural network architectures with logical knowledge~\citep{neural_symbolic}.

Neural Networks are perfectly suited for pattern recognition, even in the presence of noisy data. They are particularly good at mapping low-level perceptions to more abstract concepts (for instance, going from images to classes). However, it is hard for a NN to reason with this high-level abstractions. Furthermore, NNs are demanding in terms of training data. On the other hand, pure logical approaches are not efficient at learning from low-level features and they struggle in the presence of noise. Nevertheless, they perform well on reasoning with highly abstract concepts and on learning from a small number of samples. Given these opposite strengths and weaknesses, it is not a surprise that a lot of interest has been drawn towards Neural-Symbolic systems. Indeed, the goal is to combine these two paradigms to obtain the best of the two worlds.

In this Section, we are going to introduce different works on Neural-Symbolic Integration, classifying them based on the type of tasks they can be applied to. According to this view, we can classify Neural-Symbolic Systems in three macro areas, corresponding on different objectives:
\begin{description}
    \item[Differentiable Reasoning (DR):] in this category the focus is on the development of differentiable approaches for reasoning, which can be used inside a neural network;
    \item[Inductive Logic Programming (ILP):] in this case, the goal is to learn Knowledge from the data, either from scratch or by refining an initial given Knowledge; 
    \item[Knowledge Guided Learning (KGL):] here the focus is on Learning and the Knowledge acts as additional supervision for the learning process. 
\end{description}

While some methods fall within multiple categories, the majority of them are focused only on one of the three aspects. We believe that a direct comparison between methods in different categories should be avoided since the tasks they can be used for are essentially different.
For this reason, in the next sections, we will introduce the various approaches dividing them accordingly with this categorization, with a major focus on the last class of methods (Knowledge Guided Learning), which is the topic of this paper.

\subsection{Differentiable Reasoning methods}
Deductive reasoning can be defined as the process of producing logical consequents from initial knowledge. DR methods aim at building architectures that integrate deductive reasoning with neural network models by defining inference rules in terms of differentiable functions.

Early proposals, like Knowledge-Based Artificial Neural Networks (KBANN), codify the logical knowledge into the weights of a neural network~\citep{KBANN}. 
In KBANN~\citep{KBANN}, background knowledge consists of a set of logical implications. Such implications are reorganized to obtain a hierarchical structure, where higher-level propositions can be inferred by lower-level ones. From this hierarchy, an initial Neural Network is created. Hidden layers are then added together with missing edges, which receive initial weights close to zero. 

A similar approach, called C-IL$^2$P is proposed in~\citep{CILP}. In this case, the underlying Neural Network is recursive. Both KBANN and C-IL$^2$P can refine the initial background knowledge based on the training samples. However, they are both restricted to propositional logic. This is a strong limitation, in particular in relational contexts since the relations can not be efficiently represented as propositions. To this end, CILP++ was proposed~\citep{CILP++}. CILP++ starts from a First Order Logic Knowledge base and, through the usage of propositionalization, generates a C-IL$^2$P network.

More recent approaches consist of embedding logical knowledge into real-valued tensors. A first attempt in this direction was proposed in \citep{low_dimensional}, where domain elements are represented as real-valued vectors and predicates as matrices. The method deals with simple FOL implications rules of the form $A \to B$, where $A$ and $B$ are atomic formulas. An example coming from the paper is the rule $\forall x \ \forall y \ profAt(x, y) \to worksFor(x, y)$.
The implication is represented as a $2 \cross 2 \cross 2$ tensor. Since the entire architecture is differentiable, one can learn the embeddings of predicates and terms from an incomplete knowledge using back-propagation in order to maximize the satisfaction of the initial background knowledge, and then use the learned embeddings to predict truth values for the unknown facts.

Another technique that relies on tensorization of the knowledge and maximization of the satisfiability is Logic Tensor Networks (LTN) \citep{LTN}. This time, the satisfaction of logical rules is calculated in a fuzzy logic semantic. Although LTN was mainly used in the KGL context (learning in the presence of constraints), it has been shown that it is indeed capable of tackling reasoning tasks \citep{LTN_reasoning}.
Another method that defines reasoning as a maximum satisfiability problem is SATNET \citep{satnet}, which is a differentiable solver for a smooth relaxation of MAXSAT.

A further approach involves defining continuous and differentiable relaxations for standard inference techniques used in logic programming \citep{garcez_nesy}. There are two main techniques derived from pure symbolic approaches that are used in the context of Neural-Symbolic Integration: forward and backward chaining.

In both cases, the knowledge is a logic program.
In forward chaining, the inference is carried out by matching known facts with the body of an implication rule to derive new facts (corresponding to the head of the rule). In backward chaining the order is inverted: starting from a goal (an atomic formula that has to be proved), the inference goes down in reverse by matching the goal with the head of a rule. The literals in the body of such rule become new goals and the process is applied again recursively.

Neural Logic Machines are based on forward chaining \citep{NLM}, while TensorLog \citep{tensorlog}, Neural Theorem Prover (NTP)~\citep{NTP1,NTP2}, and DeepProbLog~\citep{deepproblog} are based on backward chaining.



\subsection{Inductive Logic Programming methods}
Given a background knowledge $\clauses$ and a set of evidences $\mathcal{E}$, the goal of ILP is to find a hypothesis $\mathcal{H}$ such that $\clauses$ and $\mathcal{H}$ together entail $\mathcal{E}$. Traditional symbolic ILP methods are very efficient in learning from a small number of samples, in contrast with neural networks which are demanding in terms of training data. On the other hand, neural networks are robust to noisy data, as opposed to traditional ILP methods. Again, the two paradigms are complementary, and combining them could bring enormous advantages.

Some methods in the DR categories can also be applied in the ILP context. Among them, Neural Logic Programming (Neural LP) is a method that relies on tensorlog. Using an attention mechanism, Neural LP is able to softly select tensorlog operations \citep{tensorlog_ILP}, learning in this way new rules. Also NTP can be used in the ILP context: since the rules are represented through learnable embeddings, it is possible to add rules templates, i.e. rules with a predefined structure that contains learnable embeddings \citep{NTP2}. However, it has been shown that this kind of approach fails at learning non-trivial rules \citep{NTP_do_not}. The problem arises from the greedy choice made by NTP when selecting the proof (it selects the one with the highest score), which makes the final result strongly dependent on the initialization. To obviate this problem, in \citep{NTP_do_not} a new version of NTP is proposed. This new version has an increased ability to explore the space of proofs and, as a consequence, a reduced probability of finding a local minimum.

Other two methods that rely on rule templates are $\delta ILP$ \citep{dilp}, and \citep{rule_induction}. Different from NTP, they make use of forward chaining. Moreover, in \citep{dilp_reinforcment} it has been shown that $\delta ILP$ can be also applied to reinforcement learning tasks.



In \citep{NLRL}, Neural Logic Rule Layer (NLRL) is proposed. It consists of a NN layer that is capable, trough the usage of gates, of representing simple AND/OR rules. By stacking multiple layers, the model can represent any propositional logic rule.

\subsection{Knowldge Guided Learning methods}
\label{KGL}


In the case of KGL solutions, the focus is on augmenting the learning of a standard NN (typically in a Supervised Learning task) by providing some Prior Knowledge expressed in terms of logical formulas. 



There are mainly two approaches for learning in the presence of Prior Knowledge: the first consists of treating logical rules as constraints on the predictions of the Neural Network. The problem is reduced to maximize the satisfiability of the constraints and can be efficiently tackled by adding a regularization term in the Loss function which penalizes the unsatisfaction of the constraints. The second approach is to modify the neural network by injecting the knowledge into its structure.


\subsubsection{Regularization approaches}

Two notable examples of this approach are \emph{Logic Tensor Network} (LTN)~\citep{LTN} and \emph{Semantic Based Regularization} (SBR)~\citep{SBR}. Both methods maximize the satisfaction of the constraints, expressed as FOL formulas, under a fuzzy logic semantic. As mentioned above, LTN can be also classified as a DR approach. Although it was shown to be able to perform reasoning, the greater achievements of this model have been obtained in the KGL paradigm, in particular in the context of image interpretation \citep{IvanThesis,DonadelloSG17}. 
A similar strategy is used also by \emph{Semantic Loss Function}~\citep{semantic_loss}, but instead of relying on fuzzy logic, it optimizes the probability of the rules being true. Nevertheless, this approach is restricted to propositional logic. \citep{dl2} introduces DL2 which is also a method that includes logical knowledge through the Loss function. Nonetheless, it can be used only in the context of regression tasks, where the predicates correspond to comparison constraints (e.g. $=$, $\neq$, $\leq$) between terms that can be either outputs of the network or constants. \citep{adversarially} also proposes a method that regularizes the Loss, but they focus on a specific task of Natural Language Processing. Their approach differs from the others because it makes use of adversarial examples to calculate the regularization term: first, it generates samples that maximize the unsatisfaction of the constraints, then it optimizes the NN to increase their satisfaction. Finally, in \citep{semi_supervised} it is again proposed a regularization technique. This time, the task is Semi-Supervised Learning, and the regularization term is calculated from the unlabelled data.

In~\citep{harnessing}, a distillation mechanism is used to inject FOL rules: a student network is trained both on predicting the correct labels and emulating a teacher network; the teacher network is obtained by reducing the KL-Divergence with the student network and at the same time optimizing the satisfaction of the rules. While this approach seems different from the previous ones, we argue that it can be also considered as a method based on regularization, since the teacher network (which encodes the rules), is used to regularize the Loss applied on the student network.

Finally, in~\citep{old_dog}, a new technique is defined which adds the knowledge at the level of the labels by modifying them after each training step. The approach is quite general, but it does not provide any guarantees on the convergence. Again, even if this method is different from the previous ones, the effect on learning should be similar: in the previous cases, the Loss is combined with a regularization term to find a balance between fitting the correct labels and satisfying the knowledge; here, this is done indirectly by modifying the labels to increase the satisfaction of the knowledge.

\subsubsection{Model based approaches}

\kenn, the proposed approach in this paper, works pretty differently from the methods seen so far. The main difference lies in the way logic formulas are used: in \kenn\ they become part of the classifier instead of being enforced during training. Previously mentioned approaches based on regularization force the constraints satisfaction during training making the assumption that the knowledge is in general correct. Instead, here we assume there is a relationship between clauses and correct results, but this relationship is not known.
The logical constraints are seen as a prior belief rather than prior knowledge. More in detail, \kenn\ has internal learnable parameters associated with the logic formulas. 
This makes \kenn\ suitable for scenarios where the given knowledge contains errors or when rules are softly satisfied in the real world but it is not known in advance the extent on which they are correct. The differences between \kenn\ and methods that inject logical knowledge through the Loss function will be further analyzed in Section~\ref{sec:rel_work_2}.

Another approach that directly encodes the knowledge into the structure of the model is provided by Li and Srikumar who recently proposed a method that codifies the logical constraints directly into the neural network model~\citep{augmenting}. However, they restrict the rules to implications with exactly one consequent and they do not provide the possibility to learn clause weights, which in their system are added as hyper-parameters. 

Going in the same direction, \citep{RNM} proposed Relational Neural Networks (RNM). RNM can be also inserted in the set of approaches that add the logic directly into the model and, as the best of our knowledge, it is the only method other than \kenn\ which is capable of integrating logical knowledge with a neural network while learning the clause weights. However, it requires to solve an optimization problem at inference time and during each training step. The main advantage of \kenn\ over RNM is scalability, since \kenn\ does not need to perform optimization. A more detailed comparison between \kenn\ and RNM in Section~\ref{sec:rel_work_2}.

\newcounter{barcounter}

\newenvironment{mygraph}[3]{
    \setcounter{barcounter}{0}

    \def\posx{#1}
    \def\posy{#2}

    \foreach \i in {0.1,0.2,0.3,0.4,0.5,0.6,0.7,0.8,0.9,1.0}
    {
        \draw[ultra thin, fill=gray!0] (#1-0.1,#2+4*\i) -- (#1+6,#2+4*\i);
        \node[draw=none] at (#1-0.6,#2+4*\i) {\i};
    }

    \newcommand{\mybar}[2]{
        \node[
            draw=none,
            fit={(#1+0.2+\value{barcounter}*1.2,#2) (#1+\value{barcounter}*1.2+1.2,#2+##1*4)},
            inner sep=0,
            outer sep=0,
            fill=gray!35] (#3B\arabic{barcounter}) {};

        \node[rotate=45,left] at (#1+\value{barcounter}*1.2+0.9,#2-0.1) {##2};

        \stepcounter{barcounter}
    }

    \newcommand{\mybard}[3]{
        \node[draw=none,fit={(#1+0.2+\value{barcounter}*1.2,#2+##1*4) (#1+\value{barcounter}*1.2+1.2,#2+##1*4+##2*4)}, inner sep=0, outer sep=0, draw=none, fill=green!20] (#3D\arabic{barcounter}) {};

        \mybar{##1}{##3}

    }

    \newcommand{\mylastbar}[2]{
        \node[draw=none,fit={(#1+0.8+\value{barcounter}*1.2,#2) (#1+\value{barcounter}*1.2+1.8,#2+##1*4)}, inner sep=0, outer sep=0, draw=none, fill=gray!35] (#3Last) {};

        \node[rotate=45,left] at (#1+\value{barcounter}*1.2+1.4,#2-0.1) {##2};
    }

    \newcommand{\mylastbard}[3]{
        \node[draw=none,fit={(#1+0.8+\value{barcounter}*1.2,#2+##1*4) (#1+\value{barcounter}*1.2+1.8,#2+##1*4+##2*4)}, inner sep=0, outer sep=0, draw=none, fill=green!20] (#3LastD) {};

        \mylastbar{##1}{##3}
    }
}
{   
    \draw (\posx,\posy) -- (\posx+6,\posy);
    \draw[->] (\posx,\posy) -- (\posx,\posy+4.5);  
}

\section{Prior Knowledge}
\label{sec:prior_knowledge}
We define the Prior Knowledge in terms of formulas of a function-free first order language $\L$. Its signature is defined with a set of constants $\C \triangleq \{a_1, a_2, ... a_l \}$ 
and a set of predicates $\P \triangleq \{P_1, P_2 ... P_q \}$.

Unary predicates can be used to express properties of singular objects. For instance, to represents that a person $a \in \C$ is a smoker, we can use $Smoker(a)$, where $Smoker \in \P$ is a predicate with arity one. Predicates with higher arity can express relations among multiple objects in the domain, e.g. $Friends(a,b)$ states that person $a$ is a friend of $b$.

\noindent The Prior Knowledge is defined as a set of clauses: $\clauses \triangleq \{ c_1, c_2, ... c_r \} $. A clause is a disjunction of literals, each of which is a possibly negated atom:
	$$
	c \triangleq \bigvee\limits_{i=1}^{k} l_i
	$$
\noindent where $k$ is the number of literals in $c$ and $l_i$ is the $i^{th}$ literal. We assume that there are no repeated literals.

As an example, the clause $\lnot Smoker(x) \lor \lnot Friends(x,y) \lor Smoker(y)$ states that if a person $x$ is a smoker and he is a friend of another person $y$, then $y$ is also a smoker.

Notice that the previous clause has no constants and the two variables $x$ and $y$ in it are assumed to be universally quantified. This is because we are interested in representing general knowledge, and in the following, we will always apply this assumption implicitly to each clause in our knowledge.


We define the grounding of a clause $c$, denoted by $c[x/a, y/b ...]$, as the clause obtained by substituting all of its free variables with the corresponding constant symbol. For instance, if we take into consideration two specific persons $a$ and $b$, then
$$
(\lnot Smoker(x) \lor \lnot Friends(x,y) \lor Smoker(y))[ x/a,y/b ]$$
is equivalent to
$$
\lnot Smoker(a) \lor \lnot Friends(a,b) \lor Smoker(b)
$$

We will denote with $\G(c, \C)$ the set of all the groundings of a clause $c$ and with $\G(\clauses, \C)$ the set of all the grounded clauses.


\subsection{Semantic of $\L$}
\label{sec:semantic}


The semantic of $\L$ is defined in terms of an intepretation $\I$, which is a function that maps terms and grounded atoms to real values. More in details, constant symbols $o \in \C$ of $\L$ are interpreted as real-valued vectors and atoms as values in $[0,1]$. 

Intuitively, the vectors associated to constants represent the features of the corresponding objects in the real world, and they are given as input, while the truth values of atomic formulas are provided by a neural network. Since we are dealing with neural network predictions, which returns continuos values, we can not make use of classic logic. For this reason we rely on fuzzy logic, where the truth values of atoms are mapped to values in the range $[0,1]$. 

Formally, given a list of constants $o_i \in \C$, and a predicate $P \in \P$ with ariety $n$, an intepretation $\I$ is a function that satisfies the following properties: 

\begin{itemize}
	\item $\I(o_i) = \bx_i$, with $\bx_i \in \Reals^l$
	\item $\I(P) = \Reals^{nl} \to [0,1]$, where $\I(P)$ is a function which takes as input the interpretation (features) $\I(o_1), \I(o_2) \dots \I(o_n)$ of $n$ constant symbols and returns the truth value of $P(o_1,o_2 \dots o_n)$.
\end{itemize}







\subsection{T-conorm functions}
\label{sec:fuzzy}

The previous definition of interpretation provides the semantic for terms and atoms. We need to specify how to interpret clauses, i.e., define the interpretation of negation and disjunction operators.

In fuzzy logic, a standard way to calculate the truth value of a negated atom is the Lukasiewicz negation: $\I(\lnot A) = 1 - \I(A)$.
The truth value of a disjunction of literals is computed by a t-conorm function, which maps the truth values (expressed in the range $[0,1]$) of two literals to the truth value of their disjunction.

\begin{definition}
	A t-conorm $\bot: [0,1] \times [0,1] \to [0,1]$ is a binary function which satisfies the following properties:
	\begin{enumerate}
		\item $\bot(a,b) = \bot(b,a)$
		\item $\bot(a,b) \leq \bot(c,d)  \ \ if \ a \leq c \ and \ b \leq d $
		\item $\bot(a, \bot(b,c)) = \bot(\bot(a,b),c$)
		\item $\bot(a,0)=a$
	\end{enumerate}
	
\end{definition}

\noindent In the following, we represent a t-conorm as a unary function over vectors ($\blit = \left< \lit_1, \lit_2 ... \lit_n \right>$):
$$
\bot(\blit) = \bot(\lit_1, \bot(\lit_2, \bot(\lit_3 ... \bot(\lit_{n-1},\lit_n))))
$$

Given the interpretation of the constant symbols, the base NN 
returns the interpetation for the atomic formulas.  The KE should take as input the intepretation provided by the base NN
and return a new interpretation with increased satisfaction of the contraints, i.e., with higher truth values of the grounded clauses, which in turns corresponds to increase the t-conorm value.

We have now all the elements to intepret all the sentences in $\L$: the intepretation $\I$ defined in Section~\ref{sec:semantic} is implicitly extended to the clauses by the following rules:
\begin{itemize}
	\item $\I(\lnot A) = 1 - \I(A)$
	\item $\I(c) = \bot(\I(l_1),\I(l_2),\dots,\I(l_k))$ 
\end{itemize}
where $A$ is an atomic formula and $c=\bigvee_{i=1}^k l_i$ a clause.

\def\X{\mathcal{X}}
\def\INN{\I_{\mathrm{NN}}}
{
  Suppose that we have a dataset with features $\X=\{\bx_i\}$.
    Let NN be a neural network that takes in input $\X$ and produces in output
    a vector $\by\in[0,1]^m$ that contains an element $y_A$ for every ground atom $A$
    of the language.  The base NN, together with the input $\X$, defines an interpretation
    $\INN$, where $\INN(A) = y_A$. W.l.o.g we can assume that $\bm y$ contains also
    an element for every negative literal $\neg A$, denoted by $y_{\neg A}$, whose value is  $1-y_A$.
    For every clause $c=\bigvee_{i=1}^k l_i$, if $\by_c$ denotes the vector
    $(y_{l_1},\dots,y_{l_k})$, then the truth value predicted by NN for a clause $c$ is $\bot(\by_c)$.
    The intuition at the base of KENN is to add a layer on top of NN that modifies $\by$
    into $\by'$ in order to improve the truth value of each single clause $c\in\clauses$. 
    However such a change should be minimal, i.e, we have to improve the value of the clauses keeping minimal the
    change on the predictions expressed in term of the norm $||\by-\by'||$.}

\section{Clause Enhancement}
\label{sec:kenn}

In this section, we will focus on the enhancment of a single clause: given the truth values $\blit$ of the clause's literals, we want to produce a new assignment for $\blit$ that increase the t-conorm value. In order to achieve this goal, we will define a new class of functions, called \emph{t-conorm boost functions}. Later, we will extend the results of this section with the enhancement of the entire knowledge.  

\subsection{t-conorm boost functions}
\label{TBFS} 

\begin{definition}
	A function $\delta: [0,1]^n \to [0,1]^n$  is a \emph{t-conorm boost function} (\boost) if:
	$$\forall n \in \Naturals \ \ \forall \blit \in [0,1]^n \ \ 0 \leq \lit_i + \delta(\blit)_i \leq 1$$
\end{definition}
\noindent Let $\Delta$ denote the set of all \boost s. 

\begin{proposition}
For every t-conorm $\bot$ and every \boost\ $\delta$,
	$
	\bot(\blit) \leq \bot(\blit + \delta(\blit))
	$
	\begin{proof}
		By definition of \boost, 
		\mbox{$\forall i \in [1,n]$, $\lit_i \leq \lit_i + \delta(\blit)_i$};
		the conclusion directly follows from Property 2. of t-conorms.
	  \end{proof}
\end{proposition}

\begin{figure}
		\centering
		\begin{tikzpicture}[thick,scale =1,every node/.style={scale=1}]


    \begin{mygraph}{0}{0}{G1}
        \mybard{0.6}{0.1}{$\lnot Smoker(a)$}
        \mybard{0.1}{0.2}{$\lnot Friends(a,b)$}
        \mybard{0.3}{0.5}{$Smoker(b)$}
        \mylastbard{0.6}{0.2}{$\bot(c(a,b))$}

    \end{mygraph}

    \begin{mygraph}{8}{0}{G2}
        \mybard{0.6}{0.2}{$\lnot Smoker(a)$}
        \mybar{0.1}{$\lnot Friends(a,b)$}
        \mybar{0.3}{$Smoker(b)$}
        \mylastbard{0.6}{0.2}{$\bot(c(a,b))$}

    \end{mygraph}

\end{tikzpicture}










		\caption{Effect of \boost s on base NN predictions for the grounded clause $c(a,b): \lnot Smoker(a) \lor \lnot Friends(a,b) \lor Smoker(b)$, where predicate $Friends$ is true when the two persons are friends. The gray areas represent the original predictions of the base NN while the green ones the changes applied by the \boost . The last bar represents the truth value of $c$ under the G\"{o}del t-conorm, which is increased in both cases. Indeed, both graphs shows a \boost. The improvements on the left are not minimal, since it is possible to reach the same improvement with smaller changes. On the contrary, the \boost \ on the right is minimal for the G\"{o}del t-conorm.}
		\label{fig:truth_chart}
\end{figure}
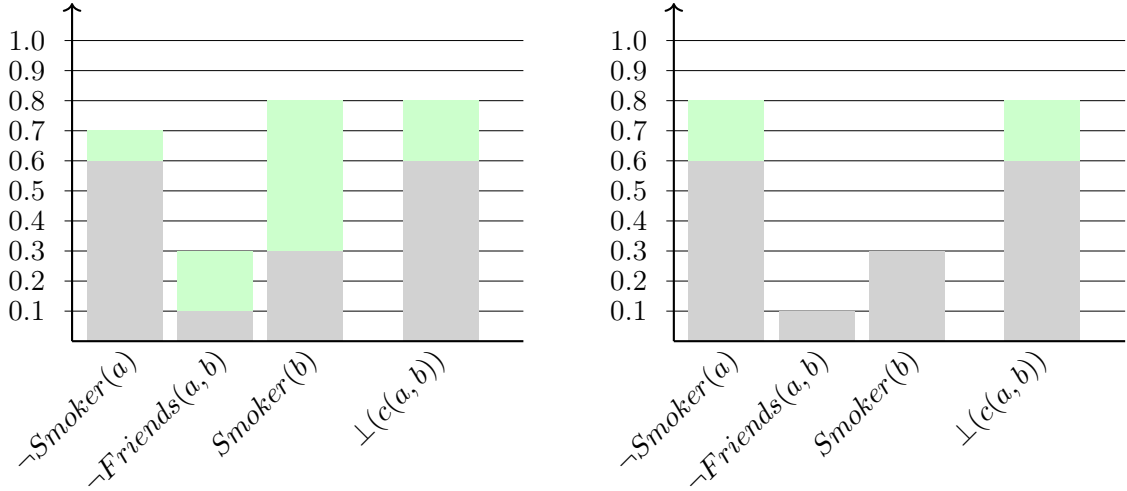

\boost s are used in the KE to update the initial predictions $\by$ done by the base NN. While there are infinite many \boost s, not all of them can be used for our purposes. Consider for example the function $\delta(\blit) = \mathbf{1} - \blit$: this of course is a \boost \ and it makes the t-conorm completely true for every possible NN predictions. 
Although the constraint reaches its maximum satisfaction, such a function returns a constant value for each literal (${\forall i \ t_i + \delta(\blit)_i = 1}$) which is useless for our purposes since the NN's predictions are not taken into account. We want to find a balance between the NN's predictions and the satisfaction of the clause: the \boost \ should keep the change on the initial predictions as minimal as possible. Therefore we look at \boost's that improve the  t-conorm value \emph{in a minimal way} so that it is not possible to obtain a higher improvement with smaller modifications on literals values. An example of this idea can be seen in Fig.~\ref{fig:truth_chart}.
We now define the concept of minimality for a \boost s.

\begin{definition}
	A function $\delta \in \Delta$ is minimal with respect to a norm $\|\cdot\|$ and a t-conorm $\bot$ iff:
	\begin{equation}
    	\begin{split}
    	&\forall \delta' \in \Delta \ \ \forall n \in \Naturals \ \ \forall \blit \in [0,1]^n \\
    	&\| \delta'(\blit)  \| < \| \delta(\blit) \| \to \bot(\blit + \delta'(\blit)) < \bot(\blit + \delta(\blit))
    	\end{split}
	\end{equation}
\end{definition}

\noindent Minimal TBF is defined with respect to a given norm  $\|\cdot\|$ and a t-conorm $\bot$. We are going to focus on G\"{o}del t-conorm which is defined as
$$
\bot({\blit}) = \max_{i=1}^n(\lit_i)
$$
and $l_p$-norm:
$$
\|\blit \|_p = \left( \sum_{k=1}^{n} | \lit_k |^p \right)^{1/p}
$$
\noindent
Notice that G\"{o}del t-conorm value depends only on the highest literal in the clause. Since we are interested in the minimal change in the predictions it seems reasonable to increase only the highest truth value.


For any function $f:{\mathbb R}^n\rightarrow {\mathbb R}$
we define $\delta^{f}: {\mathbb R}^n\rightarrow {\mathbb R}^n$ as
\begin{equation}
\delta^f(\blit)_i = 
\left\{
  \begin{array}{ll}
    \label{eq:delta-f}
f(\blit) \ \ \ & \mbox{if $i = \argmax_{j=1}^n \lit_j$} \\
0 & \mbox{otherwise}
\end{array}
\right.
\end{equation}

\begin{theorem}
If $f:[0,1]^n\rightarrow[0,1]$ satisfies \mbox{$0 \leq f(\blit) \leq 1 - \max_{j=1}^n\lit_j$}, then $\delta^f$ function is minimal \boost s for the G\"{o}del t-conorm and $l_p$-norm.
\label{th:minimal_TBF}
\end{theorem}

\begin{figure}
		\centering

\scalebox{0.5}{
\begin{tikzpicture}[align=left,font=\fontsize{20}{0}\selectfont]

\node[
draw=none,
font=\fontsize{25}{0}\selectfont
] at (5.5, 13) {$\lnot Smoker(x) \lor Cancer(x)$};

\node[
draw=none, 
rectangle,
fill=green!15,
minimum height=6.5cm,
minimum width=6.5cm] at (3.25,3.25){};



\node[
draw,
circle,
blue!30,
fill=blue!10,
dashed,
minimum height=2.97cm] at (3,5){};

\path[draw, ->] (0,0) -- (0,12);
\path[draw, ->] (0,0) -- (12,0);

\node[
draw=none,
font=\fontsize{22}{0}\selectfont
] at (-2.1, 11.4) {$\lnot Smoker(x)$};

\node[
draw=none,
font=\fontsize{22}{0}\selectfont
] at (12.2, -0.7) {$Cancer(x)$};

\path[draw, black!30, dashed] (0,5) -- (3,5);
\path[draw, black!30, dashed] (3,0) -- (3,5);
\path[draw, thick, ->] (0,0) -- (3,5);
\path[draw, ->, thick, blue!50] (3,5) -- (3,6.5);
\node[] at (-0.6, 10) {1};
\node[] at (10, -0.6) {1};
\path[draw] (-0.2,10) -- (0.2,10);
\path[draw] (10,-0.2) -- (10,0.2);

\path[draw, black!30, dashed] (0,6.5) -- (3,6.5);
\node[] at (-0.95, 6.5) {0.65};
\path[draw] (-0.2,6.5) -- (0.2,6.5);

\node[] at (-0.8, 5) {0.5};
\node[] at (3, -0.6) {0.3};
\path[draw] (-0.2,5) -- (0.2,5);
\path[draw] (3,-0.2) -- (3,0.2);

\node[] at (-0.6,-0.6) {O};
\node[] at (2,2.5) {$\blit$};
\node[] at (4,5.75) {$\delta^f(\blit)$};


\end{tikzpicture}}
		\caption{Geometric intuition for minimality of $\delta^f$: an example with $\blit=[0.3,0.5]$ (black line). The blue line represents the change applied by $\delta^f$ while the blue circle contains the set of points $\blit + \mathbf{k}$ with $\| \mathbf{k} \|_2 < \| \delta^f(\blit)\|_2$. The green area is the set of all point $\blit + \mathbf{k}$ with $\bot(\blit+ \mathbf{k}) < \bot(\blit+\delta^f(\blit))$. The blue area is completely inside the green square, i.e. $\delta^f(\blit)$ is the minimal change we could apply for incrementing the t-conorm of 0.15.}
		\label{fig:proof}
\end{figure}
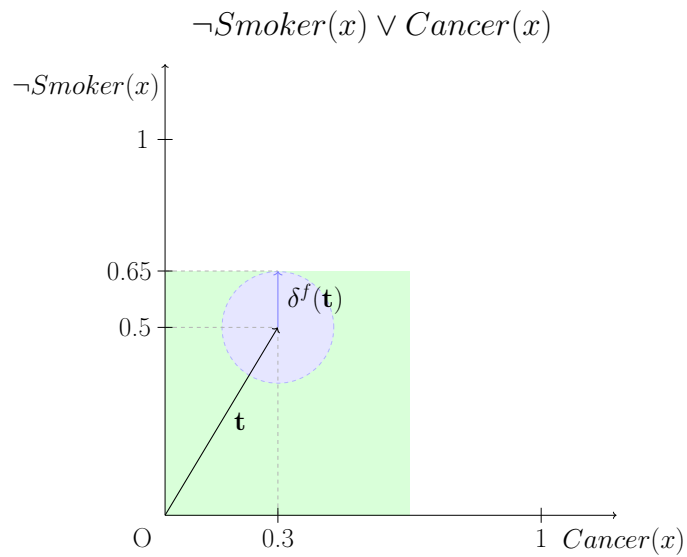
Fig.~\ref{fig:proof} shows a geometric interpretation for Theorem~\ref{th:minimal_TBF}. In the example, it is shown a specific case: the initial truth values for the two literals $\lnot Smoker(x)$ and $Cancer(x)$ are $0.5$ and $0.3$ respectively. The G\"{o}del t-conorm value is $0.5$, since it is the highest value. To be minimal, a \boost \ must increase the value of the t-conorm with the smallest possible change, meaning that it is not possible to reach the same improvement on the t-conorm with smaller changes. The blue circle contains the set of points that apply a smaller change to $\blit$ than $\delta^f$ (in this case, we consider the $l_2$-norm). The picture shows that for all those points the G\"{o}del t-conorm is smaller than $0.65$, the value obtained by using $\delta^f$.

Of course, this does not represent a proof for the minimality of $\delta^f$, since it is applied to a specific set of truth values and it considers only $l_2$-norm. We need to prove it formally.

\begin{proof}
It is easy to see that $\delta^f(\blit)$ is a \boost: the function always returns values inside the range $[0,1]$ and even the final predictions are in this range. We need to prove that it is minimal. Suppose that $\delta\in \Delta$ is such that 
	$$
	\| \delta(\blit) \|_p < \| \delta^f(\blit) \|_p
	$$
	If $j =
        \argmax_{k=1}^n(\lit_k + \delta(\blit)_k)$, 
	we can derive:
	$$
	\bot(\blit + \delta(\blit)) = \lit_j + \delta(\blit)_j
	$$
	and, if $i = \argmax_{k=1}^n\lit_k$, we have that 
	$$
	\bot(\blit + \delta^f(\blit)) = \lit_i + f(\blit)
	$$
\noindent
Since $\lit_i\geq \lit_j$, we just need to demonstrate that \mbox{$\delta(\blit)_j < f(\blit)$}. Notice that: 
	\begin{align*}
	\delta(\blit)_j &= {(|\delta(\blit)_j|^p)}^{1/p} \leq \left( \sum_{k=1}^{n} | \delta(\blit)_k |^p \right)^{1/p}
	= \| \delta(\blit) \|_p < \| \delta^f(\blit) \|_p
	\end{align*}		
Since $\delta^f(\blit)$ changes only the value of the $i^{th}$ component of $\blit$, \mbox{$\|\delta^f(\blit)\|_p = f(\blit)$}.
\end{proof}

\section{Boosting preactivations}
\label{preact}
In previous sections, we proved that, with an appropriate choice of function $f$, $\delta^f$ is a minimal \boost. However, we still need to choose a function $f$ suitable for our purposes. 
Notice that we have to respect the constraint that the final predictions remain in $[0,1]$. Indeed, Theorem~\ref{th:minimal_TBF} states that \mbox{$0 \leq f(\blit) \leq 1 - \max_{j=1}^n\lit_j$}. This limits the possible functions $f$ that can be used. For instance, $f$ cannot be a linear function. In this section, we will extend the results obtained so far, showing that such a constraint does not have to be respected when applying $\delta^f$ directly on the preactivations of the NN's last layer.



More in details, the outputs $\by$ of the base NN are calculated by applying an activation function over the preactivations $\bz$ generated in the last layer. We assume the activation function 
to be the logistic function, i.e.:
$$
y_i = \sigma(z_i) = \frac{1}{1 + e^{-z_i}}
$$



\noindent where $y_i$ is the activation of the $i^{th}$ predicate and $z_i$ the corresponding preactivation. As with the activations, we distinguish between preactivations of the grounded atoms and the ones of the literals of a grounded clause by using $\bz$ and $\bv$ respectively. More precisely:
$$\bv = \sigma^{-1}(\blit)$$

In section~\ref{TBFS} we showed that if we increase only the value of the highest literal (highest activation), such a change is minimal for G\"{o}del t-conorm. We can apply the same strategy to preactivations and still have a minimal change that increases the t-conorm. In other words, we can increase $v_i$ instead of $t_i$, when $i = \argmax_{j=1}^n v_j$ and the previously proven properties still hold.

The initial predictions are
$
\blit = \sigma(\bv)
$
and the final ones are
$
\blit' = \sigma(\bv + \delta^f(\bv))
$.



\begin{proposition}
For any function $f:\Reals^n\rightarrow\Reals^+$, the function 
\begin{equation}
  \delta^g(\blit) = \sigma(\bv + \delta^f(\bv)) - \sigma(\bv)  
\end{equation}
is a minimal \boost \ for G\"{o}del t-conorm and $l_p$-norm.
\end{proposition}


	  
\begin{proof}
  Notice that final predictions $\blit' = \sigma(\bv + \delta^f(\bv))$ are in $[0,1]$ since this range is the image of logistic function. Furthermore, $\sigma$ is monotonic increasing, which means that the highest preactivation corresponds to the highest activation:
  $$\argmax_{j=1}^n \sigma(v_j) = \argmax_{j=1}^n v_j$$
  Another implication of the monotonicity of $\sigma$ is that increasing a preactivation produces also an increase of the corresponding activation:
  $$f(\bv) \geq 0 \to \sigma(v_i + f(\bv)) \geq \sigma(v_i)$$
  \noindent
  Putting the previous two together we obtain that increasing the highest preactivation does indeed imply an increase of the highest activation, meaning that $\delta^g$ is a minimal \boost.

  More formally, let's define function $g(\blit) = \sigma(v_i + f(\bv)) - \sigma(v_i)$.
  We obtain:
	\begin{equation}\begin{split}
		\delta^g(\blit)_i &= \sigma(\bv + \delta^f(\bv))_i - \sigma(\bv)_i \\
		&= 
		\left\{
		  \begin{array}{ll}
			g(\blit) \ \ \ & \mbox{if $i = \argmax_{j=1}^n t_j$} \\
			0 & \mbox{otherwise}
		\end{array}
		\right.
	\end{split}
	\end{equation}
   Theorem~\ref{th:minimal_TBF} guarantees that $\delta^g$ is a minimal \boost \ for G\"{o}del t-conorm and $l_p$-norm
\end{proof}
The function $\delta^g$ is not directly used by KENN; it is implicitly induced by the application of $\delta^f$ on $\bv$. Therefore, by showing that it is a minimal TBF, we prove that applying $\delta^f$ on $\bv$ is indeed equivalent to apply a minimal TBF on the NN predictions. 

Until now we have seen that we can modify the preactivations $\bv$ instead of changing directly the final truth values ($\blit$), but we haven't justified why it is better to do this. The first reason is that applying changes on $\bv$ has the advantage of guaranteeing the final predictions to be in $\left[0,1\right]$ which means that we don't have any constraint on the choice of function $f$. Moreover, we could interpret the results of the Neural Network as levels of confidence in its predictions. According to this view, an initial value close to one (or zero) of the activation means high confidence in the prediction and we would like $f$ to apply a small change in such cases. On the other hand, if the initial prediction is close to $0.5$ (preactivation close to $0$) then we have the maximum uncertainty. Fig.~\ref{fig:logistic} shows how the same change on two different preactivation values results in different changes on the activations: the closer the initial value is to zero (maximum uncertainty), the higher the change on final truth values.

\DONE{Summarizing, to increase $c$ satisfaction we increase the preactivation of its highest literal. The extent of the change depends on the function $f$. \kenn\ uses a distinct $f_c$ function for every clause $c$.} This is motivated by the fact that we want $\delta^{f_c}$ to be proportional to \emph{clause weight} $w_c$ (learnable parameter) which expresses the strength of the clause. The simplest function that conforms to this property is the constant function $w_c$ (with $w_c \in [0,\infty]$). The function applied to $\bv$ to increase $c$ satisfaction is therefore $\delta^{w_c}$ defined as in \ref{eq:delta-f}:
\begin{equation}
	\delta^{w_c}(\bv)_i = 
	\left\{
	  \begin{array}{ll}
	w_c \ \ \ & \mbox{if $i = \argmax_{j=1}^n v_j$} \\
	0 & \mbox{otherwise}
	\end{array}
	\right.
\end{equation}

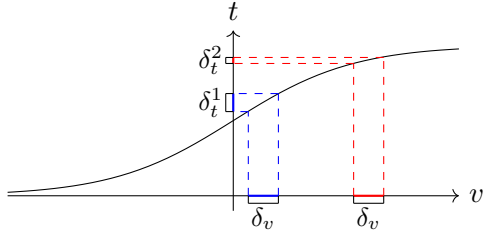
\begin{figure}[t]
	\centering
    \begin{tikzpicture}[scale=2]
        \draw[->] (-1.5,0) -- (1.5,0) node[right] {$v$};
        \draw[->] (0,-0.1) -- (0,1.1) node[above] {$t$};
        \draw[scale=1.0,domain=-1.5:1.5,smooth,variable=\x,black] plot ({\x},{1/(1+e^(-2.5*\x))});
        
        \draw[blue, thick] (0.1,0) -- (0.3,0);
        \draw (0.1, 0) -- (0.1, -0.05);
        \draw (0.3, 0) -- (0.3, -0.05);
        \draw (0.1, -0.05) -- (0.3, -0.05);
        \node[font=\small\sffamily] at (0.2, -0.15) {$\delta_v$};
        
        \draw[blue, dashed] (0.1, 0) -- (0.1, 0.56);
        \draw[blue, dashed] (0.3, 0) -- (0.3, 0.68);
        \draw[blue, dashed] (0.1, 0.56) -- (0, 0.56);
        \draw[blue, dashed] (0.3, 0.68) -- (0, 0.68);
        \draw[blue, thick] (0, 0.56) -- (0, 0.68);
        \draw (0, 0.56) -- (-0.05, 0.56);
        \draw (0, 0.68) -- (-0.05, 0.68);
        \draw (-0.05, 0.56) -- (-0.05, 0.68);
        \node[font=\small\sffamily] at (-0.15, 0.62) {$\delta_t^1$};

        \draw[red, thick] (0.8,0) -- (1.0,0);
        \draw (0.8, 0) -- (0.8, -0.05);
        \draw (1, 0) -- (1, -0.05);
        \draw (0.8, -0.05) -- (1, -0.05);
        \node[font=\small\sffamily] at (0.9, -0.15) {$\delta_v$};
        \draw[red, dashed] (0.8, 0) -- (0.8, 0.88);
        \draw[red, dashed] (1, 0) -- (1, 0.92);
        \draw[red, dashed] (0.8, 0.88) -- (0, 0.88);
        \draw[red, dashed] (1, 0.92) -- (0, 0.92);
        \draw[red, thick] (0, 0.88) -- (0, 0.92);
        
        \draw (0, 0.88) -- (-0.05, 0.88);
        \draw (0, 0.92) -- (-0.05, 0.92);
        \draw (-0.05, 0.88) -- (-0.05, 0.92);
        \node[font=\small\sffamily] at (-0.15, 0.90) {$\delta_t^2$};
  \end{tikzpicture}
        \vspace*{-0.5cm}
	\caption{The same change $\delta_v$ applied on two different values of preactivations results in different changes on activations: the more close the preactivation to zero (maximum uncertainty) the highest the modification on final predictions}
	\label{fig:logistic}
      \end{figure}

Notice that, although $\delta^{w_c}$ respects our minimality property, there are two problems when using it inside a neural network: first, it is not differentiable; second, it is too strict when multiple literals have close values. In those cases, it increases just one of the values even if the difference is minimal.
To obviate these problems, in our implementation we substitute $\delta^{w_c}$ with the $\softmax$ function multiplied by $w_c$, that can be seen as a soft differentiable approximation of $\delta^{w_c}$:
\begin{equation}
\delta^{w_c}_s(\bv)_i = w_c \cdot \softmax(\bv)_i = w_c \cdot \frac{e^{v_i}}{\sum_{j=1}^{n}e^{v_j}}
\label{eq:delta_wc_s}
\end{equation}

\section{Increasing the satisfaction of the entire knowledge}
\label{sec:entire_knowledge}
Until now we considered only the changes to be applied to increase a single clause, while in the general case $\clauses$ will be composed of multiple clauses. We need a way to aggregate the multiple changes proposed by each clause to obtain the final changes to be applied on all the predictions based on the entire knowledge.

For every atom $A$, let $z_A$ be the preactivation of $y_A$, which is $z_A=\sigma^{-1}(y_A)$.
As done for the activations $\by$, we can extend the notation to literal, letting $z_{\neg A}$ be the
preactivation corresponding to $y_{\neg A}$, namely $z_{\neg A}=\sigma^{-1}(y_{\neg A})$. From the following property of logistic function 
$$1 - \sigma(x) = \sigma(-x)$$ and from the fact that $y_{\neg A}= 1-y_A$ we can consistently define 
$$
z_{\lnot A} = - z_A
$$
Finally, we define $\bz_c=\left<z_{l_1},\dots,z_{l_n}\right>$ for every clause $c=\bigvee_{i=1}^nl_i$.

Let $\clauses$ be a set of clauses and $\{w_c\}_{c\in\clauses}$ their corresponding weights.
  For every clause $c\in\clauses$ we define $\bdelta^c$, a vector that contains one value $\delta^c_A$ for every atom $A$ of the language. Each element $\delta^c_A$ of $\bdelta^c$ is defined as follows: 
  \begin{equation}
    \label{eq:deltac}
  \delta^c_A=\begin{cases}
    \delta^{w_c}_s(\bz_c)_A & \mbox{if $A\in c$} \\
    -\delta^{w_c}_s(\bz_c)_{\lnot A} & \mbox{if $\neg A\in c$} \\
    0 & \mbox{Otherwise}
  \end{cases}
\end{equation}
Intuitively, $\bdelta^c$ is the vector that contains the contributions of the clause $c$ to all the preactivations $\bz$. Notice that
  the clauses affects only the truth value of the atoms that appears in it. In particular, the clause $c$ contributes positively
  to the atom $A$ if $A$ appears positively in $c$, and it contributes negatively if $\neg A$ appears in $c$.
  To aggregate the contributions of each clause, we simply sum all the contributions. Therefore the final prediciton is defined as
  the sigmoid applied to the modified preactifations. Formally: 
\begin{equation}
	\by' = \sigma(\bz + \sum_{c \in \G(\clauses, \C)} \bdelta^c)
	\label{eq:final_intepretation}
\end{equation}

The choice of the sum as the aggregator of the various contributions makes learning and inference fast. In this way, the scalability is increased but also the probability of inconsistency between the final predictions and the logical rules. This aspect will be further investigated in Section~\ref{sec:rel_work_2}, where the extent of the problem will be analyzed and a comparison with RNM from this point of view will be provided.

\section{\kenn\ architecture}
\label{sec:architecture}

As explained in Section~\ref{sec:kenn}, the core component of \kenn\
is in the \boost . In \kenn\ the \boost\ for a certain clause $c$ is
implemented by an architecture called \emph{Clause Enhancer} (CE), which is a
module instantiated that computes $\delta_c$ according to Equation~\eqref{eq:deltac}. 
Figure~\ref{fig:CE} shows the details of the architecture of the CE. 


\begin{figure}[H]
  \centering
\definecolor{predicates}{rgb}{0.9,0.2,0.2}
\definecolor{clauses}{rgb}{0.3,0.3,0.7}
\definecolor{inputs}{rgb}{0.1,0.6,0.2}
\definecolor{outputs}{rgb}{0.9,0.2,0.2}

\tikzstyle{block} = [rectangle, draw, fill=black!5, text centered, rounded corners, minimum height=1cm, minimum width=1.8cm, align=center, font=\fontsize{8}{0}\selectfont]
\tikzstyle{number} = [circle, align=center, font=\fontsize{9}{0}\selectfont]

\scalebox{1.0}{
\begin{tikzpicture}[scale=0.4, node distance=1mm, align=left]

\node[
rectangle split,
draw,
rectangle split horizontal,
rectangle split parts=3,
minimum height=0.6cm,
minimum width=2.2cm,
align=center,
font=\fontsize{10}{0}\selectfont](z2) at (0, 0){
	\nodepart[]{one} $ z_{\scaleto{A}{4pt}}$
	\nodepart[]{two} $ z_{\scaleto{B}{4pt}}$
	\nodepart[]{three} $ z_{\scaleto{C}{4pt}}$};


\node[
draw=none,
minimum height=0.6cm,
minimum width=2.5cm,
thick,
rounded corners,
black!70,
above=of z2,
yshift=0.7cm,
fill=black!10]  (phi) {\bfseries{select}};

\node[
above=of phi,
rectangle split,
draw,
rectangle split horizontal,
rectangle split parts=2,
minimum height=0.6cm,
yshift=0.45cm,
align=center,
font=\fontsize{10}{0}\selectfont](z3) {
	\nodepart[]{one} $\  z_{\scaleto{A}{4pt}} \ $
	\nodepart[]{two} $ z_{\scaleto{\lnot B}{4pt}}$};

\path[draw, ->] (z2) -- (phi);
\path[draw, ->] (phi) -- (z3);

\node[
draw=none,
minimum height=0.6cm,
minimum width=2.5cm,
thick,
rounded corners,
above=of z3,
black!70,
yshift=0.5cm,
fill=black!10]  (softmax) {\bfseries{softmax}};

\node[left=of softmax, xshift=-0.6cm] (w) {$w_c$};

\node[draw=none, fit=(softmax) (w)] (swwrapper){};

\node[
	draw=none,
	circle,
    above=of swwrapper,
	yshift=0.45cm,
	xshift=-0.65cm,
	black!70,
	fill=black!10,
	font=\fontsize{13}{0}\selectfont
] (prod) {$*$};

\path[draw, ->] (w) -- (prod);
\path[draw, ->] (softmax) -- (prod);

\path[draw, ->] (z3) -- (softmax);

\node[
rectangle split,
draw,
rectangle split horizontal,
rectangle split parts=2,
minimum height=0.6cm,
minimum width=2.5cm,
align=center,
above=of prod,
yshift=0.8cm,
font=\fontsize{10}{0}\selectfont](dl) {
	\nodepart[]{one} $\ {\delta_{\scaleto{A}{4pt}}^{c}} \ $
	\nodepart[]{two} $ {\delta_{\scaleto{\lnot B}{4pt}}^{c}}$};

\path[draw, ->] (prod) -- (dl);


\node[
draw=none,
minimum height=0.6cm,
minimum width=2.5cm,
thick,
rounded corners,
black!70,
above=of dl,
yshift=0.45cm,
fill=black!10]  (phip) {Eq.7};

\path[draw, ->] (dl) -- (phip);

\node[
above=of phip,
rectangle split,
draw,
rectangle split horizontal,
rectangle split parts=3,
minimum height=0.6cm,
minimum width=1.8cm,
yshift=0.6cm,
align=center,
font=\fontsize{10}{0}\selectfont](dz3){
	\nodepart[]{one} $ {\delta_{\scaleto{A}{4pt}}^{c}}$
	\nodepart[]{two} $ {\delta_{\scaleto{B}{4pt}}^{c}}$
	\nodepart[]{three} { \ \emph{0} }};

\path[draw, ->] (phip) -- (dz3);

\node[left=of z2, align=center] (zlabel) {$\bz$};
\node[left=of dz3, align=center] (dclabel) {$\bdelta^c$};
\node[align=center] at (z3 -| zlabel) {$\bz_c$};


\begin{scope}[on background layer]
\node[
draw=none,
ultra thick,
fit=(phip) (softmax) (z3),
rounded corners,
minimum height=8cm,
minimum width=6cm,
inputs,
yshift=-0.7cm,
outputs,
fill=red!15]  (CEBOX3) {};
\end{scope}

\node[] (CE) at (3.7,20.4) {\bfseries{CE}};
\node[below=of CE, yshift=0.2cm] {$c$: $A \lor \lnot B$};

\node[
	draw,
	dashed,
	fit=(w) (prod) (softmax),
	rounded corners,
	minimum height=2.5cm,
] (dsw) {};

\node[
	draw=none,
	right=of dsw,
	xshift=2.5cm
] (dswlabel) {$\delta_s^{w}$ \\ Eq.~\ref{eq:delta_wc_s}};

\path[draw, ->, dashed] (dsw) -- (dswlabel);

\end{tikzpicture}
}
	\caption{Clause Enhancer for $A \lor \lnot B$, where $A$ and $B$ are grounded atoms}
	\label{fig:CE}
\end{figure}
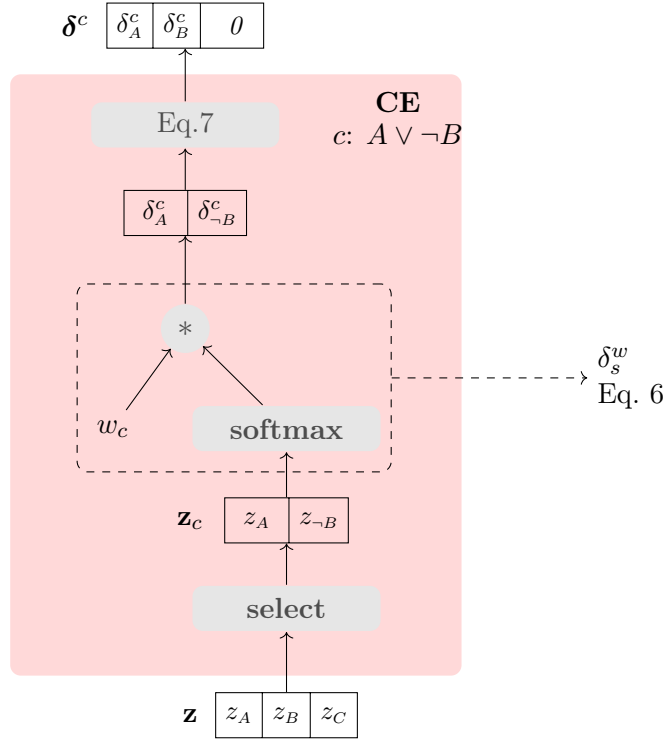

The CE receives as input a vector containing the preactivations generated by the base NN and it selects the right literals, returning $\bz_c$, the preactivations of the literals of clause $c$.
Then, Equation \ref{eq:delta_wc_s} of $\delta_s^w$ is applied. 
The term $w_c$ is a positive weight associated to clause $c$ that represents the strength of $c$: the higher the $w_c$, the bigger the contribution of the clause in the final predictions. As opposed to methods based on the regularization of the Loss, $w_c$ is not a hyper parameter. On the contrary, it is a parameter of the model that is learned during training. This means that \kenn\ can discover the importance of each clause based on the data. As a special case of this behavior, it is possible for \kenn\ to learn to ignore clauses by setting $w_c$ to zero, making clause $c$ irrelevant for the final predictions.

The final step of the CE is a post elaboration that projects the proposed changes on the literals into changes to be applied on the atoms according to equation \ref{eq:deltac}.


The outputs of the CEs are combined inside the \emph{Knowledge Enhancer} (KE), which is the final layer of \kenn\ model. Figure~\ref{fig:KE} shows the architecture of KE, which calculates the final interpretation as in Equation~\ref{eq:final_intepretation}.

\begin{figure}[H]
	\centering
\definecolor{predicates}{rgb}{0.9,0.2,0.2}
\definecolor{clauses}{rgb}{0.3,0.3,0.7}
\definecolor{inputs}{rgb}{0.1,0.6,0.2}
\definecolor{outputs}{rgb}{0.9,0.2,0.2}

\tikzstyle{block} = [rectangle, draw, fill=black!5, text centered, rounded corners, minimum height=1cm, minimum width=1.8cm, align=center, font=\fontsize{8}{0}\selectfont]
\tikzstyle{number} = [circle, align=center, font=\fontsize{9}{0}\selectfont]
\scalebox{0.7}{
\begin{tikzpicture}[scale=0.4, node distance=1mm, align=left]

\node (z) at (0, 0) {$\bz$};

\node[
	draw=none,
	ultra thick,
	rounded corners,
	minimum height=1.5cm,
	minimum width=2.6cm,
	inputs,
	outputs, 
	fill=red!15]  (CEBOX1) at (-10,7.7) {};
\node[font=\fontsize{13}{0}\selectfont] (C1) at (-10,7) {$c_1$};
\node[above=of C1,font=\fontsize{13}{0}\selectfont] (CE1) {\bfseries{CE}};

\node[
	draw=none,
	ultra thick,
	rounded corners,
	minimum height=1.5cm,
	minimum width=2.6cm,
	outputs,
	fill=red!15]  (CEBOX2) at (0,7.7) {};
\node[font=\fontsize{13}{0}\selectfont] (C2) at (0, 7) {$c_2$};
\node[above=of C2,font=\fontsize{13}{0}\selectfont] (CE2) {\bfseries{CE}};

\path[draw, ->] (z) -- (0,4) -- (-10, 4) -- (CEBOX1);
\path[draw, ->] (z) -- (CEBOX2);

\node (dz1) at (-10, 12){$\bdelta_{c_1}$};

\path[draw, ->] (CEBOX1) -- (dz1);

\node (dz2) at (0, 12){$\bdelta_{c_2}$};

\path[draw, ->] (CEBOX2) -- (dz2);
\node[
draw=none,
circle,
black!70,
fill=black!10,
minimum height=0.9cm,
font=\fontsize{13}{0}\selectfont] (sum) at (0, 15.5) {$+$};

\node[
draw=none,
above=of sum,
circle,
black!70,
fill=black!10,
yshift=0.8cm,
minimum height=0.9cm,
font=\fontsize{13}{0}\selectfont] (sigma) at (0, 15.5) {$\sigma$};

\path[draw, ->] (dz1) --++ (0,3.5) -- (sum);
\path[draw, ->] (dz2) --  (sum);
\path[draw, ->] (z) -- (0, 3) node[] (two){} -- (6.5,3) -- node[pos=0.89, anchor=west] (one){} (6.5, 15.5) -- (sum);
\path[draw, ->] (sum) -- (sigma);

\node (y) at (0, 22.8){$\by'$};

\path[draw, ->] (sigma) -- (y);

\node[font=\fontsize{13}{0}\selectfont] (GE) at (9, 18){\bfseries{KE}};

\begin{scope}[on background layer]
\node[
draw=none,
fit=(CEBOX1) (one) (sigma) (two) (GE),
ultra thick,
rounded corners,
minimum height=4cm,
minimum width=10cm,
inputs,
fill=green!15]  (GEBOX1) {};
\end{scope}

\end{tikzpicture}}
	\caption{Knowledge Enhancer architecture. It implements Equation~\ref{eq:final_intepretation} by summing the deltas values of each clause enhancer (one per grounded clause) with the initial preactivations $\bz$ and applying the logistic function }
	\label{fig:KE}
\end{figure}
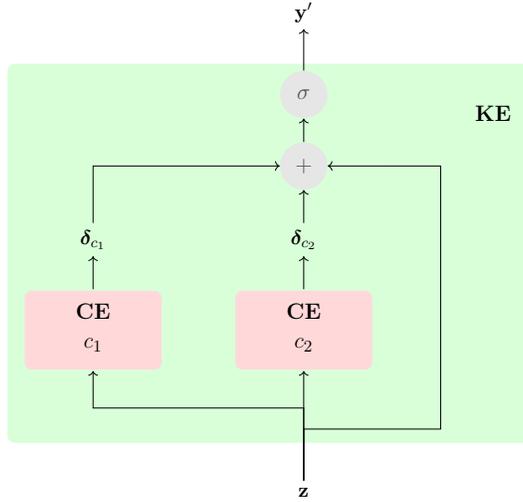

\subsection{Working with batches of data}
In the architecture defined so far, since all the grounded atoms are inside the vector $\bz$, there are multiple instantiation of the same CE for each clause $c \in \clauses$, one for each grounding of $c$. For instance, in Figure~\ref{fig:KE_matrix}(a), the CE for clause $c: \lnot S(x) \lor C(x)$ is instantiated two times, one for object $a$, another for object $b$. 

In many scenarios we want to use batches of data, which would require to modify the structure of the model graph for each batch, in particular after each training step, since the groundings are constantly changing and for each grounding a different CE is instantiated. Notice that, even if the CEs are different, the internal operations are the same and the only change is the provided input. For these reasons, we want to instantiate a single CE for each clause $c$ which works in a single batch containing all possible groundings of $c$.

If the groundings involve a single object, a simple solution is to define $\bz$ as a matrix instead as a vector, where columns represent predicates and rows individuals. More in details, $\bz$ is now defined as a matrix such that the element $z_{ij}$ contains the preactivation of $P_j(o_i)$, with $P_j$ the $j^{th}$ predicate and $o_i$ the $i^{th}$ object. 
Notice that this kind of representation is common when working with neural networks since the columns (predicates) correspond to the labels and the rows (groundings) to the samples. 

Figure~\ref{fig:KE_matrix}(b) shows the overall idea. The CE is defined as before, but now it takes as input a matrix and it acts on each row separately. This can be done only if the same atom does not appear in multiple rows, since the changes are applied independently to each row and are not aggregated together. This property always holds with clauses that
involve a single variable.


When a clause contain multiple variables, the same ground atom may occour in multiple groundings of the clause. 
For instance, consider the clause 
$$c: \lnot Smoke(x) \lor \lnot Friends(x,y) \lor Smoke(y)$$
The two groundings $c[x/Alice,y/Bob]$ and $c[x/Bob,y/Carl]$ share a common grounded atom: $Smoke(Bob)$. For this reason, using a unique CE for $c$ is problematic since the values returned by the same CE are not aggregated, and we would end up with multiple changes proposed for the same atom.

In the next sections, we will present the architecture of \kenn\ that can be used to instantiate a single CE even in relational domains, i.e. when dealing with predicates with arity greater than one.

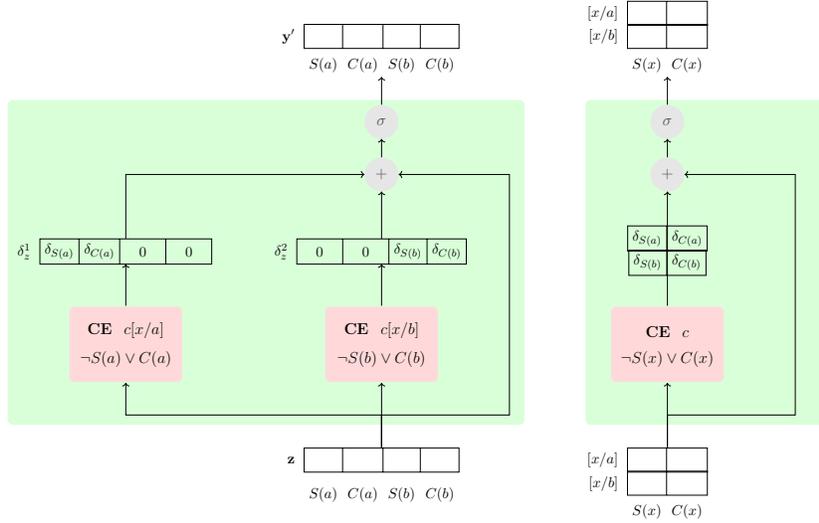
\begin{figure}[H]
	\centering
\definecolor{predicates}{rgb}{0.9,0.2,0.2}
\definecolor{clauses}{rgb}{0.3,0.3,0.7}
\definecolor{inputs}{rgb}{0.1,0.6,0.2}
\definecolor{outputs}{rgb}{0.9,0.2,0.2}

\tikzstyle{block} = [rectangle, draw, fill=black!5, text centered, rounded corners, minimum height=1cm, minimum width=1.8cm, align=center, font=\fontsize{8}{0}\selectfont]
\tikzstyle{number} = [circle, align=center, font=\fontsize{9}{0}\selectfont]
\scalebox{0.5}{
\begin{tikzpicture}[scale=0.4, node distance=1mm, align=left]

	\node[
		rectangle split,
		draw,
		rectangle split horizontal,
		rectangle split parts=4,
		minimum height=0.6cm,
		text opacity=0,
		align=center](z) at (0, 0){
		\nodepart[]{one} $ S(a) $
		\nodepart[]{two} $ C(a)$
		\nodepart[]{three} $ S(b)$
		\nodepart[]{four} $ C(b)$};

	\node[
		rectangle split,
		draw=none,
		below=of z,
		rectangle split horizontal,
		rectangle split parts=4,
		minimum height=0.6cm,
		yshift=-0.5cm,
		align=center] at (0, 0){
		\nodepart[]{one} $ S(a) $
		\nodepart[]{two} $ C(a)$
		\nodepart[]{three} $ S(b)$
		\nodepart[]{four} $ C(b)$};

	\node[left=of z] {$\bz$};

	\node[
		draw=none,
		ultra thick,
		rounded corners,
		minimum height=2cm,
		minimum width=3cm,
		inputs,
		outputs, 
		fill=red!15]  (CEBOX1) at (-17,7.7) {};
	\node[font=\fontsize{13}{0}\selectfont] (C1) at (-17,6.7) {$\lnot S(a) \lor C(a)$};
	\node[above=of C1,font=\fontsize{13}{0}\selectfont] (CE1) {\bfseries{CE} \ $c[x/a]$};

	\node[
		above=of CEBOX1,
		rectangle split,
		draw,
		rectangle split horizontal,
		rectangle split parts=4,
		minimum height=0.6cm,
		yshift=1cm,
		align=center](dz1) {
		\nodepart[]{one} $ \delta_{S(a)} $
		\nodepart[]{two} $ \delta_{C(a)}$
		\nodepart[]{three} $ \ \ \ 0 \ \ \ $
		\nodepart[]{four} $ \ \ \ 0 \ \ \ $};

	\node[left=of dz1] (dz1l) {$\delta_z^1$};

	\node[
		draw=none,
		ultra thick,
		rounded corners,
		minimum height=2cm,
		minimum width=3cm,
		outputs,
		fill=red!15]  (CEBOX2) at (0,7.7) {};

	\node[
		above=of CEBOX2,
		rectangle split,
		draw,
		rectangle split horizontal,
		rectangle split parts=4,
		minimum height=0.6cm,
		yshift=1cm,
		align=center](dz2) {
		\nodepart[]{one} $ \ \ \ 0 \ \ \ $
		\nodepart[]{two} $ \ \ \ 0 \ \ \ $
		\nodepart[]{three} $ \delta_{S(b)}$
		\nodepart[]{four} $ \delta_{C(b)}$};

	\node[left=of dz2] {$\delta_z^2$};

	\node[font=\fontsize{13}{0}\selectfont] (C2) at (0, 6.7) {$\lnot S(b) \lor C(b)$};
	\node[above=of C2,font=\fontsize{13}{0}\selectfont] (CE2) {\bfseries{CE} \ $c[x/b]$};

	\path[draw, ->] (z) -- (0,3) -- (-17, 3) -- (CEBOX1);
	\path[draw, ->] (z) -- (CEBOX2);

	\node[
	draw=none,
	circle,
	black!70,
	fill=black!10,
	minimum height=0.9cm,
	font=\fontsize{13}{0}\selectfont] (sum) at (0, 19) {$+$};

	\node[
	draw=none,
	circle,
	black!70,
	fill=black!10,
	yshift=0.8cm,
	minimum height=0.9cm,
	font=\fontsize{13}{0}\selectfont] (sigma) at (0, 20.5) {$\sigma$};

	\path[draw, ->] (CEBOX1) -- (dz1);
	\path[draw, ->] (CEBOX2) --  (dz2);
	\path[draw, ->] (dz1) |- (sum);
	\path[draw, ->] (dz2) --  (sum);
	\path[draw, ->] (z) -- (0, 3) node[] (two){} -- (8.5,3) -- node[pos=0.89, anchor=west] (one){} (8.5, 19) -- (sum);
	\path[draw, ->] (sum) -- (sigma);


	\node[
		rectangle split,
		draw=none,
		rectangle split horizontal,
		rectangle split parts=4,
		minimum height=0.6cm,
		align=center](y) at (0, 26.3){
		\nodepart[]{one} $ S(a) $
		\nodepart[]{two} $ C(a)$
		\nodepart[]{three} $ S(b)$
		\nodepart[]{four} $ C(b)$};

	\node[
	rectangle split,
	draw,
	above=of y,
	rectangle split horizontal,
	rectangle split parts=4,
	minimum height=0.6cm,
	text opacity=0,
	align=center] (yl) {
	\nodepart[]{one} $ S(a) $
	\nodepart[]{two} $ C(a)$
	\nodepart[]{three} $ S(b)$
	\nodepart[]{four} $ C(b)$};

	\node[left=of yl] {$\by'$};

	\path[draw, ->] (sigma) -- (y);

	\begin{scope}[on background layer]
	\node[
	draw=none,
	fit=(dz1) (one) (sigma) (two) (dz1l),
	ultra thick,
	rounded corners,
	minimum height=6.5cm,
	minimum width=10cm,
	inputs,
	fill=green!15]  (GEBOX1) {};
	\end{scope}


	\node[
		rectangle split,
		draw,
		rectangle split horizontal,
		rectangle split parts=2,
		minimum height=0.6cm,
		text opacity=0,
		align=center](z2) at (19, 0){
		\nodepart[]{one} $ S(x) $
		\nodepart[]{two} $ C(x)$};

	\node[left=of z2, draw=none] {$[x/a]$};

	\node[
		below=of z2,
		rectangle split,
		draw,
		rectangle split horizontal,
		rectangle split parts=2,
		minimum height=0.6cm,
		text opacity=0,
		yshift=0.14cm,
		align=center](z3){
		\nodepart[]{one} $ S(x) $
		\nodepart[]{two} $ C(x)$};

	\node[left=of z3, draw=none] {$[x/b]$};

	\node[
		below=of z3,
		rectangle split,
		draw=none,
		rectangle split horizontal,
		rectangle split parts=2,
		minimum height=0.6cm,
		align=center](z3){
		\nodepart[]{one} $ S(x) $
		\nodepart[]{two} $ C(x)$};

	\node[
	draw=none,
	ultra thick,
	rounded corners,
	minimum height=2cm,
	minimum width=3cm,
	inputs,
	outputs, 
	fill=red!15]  (CEBOX3) at (19,7.7) {};

	\node[
		rectangle split,
		above=of CEBOX3,
		draw,
		rectangle split horizontal,
		rectangle split parts=2,
		minimum height=0.6cm,
		yshift=0.7cm,
		align=center](dz3){
		\nodepart[]{one} $ \delta_{S(b)} $
		\nodepart[]{two} $ \delta_{C(b)}$};
	
	\node[
		above=of dz3,
		rectangle split,
		draw,
		rectangle split horizontal,
		rectangle split parts=2,
		minimum height=0.6cm,
		yshift=-0.14cm,
		align=center](z3){
		\nodepart[]{one} $ \delta_{S(a)} $
		\nodepart[]{two} $ \delta_{C(a)}$};

	\node[font=\fontsize{13}{0}\selectfont] (C3) at (19,6.7) {$\lnot S(x) \lor C(x)$};
	\node[above=of C3,font=\fontsize{13}{0}\selectfont] (CE1) {\bfseries{CE} \ $c$};

	\path[draw, ->] (z2) -- (CEBOX3);

	\node[
	draw=none,
	circle,
	black!70,
	fill=black!10,
	minimum height=0.9cm,
	font=\fontsize{13}{0}\selectfont] (sum2) at (19, 19) {$+$};

	\node[
	draw=none,
	circle,
	black!70,
	fill=black!10,
	yshift=0.8cm,
	minimum height=0.9cm,
	font=\fontsize{13}{0}\selectfont] (sigma2) at (19, 20.5) {$\sigma$};

	\path[draw, ->] (CEBOX3) --  (sum2);
	\path[draw, ->] (z2) -- (19, 3) node[] (three){} -- (27.5,3) -- node[pos=0.89, anchor=west] (four){} (27.5, 19) -- (sum2);
	\path[draw, ->] (sum2) -- (sigma2);

	\node[
		rectangle split,
		draw=none,
		rectangle split horizontal,
		rectangle split parts=2,
		minimum height=0.6cm,
		align=center](y2) at (19, 26.3){
		\nodepart[]{one} $ S(x) $
		\nodepart[]{two} $ C(x)$};

	\node[
		above=of y2,
		rectangle split,
		draw,
		rectangle split horizontal,
		rectangle split parts=2,
		minimum height=0.6cm,
		text opacity=0,
		align=center](y3){
		\nodepart[]{one} $ S(x) $
		\nodepart[]{two} $ C(x)$};

	\node[
		above=of y3,
		rectangle split,
		draw,
		rectangle split horizontal,
		rectangle split parts=2,
		minimum height=0.6cm,
		text opacity=0,
		yshift=-0.14cm,
		align=center](y4){
		\nodepart[]{one} $ S(x) $
		\nodepart[]{two} $ C(x)$};

	\node[left=of y4, draw=none] {$[x/a]$};
	\node[left=of y3, draw=none] {$[x/b]$};
	\path[draw, ->] (sigma2) -- (y2);

	\begin{scope}[on background layer]
		\node[
		draw=none,
		fit=(three) (sigma2) (four) (CEBOX3),
		ultra thick,
		rounded corners,
		minimum height=6.5cm,
		minimum width=6.5cm,
		inputs,
		fill=green!15]  (GEBOX1) {};
	\end{scope}


\end{tikzpicture}}
	\caption{KE applied on a Prior Knowledge consisting of a single clause $c$: $\lnot S(x) \lor C(x)$ applied on a domain of two objects ($a$ and $b$): (a) using a vector with all the grounded atoms, each grounding of $c$ is managed by a different CE.  Notice that the changes proposed by the first CE are zero for $C(b)$ and $S(b)$ while the second CE returns the value zero for the other two atoms: the two grounded clauses are independent; (b) using a matrix with a row for each grounding: all the grounding of $c$ are managed by the same CE, which operates in parallel on each row of the matrix. This is more convenient when dealing with mini-batches and results in a more efficient implementation.}
	\label{fig:KE_matrix}
\end{figure}

\subsection{KE for relational data}

Maintaining a single vector with all the grounded atoms as in Section~\ref{sec:architecture} could lead to an inefficient implementation since it forces to create multiple CEs for each clause. This means that it requires to reset the computational graph every time new data is provided. On the other hand, in the approach of Fig.~\ref{fig:KE_matrix}(b), for each clause only a CE is instantiated and it can work in parallel on all the groundings of the clause. This has two advantages: it simplifies the usage of batches during training, and CEs' internal calculations are implemented as matrix operations, which are particularly efficient when working with a GPU. 

However, such an approach can be applied with clauses that involve a single variable: we call this type of clauses \emph{unary}. In contrast, a clause that contains two variables is referred as a \emph{binary} clause. For simplicity, we don't take into account clauses with higher arity. Notice however that the proposed approach can be used with predicates and clauses with any number of variables.




Let $\clauses_U$ be the set of unary clauses and $\clauses_B$ the set of binary clauses. The Prior Knowledge is now defined as $\clauses = \clauses_U \cup \clauses_B$.

The idea now is to apply the KE to these two sets separately: Equation~\ref{eq:final_intepretation} can be decomposed using the new defined partition of the knowledge:
\begin{equation}
	\begin{split}
	\by'&= \sigma \left(\bz + \sum_{c \in \G(\clauses_U)} \bdelta_c + \sum_{c \in \G(\clauses_B)} \bdelta_c \right)	
	\end{split}
	\label{eq:final_intepretation_binary}
\end{equation}

Notice that the approach defined in the previous section with a single CE for each clause can be directly applied to the unary knowledge $\clauses_U$. We need to define a strategy to deal with binary clauses. 



\subsubsection{Representing relational data in tables}
\label{sec:relational_representation}

To represent relational data, \kenn\ adopts two matrices, denoted by $\U$ and $\B$, which could be seen as tables of a relational database. Fig.~\ref{fig:representation} shows this representation using the classical Smoker-Friends-Cancer example: there is a domain composed of three objects (persons) $O = \{o_0, o_1, o_2\}$, two unary predicates ($S$ and $C$, which stand for $Smoking$ and $Cancer$) and one binary predicate $F$ ($Friends$)). Fig.~\ref{fig:representation}(a) shows the graph representation with nodes and edges labelled with the unary and binary predicates respectively. In this particular example the graph has no self-loops, meaning that $Friends(x,x)$ is not taken into account. Please notice that this limitation is specific for this example since everyone is assumed to be friends of himself a priori. However, in general, \kenn\ can work with self-loops.

Fig.~\ref{fig:representation}(b) shows the data structure used to encode the graph of Fig.~\ref{fig:representation}(a): matrix $\U$ contains as many rows as objects in the domain and as many columns as unary predicates plus a key column $i$, which could be interpreted as a primary key for the table $\U$. Intuitively, 
the $j^{th}$ row of $\U$ contains the preactivations of the unary predicates applied on the $j^{th}$ object,
i.e., $\U_j=(z_{P_1(o_j)},\dots,z_{P_n(o_j)})$. 

Matrix $\B$ contains one row for every pair of objects to which we want to apply
a binary predicate. 
Notice that we might not want to include all the pairs of objects in $\B$, because the application of binary predicates of some pairs might be meaningless. For instance, as already mentioned, the $Friends$ predicate applied to $(a,a)$ is not interesting.
$\B$ has one column for each binary predicate and other two columns $s^x$ and $s^y$ that indicate the first and the second component of the
pair. In the relational database analogy they can be seen as foreign keys to table $\U$. As with $\U$, each row represents a different grounding of the predicates, with the difference that in this case the predicates are binary and the variables to be substituted are two. More precisely, $\B_{i,j}$ corresponds to the value of predicate $B_j$ for the substitution $(x / o_{s^x_i}, y / o_{s^y_i})$:
$$\B_{i,j} = B_j(o_{s^x_i},o_{s^y_i})$$

For instance, in Fig.~\ref{fig:representation}(b): $s^x_0 = 0$ and $s^y_0 = 1$, which means that row $0$ of matrix $\B$ is referring to the pair of objects $(o_0,o_1)$.

\begin{figure}
	\centering
	\input{tikz/rel_representation}
    \caption{
		Relational data representation: in this example, three objects ($o_0$, $o_1$ and $o_2$), two unary predicates ($S$ and $C$) and only one binary predicate ($F$). (a): the graph representation; (b): the representation used by \kenn: matrixes $\U$ and $\B$ can be interpreted as tables of a relational database: $\U$ contains all the groundings of unary predicates and an index column $i$ which correspond to the `primary key' of the table; matrix $\B$ contains the groundings of binary predicates together with columns $s^x$ and $s^y$ which corresponds to ``foreign keys''  to table $\U$.
		Values refer to preactivations of atoms' truth values.}
    \label{fig:representation}
\end{figure}

\def\alice{a}
\def\bob{b}
\def\chris{c}

\kenn\ considers unary and binary clauses separately, by adopting two separate instantiations of the KE architecture, namely
KE$_u$ and KE$_b$. The former computes the contributionsof the clauses in $\clauses_u$ and the latter computes the contributions of the clauses in $\clauses_b$.
The structure of KE$_u$ and KE$_b$ 
is the same (shown in the right of Figure~\ref{fig:KE_matrix}). The only difference between the two is that KE$_u$
takes in input the preactivations on unary predicates in $\U$,
while KE$_b$ takes in input the preactivations of both unary and binary predicates. Indeed, binary clauses can contain unary predicates and
we need a method to encode unary predicates as binary.

\subsubsection{Binary extensions of unary predicates}
\label{sec:join}
Given a unary predicate $P$, we define its binary extensions as two binary predicates $P^x$ and $P^y$
and impose the constraints: 

  \begin{align}
    \begin{array}{c}
P^x(x,y) \leftrightarrow P(x) \\
      P^y(x,y) \leftrightarrow P(y)
    \end{array}
    \label{eq:binarized-unary-predicates}
  \end{align}
Intuitively, $P^x$ and $P^y$ are binary predicates that ignore the second and the first input, respectively. For instance, let $Smoker(x)$ be the unary predicate that is true if person $x$ smokes. The binary atom $Smoker^x(x,y)$ corresponds to the sentence: ``The first element of the pair $(x,y)$ is a smoker''.

For each binary clause, we can substitute the original unary predicates with their binary extensions. For instance, the clause
$$c: \lnot Smoker(x) \lor \lnot Friends(x,y) \lor Smoker(y)$$
is converted to
$$c': \lnot Smoker^x(x,y) \lor \lnot Friends(x,y) \lor Smoker^y(x,y)$$

It is easy to prove that the truth value of any grounding of $c'$ is the same as the corresponding grounding of $c$:
$$\Big( Smoker^x(\alice,\bob) \leftrightarrow Smoker(\alice) \Big) \land \Big(Smoker^y(\alice,\bob) \leftrightarrow Smoker(\bob)\Big) \models c[x/\alice,y/\bob] \leftrightarrow c'[x/\alice,y/\bob]$$
After this transformation, binary clauses can be seen as unary clauses on the domain of pairs of objects on which we can apply KE$_b$. \




\begin{figure}[t]
	\centering
	\begin{tikzpicture}[scale=0.45, node distance=1mm, align=left]
    
    \node[draw=none, font=\fontsize{12}{0}\selectfont] at (-4,-5) {$\U$};
    
    \matrix [matrix of nodes,column sep=-2.5*\pgflinewidth, row sep=-3*\pgflinewidth, hlabel/.style={draw=none},align=center, label/.style={draw=none, font=\fontsize{7}{0}\selectfont}, nodes={rectangle,draw,minimum width=1.8em, font=\fontsize{9}{0}\selectfont}] (U) at (-4, -8.5)
    {
    |[hlabel]| $i$ & |[hlabel]|$S$   & |[hlabel]|$C$ \\ 
    0 & 0   & {\color{red}-3} \\ 
    1 & 3   & \ 1 \\ 
    2 & 2   & -1 \\ 
    };

    \node[draw=none, font=\fontsize{12}{0}\selectfont] at (4,-5) {$\B$};
    
    \matrix [matrix of nodes,column sep=-3*\pgflinewidth, row sep=-1.8*\pgflinewidth, hlabel/.style={draw=none}, label/.style={draw=none, font=\fontsize{7}{0}\selectfont}, nodes={rectangle,draw,minimum width=1.8em},yshift=-2] (B) at (4, -10.5)
    {
    |[hlabel]| $s^x$ & |[hlabel]| $s^y$ & |[hlabel]|$F$ \\
    0 & 1 & -2 \\
    0 & 2 & \ 1 \\
    1 & 0 & \ 3 \\
    1 & 2 & -1 \\
    2 & 0 & \ 0 \\
    2 & 1 & \ 5 \\
    };
    
    \node[
        below=of B,
        draw=none,
        ultra thick,
        rounded corners,
        minimum height=30,
        minimum width=40, 
        fill=blue!15,
        yshift=-20] (join) {JOIN};

    \path[->, draw] (U) |- (join);
    \path[->, draw] (B) -- (join);
 
    \matrix [matrix of nodes,right=of join, xshift=20,align=center,column sep=-\pgflinewidth, row sep=-2*\pgflinewidth, hlabel/.style={draw=none}, label/.style={draw=none, font=\fontsize{7}{0}\selectfont, text width=2.5em}, nodes={rectangle,draw,minimum width=2em,text width=1em}] (M)
    {
        |[hlabel]|$s^x$ & |[hlabel]|$s^y$ & |[hlabel]|$S^x$   & |[hlabel]|$C^x$ & |[hlabel]|$S^y$   & |[hlabel]|$C^y$  &|[hlabel]|$F$ \\
                0 & 1 & 0  & {\color{red}-3} &  3  & \ 1  & -2 \\
                0 & 2 & 0   & {\color{red}-3} &   2   & -1  & \ 1 \\
                1 & 0 & 3   & \ 1 &  0   & {\color{red}-3} & \ 3 \\
                1 & 2 & 3   & \ 1 & 2   & -1 \ & -1 \\
                2 & 0 & 2   & -1 \ & 0   & {\color{red}-3}  & \ 0 \\
                2 & 1 & 2   & -1 &  3   & \ 1  & \ 5 \\
    };

    \node[draw=none, fit=(M-1-2) (M-7-6)] (wrapper) {};
    \node[draw=none, above=of wrapper, font=\fontsize{12}{0}\selectfont] {$\M$};

    \matrix [matrix of nodes,align=center,below=of M.south east,anchor=north east,yshift=-40,column sep=-\pgflinewidth, row sep=-3*\pgflinewidth, hlabel/.style={draw=none, font=\fontsize{9}{0}\selectfont}, label/.style={draw=none, font=\fontsize{7}{0}\selectfont, text width=2.5em}, nodes={rectangle,draw,minimum width=2em,text width=1em},yshift=-30] (dM)
    {
        |[hlabel]|$s^x$ & |[hlabel]|$s^y$ & |[hlabel]|$\delta S^x$   & |[hlabel]|$\delta C^x$ & |[hlabel]|$\delta S^y$   & |[hlabel]|$\delta C^y$  &|[hlabel]|$\delta F$ \\
                0 & 1 & 1  & {\color{red}-1} &  1 &  0 &  2 \\
                0 & 2 & 0  & {\color{red}1} &  1 &  2 & -1 \\
                1 & 0 & 3  &  0 & -1 &  {\color{red}0} &  0 \\
                1 & 2 & 3  &  3 &  2 & -1 &  0 \\
                2 & 0 & 1  &  1 &  2 & {\color{red}-2} &  1 \\
                2 & 1 & 0  & -2 &  0 &  1 &  0 \\
    };

    \node[draw=none, fit=(dM-1-2) (dM-7-6)] (wrapper2) {};
    \node[draw=none, below=of wrapper2, font=\fontsize{12}{0}\selectfont] {$\delta \M$};

    \node[
        below=of wrapper.south,
        draw=none,
        rounded corners,
        minimum height=30,
        minimum width=40, 
        fill=green!15,
        yshift=-20] (KE) {KE};
    
    \node[
        draw=none,
        left=of KE,
        xshift=-20
    ] (k) {$\clauses_B$};

    \path[draw,dashed,->] (k) -- (KE);
    \path[draw, ->] (join) -- (M);
    \path[draw, ->] (wrapper) -- (KE);
    \path[draw, ->] (KE) -- (wrapper2);
\end{tikzpicture}
    \caption{The JOIN query: all the binary predicates (including the binary extensions) are collected in a unique matrix. The red numbers are all referring to the same grounded atom $C(o_0)$.}
    \label{fig:join}
\end{figure}
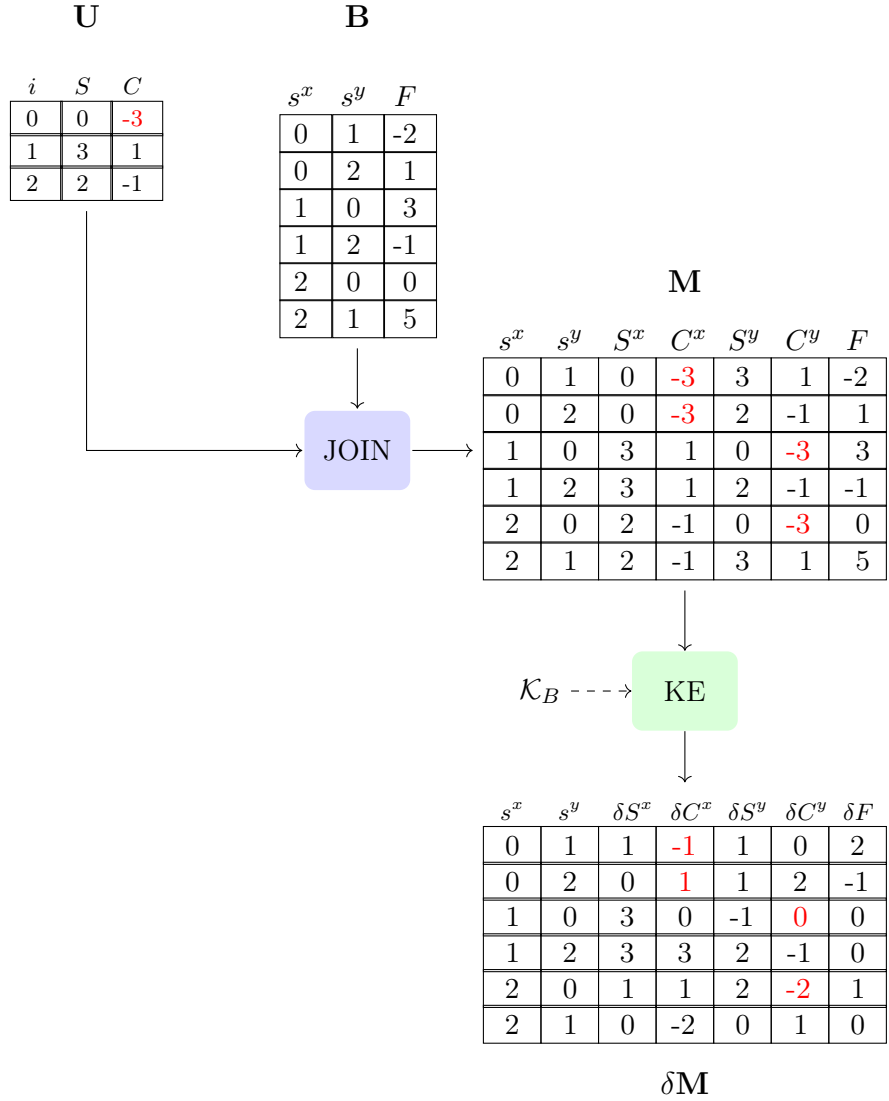


\subsection{Relational Data - \kenn\ architecture}
\label{sec:groupby}

To cope with ``binarized'' unary atoms, table $\B$ is extended with additional columns for the predicates $P^x_i$ and $P^y_i$ for every unary predicate $P_i$. In order to satisfy the constraints in \eqref{eq:binarized-unary-predicates} the values of these new attributes are computed by joining $\B$ and $\U$ on the indexes $i$, $s^x$ and $s^y$. We denote with $\M$ the result of this join:
\begin{lstlisting}[language=SQL, escapeinside={*}{*}]
	*$\M$* = SELECT *$s^x$*, *$s^y$*, *$U_i^x$*, *$U_i^y$*, *$B_j$* 
	     FROM U AS Ux, U AS Uy, B
	     WHERE Ux.*$i$* = B.*$s^x$* AND Uy.*$i$* = B.*$s^y$*
\end{lstlisting}
where $U_i^x$ is a shortcut for selecting all the unary predicates in Ux
(Ux.$U_i$ AS $U_i^x$).

{
The matrix $\M$ defined by the JOIN query is used by KE$_b$ with $\clauses_B$. KE$_b$ returns $\delta \M$, a table of the same shape (and same values for $s^x$ and $s^y$) of $\M$. Such table contains the changes on the initial predictions induced by binary clauses. Figure~\ref{fig:join} shows an example of the application of join and KE$_b$.}

Finally, the different delta values associated with each grounded atom are aggregated together. This can be achieved by the following three queries:

\begin{lstlisting}[language=SQL, escapeinside={*}{*}]
	*$\delta B$* = SELECT *$\delta B_i$*
	     FROM *$\delta M$*
	
	*$\delta U_x$* = SELECT SUM(*$\delta U_i^x$*)
	      FROM *$\delta M$*
	      GROUP BY *$s^x$*

	*$\delta U_y$* = SELECT SUM(*$\delta U_i^y$*)
	      FROM *$\delta M$*
	      GROUP BY *$s^y$*
\end{lstlisting}

\noindent Figure~\ref{fig:groupby} shows the second query, which is used to calculate the $\delta U_x$. 

\begin{figure}[H]
	\centering
	\begin{tikzpicture}[scale=0.45, node distance=1mm, align=left]
    \tikzset{every node}=[font=\fontsize{10}{0}\selectfont];

    \matrix [matrix of nodes,align=center,below=of M.south east,anchor=north east,yshift=-40,column sep=-\pgflinewidth, row sep=-3*\pgflinewidth, hlabel/.style={draw=none, font=\fontsize{9}{0}\selectfont}, label/.style={draw=none, font=\fontsize{7}{0}\selectfont, text width=2.5em}, nodes={rectangle,draw,minimum width=2em,text width=1em},yshift=-30] (dM)
    {
        |[hlabel]|$s^x$ & |[hlabel]|$s^y$ & |[hlabel]|$\delta S^x$   & |[hlabel]|$\delta C^x$ & |[hlabel]|$\delta S^y$   & |[hlabel]|$\delta C^y$  &|[hlabel]|$\delta F$ \\
                0 & 1 & 1  & -1 &  1 &  0 &  2 \\
                0 & 2 & 0  & 1 &  1 &  2 & -1 \\
                1 & 0 & 3  &  0 & -1 &  0 &  0 \\
                1 & 2 & 3  &  3 &  2 & -1 &  0 \\
                2 & 0 & 1  &  1 &  2 & -2 &  1 \\
                2 & 1 & 0  & -2 &  0 &  1 &  0 \\
    };

    \node[
        right=of dM,
        draw=none,
        ultra thick,
        rounded corners,
        minimum height=30,
        minimum width=40, 
        fill=blue!15,
        align=center,
        xshift=10] (gb) {GROUP BY};

    \matrix [matrix of nodes,right=of gb,align=center,column sep=-\pgflinewidth, row sep=-3*\pgflinewidth, hlabel/.style={draw=none, font=\fontsize{9}{0}\selectfont,opacity=.0}, label/.style={draw=none, font=\fontsize{7}{0}\selectfont,opacity=.0}, nodes={rectangle,draw,minimum width=2em,text width=1em,fill=white,opacity=0.6, text opacity=1.},xshift=10] (r)
    {
        |[hlabel]|$s^x$ &|[hlabel]|$\delta S$ &|[hlabel]|$\delta C$ \\
        0 & 1  &  0\\
        1 & 6  &  3 \\
        2 & 1  & -1 \\
    };

    \path[->, draw] (dM) -- (gb);
    \path[->, draw] (gb) -- (r);
\end{tikzpicture}
	\caption{An example of a GROUP BY query: for each unary predicate ($S$ and $C$), the values of their first binary extension ($S^x$ and $C^x$) are summed up based on the index 
	}
    \label{fig:groupby}
\end{figure}
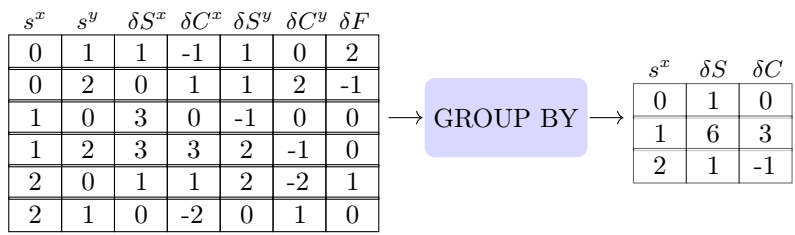

\noindent The final preactivations $\U'$ and $\B'$ are
$$\U' = \U + \delta \U_u + \delta \U_x + \delta \U_y$$
and
$$\B' = \B + \delta \B$$


\noindent Figure~\ref{fig:unary_binary} shows a high-level overview of the entire model.


\begin{figure}
	\centering
	\definecolor{KE}{rgb}{0.1,0.6,0.2}
\definecolor{mappings}{rgb}{0.9,0.2,0.2}

\tikzstyle{block} = [rectangle, draw, fill=black!5, text centered, rounded corners, minimum height=1cm, minimum width=1.8cm, align=center, font=\fontsize{8}{0}\selectfont]

\definecolor{KE}{rgb}{0.1,0.6,0.2}
\definecolor{mappings}{rgb}{0.9,0.2,0.2}

\tikzstyle{block} = [rectangle, draw, fill=black!5, text centered, rounded corners, minimum height=1cm, minimum width=1.8cm, align=center, font=\fontsize{8}{0}\selectfont]

\begin{tikzpicture}[scale=0.4, node distance=1mm, align=left]

    \node[draw=none] (U) at (0,0) {$\U$};
    \node[draw=none, below=of U, yshift=-132] (B) {$\B$};

    \node[
        right=of U,
        draw=none,
        ultra thick,
        rounded corners,
        minimum height=30,
        minimum width=40, 
        fill=green!15,
        xshift=60] (KEu) {KE};

    \node[
        right=of KEu,
        draw=none,
        xshift=35] (dUu) {$\delta \U_u$};
    
    \node[
        draw,
        circle,
        right=of dUu,
        xshift=35,
        font=\fontsize{8}{0}\selectfont] (sum) {$+$};

    \node[
        right=of sum,
        draw=none,
        xshift=35] (dUb) {$\delta \U_b$};

    \node[
        draw,
        below=of dUb,
        circle,
        yshift=-10,
        font=\fontsize{8}{0}\selectfont] (sum2) {$+$};
    
    \node[
        draw=none,
        below=of sum2,
        minimum height=23,
        yshift=-5,
        xshift=-40
    ] (dUx) {$\delta \U_x$};

    \node[
        draw=none,
        below=of sum2,
        minimum height=23,
        yshift=-5,
        xshift=40,
    ] (dUy) {$\delta \U_y$};

    \node[
        below=of dUx,
        draw=none,
        ultra thick,
        rounded corners,
        minimum height=30,
        minimum width=40, 
        fill=blue!15,
        align=center,
        yshift=-15] (rmx) {GROUP BY};
    
    \node[
        below=of dUy,
        draw=none,
        ultra thick,
        rounded corners,
        minimum height=30,
        minimum width=40, 
        fill=blue!15,
        align=center,
        yshift=-15] (rmy) {GROUP BY};

    \node[
        below=of dUb,
        draw=none,
        ultra thick,
        rounded corners,
        minimum height=30,
        minimum width=40, 
        fill=blue!15,
        align=center,
        yshift=-154] (rm) {SELECT};
    
    \node[draw=none, fit=(rmx) (rmy) (rm)] (queries_wrapper) {};

    \node[
        right=of rm,
        draw=none,
        xshift=70] (dB) {$\delta \B$};

    \node[
        above=of dB,
        yshift=190,
        draw=none] (dU) {$\delta \U$};

    \node[
        left=of queries_wrapper,
        draw=none,
        xshift=5] (dM) {$\delta \M$};

    \node[
        left=of dM,
        draw=none,
        ultra thick,
        rounded corners,
        minimum height=30,
        minimum width=40, 
        fill=green!15,
        xshift=-15] (KEb) {KE};
    
    \node[
        left=of KEb,
        draw=none,
        xshift=-15] (M) {$\M$};
    
    \node[
        left=of M,
        draw=none,
        ultra thick,
        rounded corners,
        minimum height=30,
        minimum width=40, 
        fill=blue!15,
        xshift=-15] (m) {JOIN};
    
    \node[draw=none, fit=(m) (rm)] (wrapper) {};


    \begin{scope}[on background layer]
        \node[
            fit=(m) (KEu) (rmy) (rm),
            draw=none,
            rounded corners,
            minimum height=240,
            minimum width=360,
            yshift=8, 
            fill=black!3] (kenn) {};
    \end{scope}

    \path[draw, ->] (U) -- (KEu);
    \path[draw, ->] (KEu) -- (dUu);
    \path[draw, ->] (dUu) -- (sum);
    \path[draw, ->] (sum) |- (dU);
    \path[draw, ->] (dUb) -- (sum);
    \path[draw, ->] (rmx) -- (dUx);
    \path[draw, ->] (rmy) -- (dUy);
    \path[draw, ->] (dUx) |- (sum2);
    \path[draw, ->] (dUy) |- (sum2);
    \path[draw, ->] (sum2) -- (dUb);
    \path[draw, ->] (rm) -- (dB);
    \path[draw, ->] (dM) -| (rmx);
    \path[draw, ->] (dM) -| (rmy);
    \path[draw, ->] (dM) -| (rm);
    \path[draw, ->] (KEb) -- (dM);
    \path[draw, ->] (M) -- (KEb);
    \path[draw, ->] (m) -- (M);
    \path[draw, ->] (U) -- (U|-m.170) |- (m.170);
    \path[draw, ->] (B) -- (m.west |- B);

    \node[below=of KEu, draw=none, yshift=-15] (Ku) {$\clauses_u$};
    \node[above=of KEb, draw=none, yshift=15] (Kb) {$\clauses_b$};
    
    \draw[dashed, ->] (Ku) to (KEu);
    \draw[dashed, ->] (Kb) to (KEb);
    
\end{tikzpicture}
	\caption{
	To obtain $\delta \U_u$, unary clauses $\clauses_U$ are used by KE directly on $\U$. The JOIN query is used to find matrix $\M$. Then, the KE is applied to $\M$ using binary clauses $\clauses_B$. Finally, to obtain $\delta \U_b$ and $\delta \B$ other three queries are used.}
    \label{fig:unary_binary}
\end{figure}
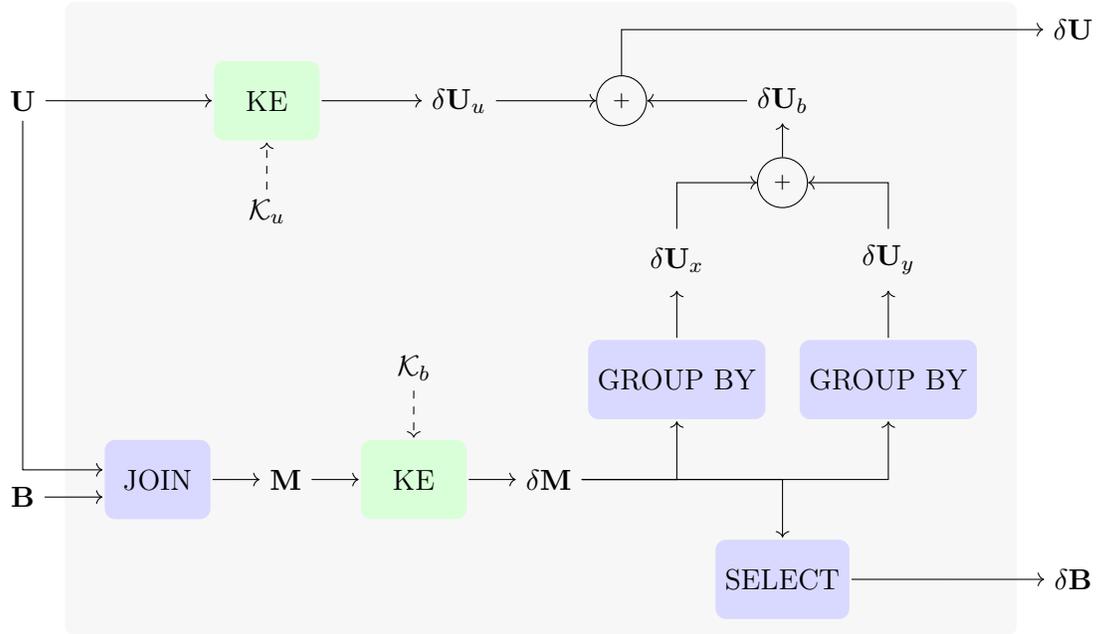

\subsection{Knowledge for inputs/outputs relationships}
\label{sec:KENN_inputs}
Until now, we defined \kenn\ as a method that injects logical knowledge which represents relationships between different predictions of a Neural Network. However, it is even possible to provide knowledge which expresses relationships between inputs and outputs values. In other words, if some of the inputs are truth values of some atomic formulas, we can add clauses that contain them, providing in this way information on the relations between inputs and outputs. 

Figure~\ref{fig:KENN_inputs} introduces this idea: the solution is straightforward, we just need to concatenate the inputs of the base NN with its outputs before applying the KE. The KE will provide an updated version of both inputs ($\bx'$) and outputs ($\by'$). Since we are interested only in the outputs, the last step discards the changed inputs $\bx'$.

\begin{figure}[H]
	\centering
	\definecolor{predicates}{rgb}{0.9,0.2,0.2}
\definecolor{clauses}{rgb}{0.3,0.3,0.7}
\definecolor{inputs}{rgb}{0.1,0.6,0.2}
\definecolor{outputs}{rgb}{0.9,0.2,0.2}

\tikzstyle{block} = [rectangle, draw, fill=green!8, text centered, rounded corners, minimum height=2cm, minimum width=3cm, align=center, font=\fontsize{16}{0}\selectfont]
\tikzstyle{block2} = [rectangle, draw, fill=green!8, text centered, rounded corners, minimum height=2.4cm, minimum width=4cm, align=center, font=\fontsize{8}{0}\selectfont, node distance=1cm]
\scalebox{0.65}{
		\begin{tikzpicture}[scale=0.5, node distance=5cm, align=left, font=\fontsize{16}{0}\selectfont]
            \node[
                draw=none,
            ] (x) at (0, 4.5) {$\bx$};

            \node[
                draw=none,
                rounded corners,
                minimum height=1.5cm,
                minimum width=2.5cm,
                black!35,
                fill=black!5
            ] (NN) at (7,0){};
            
            \node[
                text centered, font=\fontsize{16}{0}\selectfont, align=center
            ] at (7,0) {NN};

            \node[
                draw=none,
            ] (zy) at (13.5,0) {$\bz_y$};

            \node[
                draw=none
            ] (epsilon) at (4,13) {$\epsilon$};

            \node[
                circle,
                draw,
                minimum height=1.3cm
            ] (pm) at (4,9) {$\pm$};

            \node[
                circle,
                draw,
                minimum height=1.3cm
            ] (logit) at (9,9) {$\sigma^{-1}$};

            \node[
                draw=none,
            ] (zx) at (13.5,9) {$\bz_x$};
            
            \node[
                draw,rounded corners, 
                minimum height=1.5cm,
                minimum width=2.6cm
            ] (concat) at (19,4.5){};
            
            \node[
                text centered, font=\fontsize{16}{0}\selectfont, 
                align=center,
                minimum height=1.5cm,
                minimum width=2.6cm
            ] at (19,4.5) {concat};

            \node[
                draw=none,
            ] (zz) at (25.5,4.5) {$[\bz_x,\bz_y]$};

            \node[
                draw=none,
                ultra thick,
                rounded corners,
                minimum height=1.5cm,
                minimum width=2.5cm,
                inputs,
                fill=green!15] (KE) at (32,4.5){};
                \node[text centered, font=\fontsize{16}{0}\selectfont, align=center] at (32,4.5) {KE};
            
            \node[
                draw=none,
            ] (kb) at (27,-1) {$\clauses$};

            \node[
                draw=none,
            ] (xy) at (38.5,4.5) {$[\bx',\by']$};

            \node[
                draw=none,
            ] (y) at (43,4.5) {$\by'$};

            \path[draw, ->] (x) |- (NN);
            \path[draw, ->] (NN) -- (zy);
            \path[draw, ->] (x) -| (pm);
            \path[draw, ->] (epsilon) -- (pm);
            \path[draw, ->] (pm) -- (logit);
            \path[draw, ->] (logit) -- (zx);
            \path[draw, ->] (zy) |- (concat);
            \path[draw, ->] (zx) |- (concat);
            \path[draw, ->] (concat) -- (zz);
            \path[draw, ->] (zz) -- (KE);
            \path[draw, ->, dashed] (kb)  .. controls (31, 0) and (32, 2) ..  (KE.south);
            \path[draw, ->] (KE) -- (xy);
            \path[draw, ->] (xy) -- (y);
        \end{tikzpicture}}

%
%
%
%
%
	\caption{KENN used with knowledge on both inputs and predictions: the ``preactivations'' $\bz_x$ of inputs $\bx$ are calculated with the logit function and concatenated to the preactivations of the outputs $\by$. The result can be used by KE to force the satisfaction of the knowledge.}
	\label{fig:KENN_inputs}
\end{figure}
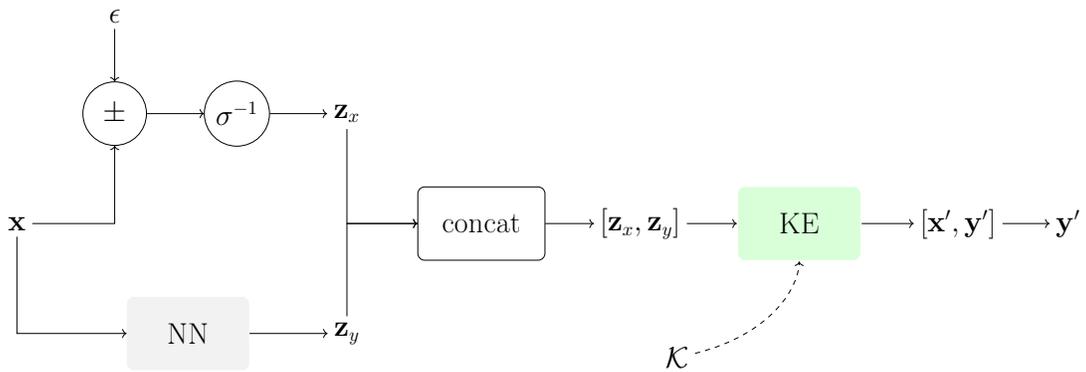

Notice that, since the KE works with preactivations, the logit function (the inverse of logistic function) must be applied to the inputs before the concatenation. Moreover, the logit is defined in the range $(0,1)$, which means that if the truth values in $\bx$ are zeros and ones, logit can not be applied directly. For this reason, a small number $\epsilon$ is added or subtracted to $\bx$ before applying the logit. Alternatively, $\bz_x$ can be set to a high value when $\bx$ is positive and a very high negative number when it is negative (in our experiments we used 500 and -500).

\section{Related Work}
\label{sec:rel_work_2}

In this section we will present 
two 
Neural-Symbolic systems and provide a comparison with \kenn. More specifically, the 
two methods are SBR and RNM.
These approaches are the most relevant in the context of this paper since they can combine general neural networks models with FOL knowledge and they can be applied in similar contexts of \kenn. For this reason, it is particularly relevant to provide a comparison of \kenn\ with them. More specifically, the comparison between \kenn \ 
and SBR is relevant since it can highlight the pros and cons of adding the knowledge through the Loss function as opposed to injecting it into the model structure. On the other hand, RNM follows a similar philosophy to \kenn, but with a different choice on the way the knowledge is implemented inside the model. For these reasons, the 
tow 
methods will be further compared with \kenn\ later, but in terms of empirical results.

\subsection{Injecting knowledge using the Loss function}
To combine logical knowledge with machine learning approaches there are two common strategies: incorporate the knowledge by including it in the Loss function or introducing it inside the model. In the first category, one of the most prominent works is Semantic Based Regularization (SBR).

SBR is based on the concept of constraint. The idea is that, given a set of logical rules, those rules induce constraints on the acceptable outputs of the model. SBR integrate those rules into the learning framework by regularizing the Loss function with an additional term that penalizes solutions that do not satisfy the constraints. To do so, it relies on a continuous relaxation of the logical rules inducted by fuzzy operators.



\subsubsection{Semantic Based Regularization}
As already mentioned, SBR introduces the knowledge during learning through the usage of a regularization term which increases when the constraints are not satisfied. The satisfaction of a constraint is calculated using a fuzzy generalization of the logic operators which is continuous and differentiable. The regularization term has the following form:
$$
R(f) = \sum_{h=1}^{H}{\lambda_h (1 - \phi_h(f))}
$$
where $H$ is the number of constraints, $\lambda_h$ is the weight associated to the $h^{th}$ constraint, $f$ is the vector of functions that represent the predicates (these are learned) and $\phi_h(f)$ is the level of satisfaction of the $h^{th}$ constraint.

Notice that, since the weights $\lambda_h$ are part of the Loss function, there is no way to let the back-propagation algorithm learn them and for this reason, they are assumed to be known a priori. 
This is one of the major drawbacks of these kinds of methods as opposed to strategies that, like \kenn, directly encode the knowledge into the model. Indeed, not always the final user of the method knows in advance the importance of a specific logical rule, and in some cases, some rule could be not correct. Moreover, by allowing rules' weights to be learnable instead of being hyper-parameters, it is in theory possible to incorporate random rules and rely on the learning algorithm to select the correct one by reducing the corresponding weights. In this way, it would be possible to discover new symbolic knowledge from data.

Figure~\ref{fig:hs_sbr} shows a representation of the Hypothesis Space (HS), i.e., the set of all the possible functions representable by the model, and uses colors to represent the value of the Loss function and the regularization term on the different hypothesis. In the shown example, the HS is represented as a subset of $\Reals^2$ and there is only one logical rule. Of course, this is not intended as a realistic HS of a neural network and the goal here is just to provide intuition on the effect of the regularizer on the training process. Red color represents high values for the loss and regularization term, while green color corresponds to low loss. The goal of the training process is to find a minimum of such function. To train a Neural Network, the standard approach is to use back-propagation, an efficient implementation of gradient descent. The algorithm starts from a random solution (the black dot in the figure). At each training step, the  forward pass calculates the predictions of the current hypothesis and uses these predictions to calculate the value of the loss function. Then, in the backward pass, the gradient of the loss function with respect to the parameters of the model is calculated and used to update the parameters. After multiple steps, a local minimum is typically reached (black cross in the figure). The top left image shows a possible evolution of this procedure when applied to the original Loss function. The top right image depicts the regularization term. Finally, the bottom image shows the values of the combination of Loss and regularizer. Here, the gradient descent, starting from the same initial hypothesis of the top left image, reaches a different hypothesis which satisfies more the given constraint on the Training Set.

Finally, notice that while the constraints are enforced at training time, there are no guarantees that they will be satisfied at inference time as well since the Loss function is calculated based on the predictions on the Training Set. To obviate this problem, at training time unlabelled data can be provided as well. In this way, the learning process has more examples to learn to enforce the knowledge even at test time. Another possibility is to enforce the constraints not only during training, but also at inference time. However, notice that in this case the time complexity of inference is increased, since the back-propagation needs to be used even in this case.

\begin{figure}
    \makebox[\textwidth][c]{
        \includegraphics[scale=0.38]{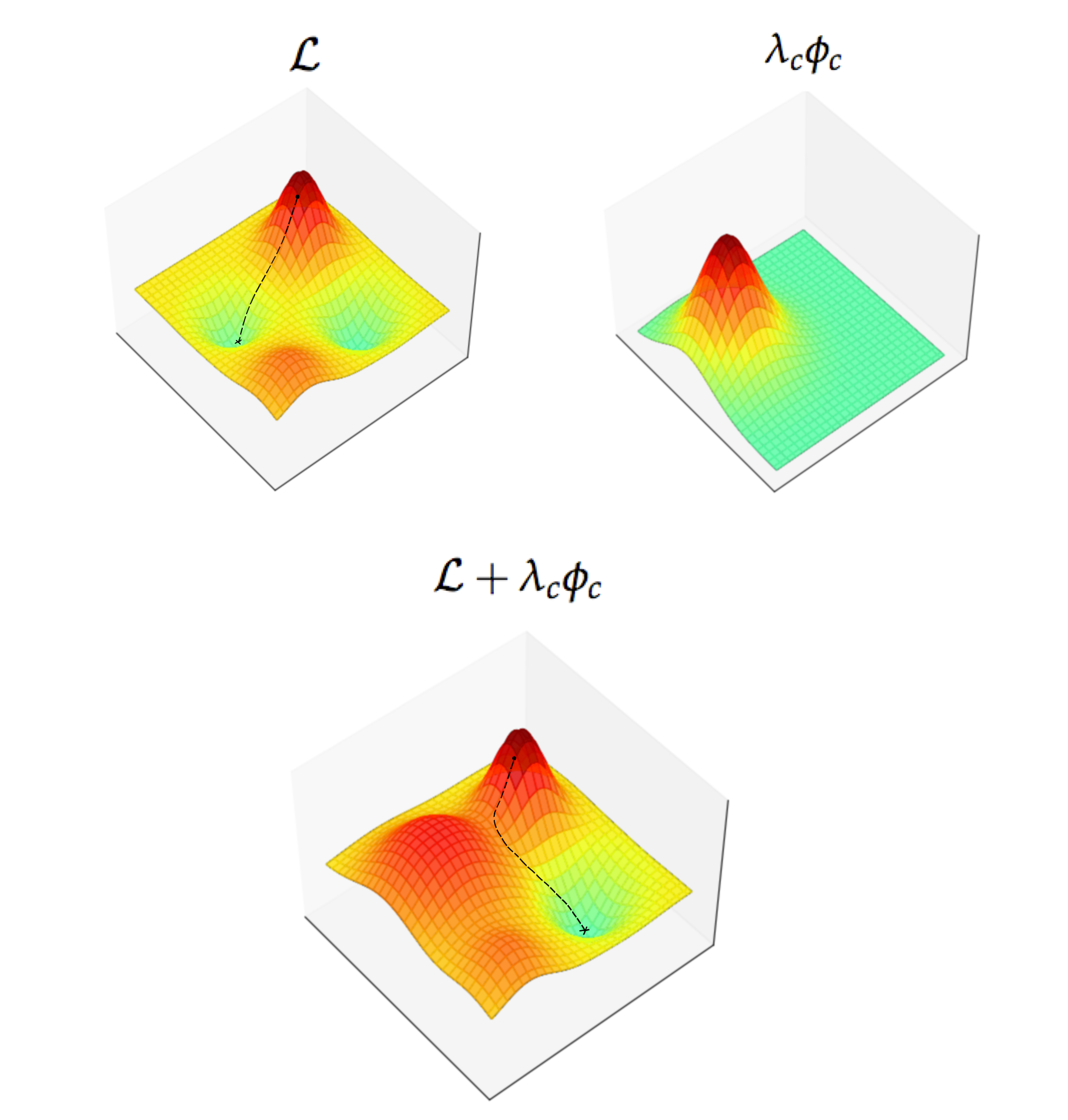}
    }
	\caption{The interaction of loss and regularization on the Hypothesis Space for SBR. On the top left, to each point of the HS it is associated a loss value. Similarly, on the top right the regularization term associates a penalty to each solution that does not satisfy a constraint $c$. The total loss, calculated by summing the two, is on the bottom. The image shows also a possible execution of the gradient descend for both the original Loss function and the regularized version. The training algorithm starts from an initial random solution (black dot) and moves inside the hypothesis space reaching a local minimum that is the final trained model. After the regularization (bottom) it is easier to reach solutions that satisfy $c$ as compared with the original Loss function (top left).}
    \label{fig:hs_sbr}
\end{figure}

\subsection{\kenn\ vs methods based on Loss}
\label{sec:kenn_vs_sbr}

\kenn\ englobes the knowledge into the network in a very different way than SBR
. As already discussed, SBR 
acts on the Hypothesis Space (HS) by changing the associated loss function according to the constraint.

The approach of 
SBR is equivalent to ``remove'' solutions from the HS that do not satisfy the constraints by penalizing them during training (they are not really removed, but difficult or impossible to reach because of the penalty given by the Loss function), while in \kenn\ the approach is the opposite: new solutions are added to the Hypothesis Space (HS) by a new additional layer, called \emph{Knowledge Enhancer}. 

For this reason, 
SBR needs to use a model that is already capable of representing functions that satisfy the constraints and a bias towards their satisfaction is introduced by penalizing the other solutions. If the HS does not contain solutions that satisfy the constraints, 
SBR can not impose their satisfaction since they are limited by the set of hypotheses in the HS.

On the other hand, \kenn\ starts from a NN with a lower capacity which is not capable to satisfy the constraints on its own and the knowledge is introduced by adding new solutions to the HS by modifying the existing ones. For this reason, \kenn\ typically does not work well with NNs that are already capable of satisfying the clauses, since it does not introduce any bias towards their satisfaction.

Summarizing, to work properly, 
methods based on Loss need a model with high capacity able to express the required Knowledge, while for \kenn\ it is the opposite. As an example, consider a Logistic Regression (LR), i.e. a neural network with no hidden layers. It is well known that with this kind of network it is not possible to represent the XOR operator $\oplus$~\citep{XOR1,XOR2}:
\begin{equation}
	\oplus(x_1,x_2) = 
	\left\{
	  \begin{array}{rl}
		0  & \qquad \mbox{if \ $x_1 = 0 \land x_2 = 0$} \\
		1  & \qquad \mbox{if \ $x_1 = 0 \land x_2 = 1$} \\
		1  & \qquad \mbox{if \ $x_1 = 1 \land x_2 = 0$} \\
		0  & \qquad \mbox{if \ $x_1 = 1 \land x_2 = 1$}
	\end{array}
	\label{eq:XOR}
	\right.
\end{equation}

Suppose we want to express with the knowledge that the target function is indeed the XOR function. 
SBR does
not change the model structure and 
it imposes
the knowledge by acting on the weights of the model. Therefore, it is not possible for it to force the satisfaction of the rule with a LR model.

On the other hand, such a goal can be achieved by \kenn\ using the architecture and knowledge shown in Figure~\ref{fig:XOR}.

The strategy is the same of Section~\ref{sec:KENN_inputs}: the input vector $\bx \in \{0,1\}^2$ is passed to a logistic regression network to obtain initial preactivations $\bz_y$. In parallel, the inverse of logistic function is applied to $\bx$ to find its ``preactivations'' $\bz_x$. By concatenating the $\bz_x$ and $\bz_y$, it is possible to use the KE to inject the clauses which represent the XOR operator's behavior.

\begin{figure}[H]
	\centering
	\definecolor{predicates}{rgb}{0.9,0.2,0.2}
\definecolor{clauses}{rgb}{0.3,0.3,0.7}
\definecolor{inputs}{rgb}{0.1,0.6,0.2}
\definecolor{outputs}{rgb}{0.9,0.2,0.2}

\tikzstyle{block} = [rectangle, draw, fill=green!8, text centered, rounded corners, minimum height=2cm, minimum width=3cm, align=center, font=\fontsize{16}{0}\selectfont]
\tikzstyle{block2} = [rectangle, draw, fill=green!8, text centered, rounded corners, minimum height=2.4cm, minimum width=4cm, align=center, font=\fontsize{8}{0}\selectfont, node distance=1cm]
\scalebox{0.60}{
		\begin{tikzpicture}[scale=0.5, node distance=5cm, align=left, font=\fontsize{16}{0}\selectfont]
            \node[
                draw=none,
            ] (x) at (4.5,0) {$\bx$};

            \node[
                draw=none,
                rounded corners,
                minimum height=1.5cm,
                minimum width=2.5cm,
                black!35,
                fill=black!5
            ] (NN) at (0,7){};
            
            \node[
                text centered, font=\fontsize{16}{0}\selectfont, align=center
            ] at (0, 7) {LR};

            \node[
                draw=none,
            ] (zy) at (0,13.5) {$\bz_y$};

            \node[
                draw=none
            ] (epsilon) at (13,4) {$\epsilon$};

            \node[
                circle,
                draw,
                minimum height=1.3cm
            ] (pm) at (9,4) {$\pm$};

            \node[
                circle,
                draw,
                minimum height=1.3cm
            ] (logit) at (9,9) {$\sigma^{-1}$};

            \node[
                draw=none,
            ] (zx) at (9,13.5) {$\bz_x$};
            
            \node[
                draw,
                minimum height=1.5cm,
                minimum width=2.8cm
            ] (concat) at (4.5,17){};
            
            \node[
                text centered, font=\fontsize{16}{0}\selectfont, align=center
            ] at (4.5, 17) {concat};

            \node[
                draw=none,
            ] (zz) at (4.5,21.5) {$[\bz_x,\bz_y]$};

            \node[
                draw=none,
                ultra thick,
                rounded corners,
                minimum height=1.5cm,
                minimum width=2.5cm,
                inputs,
                fill=green!15] (KE) at (4.5,26){};
                \node[text centered, font=\fontsize{16}{0}\selectfont, align=center] at (4.5, 26) {KE};
            
            \node[
                draw=none,
            ] (kbl) at (-8,30) {$\clauses$};

            \node[
                draw=none,
            ] (kb1) at (-8,28) {$\lnot x_1 \lor \lnot x_2 \lor \lnot y$};

            \node[
                draw=none,
            ] (kb2) at (-8,26) {$\lnot x_1 \lor \ \ \ x_2 \lor \ \ y$};

            \node[
                draw=none,
            ] (kb3) at (-8,24) {$\ \ x_1 \lor \lnot x_2 \lor \ \ y$};

            \node[
                draw=none,
            ] (kb4) at (-8,22) {$\ \ x_1 \lor \ \ x_2 \ \lor \lnot y$};
                
            \node[
                draw,
                fit=(kbl) (kb1) (kb2) (kb3) (kb4),
                dashed,
                rounded corners
            ] (wrapper) {};

            \node[
                draw=none,
            ] (xy) at (4.5,30.5) {$[\bx',\by']$};

            \node[
                draw=none,
            ] (y) at (4.5,34) {$\by'$};

            \path[draw, ->] (x) -| (NN);
            \path[draw, ->] (NN) -- (zy);
            \path[draw, ->] (x) -| (pm);
            \path[draw, ->] (epsilon) -- (pm);
            \path[draw, ->] (pm) -- (logit);
            \path[draw, ->] (logit) -- (zx);
            \path[draw, ->] (zy) |- (concat);
            \path[draw, ->] (zx) |- (concat);
            \path[draw, ->] (concat) -- (zz);
            \path[draw, ->] (zz) -- (KE);
            \path[draw, ->, dashed] (wrapper) -- (KE);
            \path[draw, ->] (KE) -- (xy);
            \path[draw, ->] (xy) -- (y);
        \end{tikzpicture}}

%
%
%
%
%
    \caption{XOR clauses applied on a LR model.}
    \label{fig:XOR}
\end{figure}
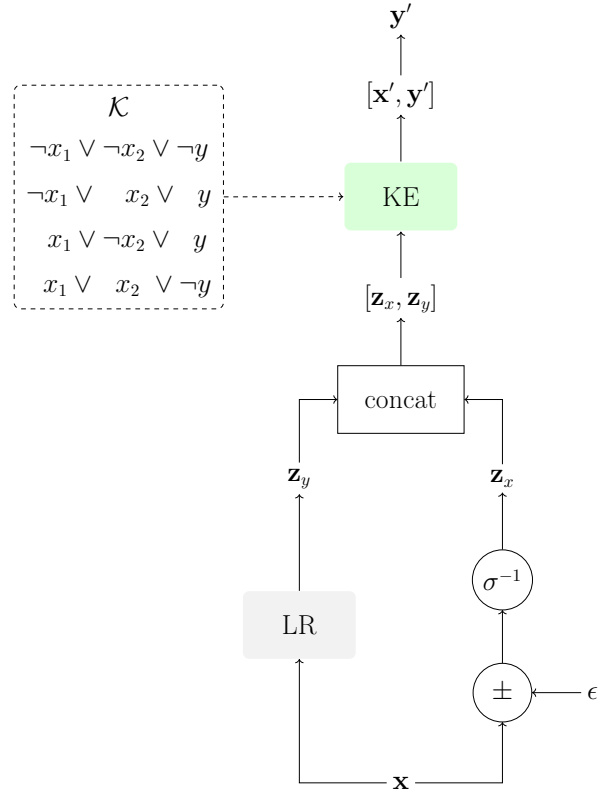

To see how this works, suppose the inputs values are $x_1 = 1$ and $x_2 = 1$. In this case, the preactivations calculated by the logit function $\sigma^{-1}$ will be very high (supposing small values of $\epsilon$). This means that preactivations for $\lnot x_1$ and $\lnot x_2$ will be very small. For this reason, the CE that enforces the first clause will increase $\lnot y$, since only the highest literal is increased. Similarly, the other three clauses will not affect the predictions of $y$. Indeed, if the clause weights have high values, the model of Figure~\ref{fig:XOR} represents the XOR function. This can be done because \kenn\ modifies the model, meaning that the limitations of LR do not apply anymore. 

Notice that the weights of the clauses can be set to zero and, as a consequence, the functions representable by LR can still be reached by the training process. Indeed, \kenn\ modifies the HS only by adding new hypotheses, not removing them. Figure~\ref{fig:hs_kenn} shows the effects of the KE layer on the HS.

\begin{figure}
    \makebox[\textwidth][c]{
        \includegraphics[scale=0.5]{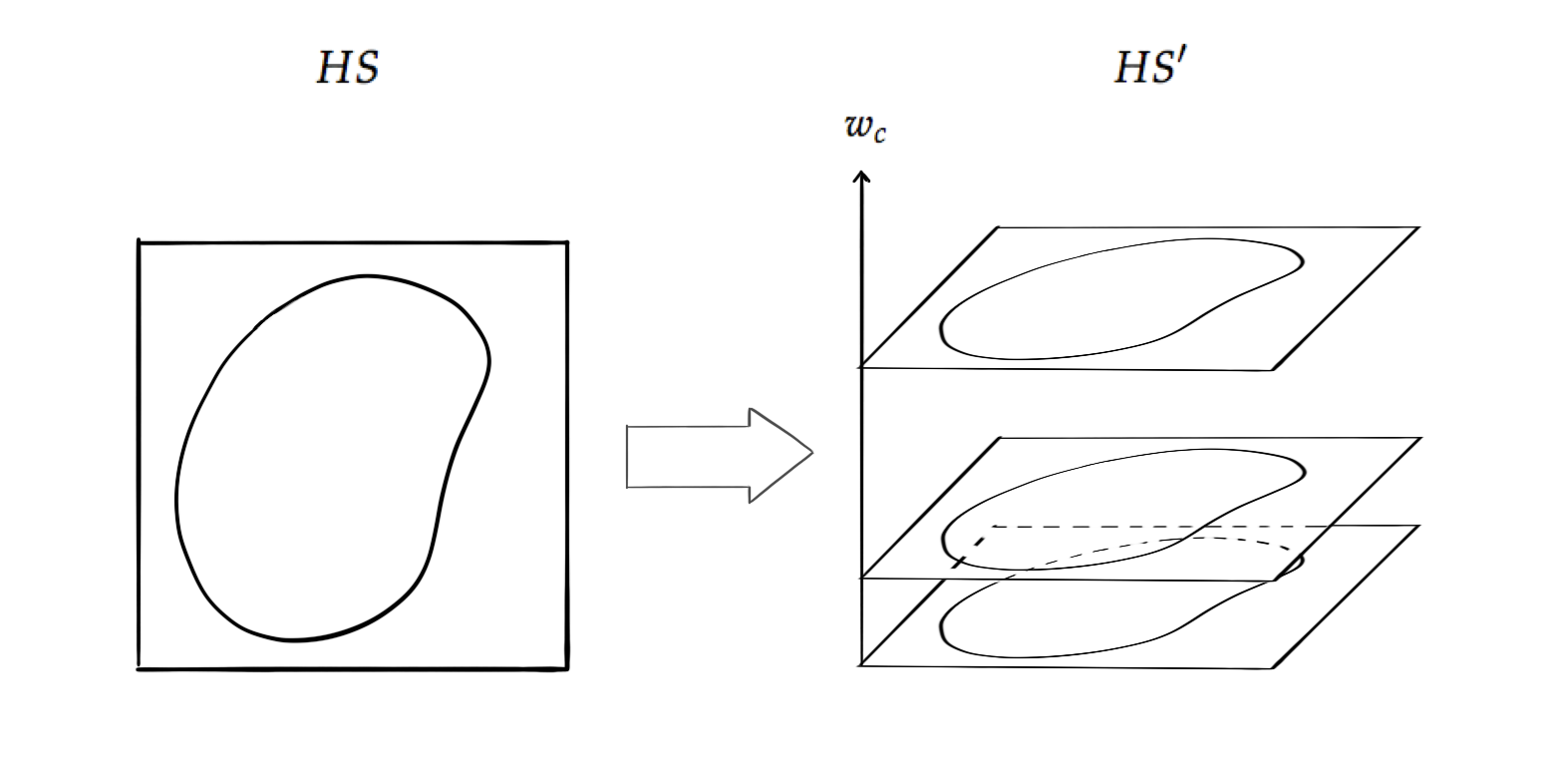}
    }
    \caption{The effect of \kenn\ on the Hypothesis Space. The left image shows the original Hypothesis Space, the right one the Hypothesis Space represented by \kenn\ model. In this picture, the original HS is represented as a portion of a plane. The effect of adding a clause $c$ is to increase the number of parameters of the model of one unit (the clause weight $w_c$). On the HS, this can be seen as adding a new dimension, where points with high values in that dimension have a higher ability to satisfy the constraints.}
    \label{fig:hs_kenn}
\end{figure}

Given a clause, the model is extended with one parameter (the clause weight), which is shown in the Figure as a new dimension in the HS. Higher values of this parameter imply higher satisfaction of the corresponding clause. The approach relies on the idea that if the clauses are satisfied in the Training Set, then the new introduced hypotheses (with clause weights greater than zero) are more capable of fitting the data since there are no solutions in the original HS that can satisfy the clauses as well as the new ones. Indeed, any additional hypothesis introduced by \kenn\ is obtained by adding a new layer that changes the output of the underlying network to increase the constraint satisfaction.


\subsection{Relational Neural Machines}

RNM is a framework that integrates a neural network model with a FOL reasoner. This is done in two stages: in the first one, a Neural Network $\mathbf{f}$ is used to calculate initial predictions for the atomic formulas; in the second stage a graphical model is used to represent a probability distribution over the set of atomic formulas. 

The distribution is defined as follow:
$$
P(\by|\mathbf{f}, \lambda) = \frac{1}{Z} exp\Big( \sum_{x \in S} \phi_0(\mathbf{f}(x), \by) + \sum_c \lambda_c \phi_c(\by) \Big)
$$
where $Z$ is the partition function, $\by$ are the grounded atoms predictions, $\mathbf{f}$ is the function codified by the Neural Network, $\phi_0$ is a potential that enforces the consistency between the predictions of the NN and the final predictions, $\lambda_c$ is the weight of constraint $c$ and $\phi_c$ is a potential that enforces the satisfaction of the constraint (higher if the constraint is satisfied).

To obtain the final predictions a Maximum a Posteriori (MAP) estimation is performed, finding the most probable assignment to the grounded atoms given the output of the Neural Network and the set of constraints:
$$
\by^* = \argmax_{y} P(\by|\mathbf{f}, \lambda)
$$

\subsection{Comparison with \kenn}

At a high-level RNM approach is similar to \kenn, since in both cases a Neural Network makes initial predictions and a post elaboration step is applied on such predictions to provide the final classification. However, RNM requires to solve an optimization problem at inference time and after each training step. This has the advantage of considering all the logical rules together at the same time at the expense of an increased computational effort. Contrary, in \kenn\ each rule is considered separately from the others, and the second stage is directly integrated inside the model as a differentiable function that can be trained end to end with the base Neural Network.

However, one could argue that with this strategy there could be some contradictory changes when combining multiple clauses with the same predicates. For instance, if the knowledge is composed by the two clauses $c_1: A \lor B$ and $c_2: \lnot B \lor C$, then the summation strategy introduced in Section~\ref{sec:entire_knowledge} would not force the satisfaction of $A \lor C$, which is a logical consequence of the two clauses. Indeed, the effect of the \boost \ could be to increase the value of $B$ when applied to $c_1$ and decrease it when applied on $c_2$. We call this type of situations as \emph{collisions}. 

For simplicity, let us take the assumption that the improvements are provided by $\delta^{w_c}$ instead of $\delta_s^{w_c}$, i.e., they are calculated by using the $argmax$ operator instead of the $softmax$. In the case of a collision, the final change on $A$ and $C$ would be $0$, while the change on $B$ would be the difference of the two clause weights: $w_{c_1} - w_{c_2}$. Therefore, when a collision happens only the stronger clause is taken into consideration, and its effect is lowered by the presence of the other clause.

A question arises: are collisions likely to happen? For now, let us assume that the base NN is a random classifier which extracts the initial predictions from a uniform distribution. The probability of $B$ to be increased by the first clause is $1/2$, which is also the probability of $\lnot B$ to be increased based on $c_2$. However, the probability of a collision is lower than $1/4$, since the two events are not independent. We remind the reader that $B$ is chosen to be increased based on $c_1$ if its truth value $\INN(B)$ is higher than $\INN(A)$, while $\lnot B$ is increased if $1 - \INN(B)$ is greater than $\INN(C)$. Therefore, if $x$ is the value of $\INN(B)$, then the probability of a collision could be expressed as:
$$
\int_{0}^{1} x (1 - x) dx = (\tfrac{1}{2} x^2 - \tfrac{1}{3} x^3) \Big|_0^1 = \tfrac{1}{6}
$$

More in general, let $c_1$ and $c_2$ be two clauses that share a common predicate with opposite sign. Lets $n$ and $m$ be the number of literals of $c_1$ and $c_2$ respectively. The probability of a collision is expressed by the beta function:
$$
B(n,m) = \int_{0}^{1} x^{1-n} (1 - x)^{1-m} dx
$$

In table~\ref{beta_function} we can see that the probability of a collision with different numbers of literals for the two clauses. Notice that the probability becomes quite small with an increasing number of literals. Moreover, these values are calculated from the assumption that the neural network returns random predictions.
If we assume that the base NN performs at least as good as a random classifier, then the values of Table~\ref{beta_function} represent an upper bound for the probability of having a collision.


\begin{table}
	\caption{Beta function values}
	\centering
	\begin{tabular}{l|l|l}
		n & m & B(n,m) \\
		\hline
		2 & 2 & 0.167 \\
		2 & 3 & 0.083 \\
		3 & 3 & 0.033 \\
		3 & 4 & 0.017 \\
		4 & 4 & 0.007 \\
	\end{tabular}
	\label{beta_function}
\end{table}

In general, we could expect better results from RNM in respect to \kenn, but faster training and inference from \kenn. However, as we will see in Section~\ref{sec:collective}, \kenn\ seems to work better in practice. One possible explanation is that in RNM the model is not trained end to end, making more difficult the learning process as compared to \kenn.


\section{Implementation}
\kenn\ has been implemented as a library for Python 3. It is based on TensorFlow 2 and Keras and it is available as an open-source project on Github\footnote{\href{https://github.com/DanieleAlessandro/KENN2}{https://github.com/DanieleAlessandro/KENN2}}.
\section{Evaluation of the model}\label{sec:collective}

In this section, we focus on experiments with relational data. More precisely, \kenn\ was tested on the context of Collective Classification: given a graph,
we are interested in finding a classification for its nodes using both features of the nodes (the objects) and the information coming from the edges of the graph (relations between objects)~\citep{collective_classification}.

In Collective Classification, there are two different learning paradigms: inductive and transductive learning. In inductive learning there are two separates graphs, one for training and the other for testing. On the contrary, in transductive learning there is only one graph that contains nodes both for training and testing. In other words, in inductive learning there are no edges between nodes for training and testing, while in transductive learning there are. Figure~\ref{fig:collective} shows the difference between the two paradigms.

\begin{figure}
	\centering
\newcommand{\mynode}[4]{
    \node[draw, circle, fill=#4, minimum width=1cm] (O#1) at (#2, #3) {};  
}

\newcommand{\myedge}[3]{
    \draw[color=black, ->] (O#1) to [bend left=#3] (O#2);
}

\begin{tikzpicture}[scale=0.45, node distance=1mm, align=left]
    \tikzset{every node}=[font=\fontsize{16}{0}\selectfont];

    \mynode{0}{0}{0}{red!20}
    \mynode{1}{-2}{3}{blue!20}
    \mynode{2}{3}{3}{blue!20}
    \mynode{3}{-1}{-4}{red!20}
    \mynode{4}{4}{-1}{green!20}
    
    \mynode{5}{16}{3}{none}
    \mynode{6}{12}{2}{none}
    \mynode{7}{14}{-1}{none}
    \mynode{8}{20}{0}{none}
    \mynode{9}{17}{-3}{none}

    \myedge{2}{1}{-5}
    \myedge{0}{3}{0}
    \myedge{3}{4}{0}
    \myedge{3}{1}{15}

    \myedge{5}{6}{-3}
    \myedge{8}{7}{0}
    \myedge{7}{5}{0}
    \myedge{8}{5}{-5}
    \myedge{9}{7}{0}

    \node[draw=none] at (9,-8) {Inductive Learning};

    \mynode{t0}{0}{-15}{red!20}
    \mynode{t1}{-2}{-12}{blue!20}
    \mynode{t2}{3}{-12}{blue!20}
    \mynode{t3}{-1}{-19}{red!20}
    \mynode{t4}{4}{-16}{green!20}
    
    \mynode{t5}{16}{-12}{none}
    \mynode{t6}{12}{-13}{none}
    \mynode{t7}{14}{-16}{none}
    \mynode{t8}{20}{-15}{none}
    \mynode{t9}{17}{-18}{none}

    \myedge{t2}{t1}{0}
    \myedge{t0}{t3}{0}
    \myedge{t3}{t4}{0}
    \myedge{t3}{t1}{15}

    \myedge{t5}{t6}{0}
    \myedge{t8}{t7}{0}
    \myedge{t7}{t5}{0}
    \myedge{t8}{t5}{-5}
    \myedge{t9}{t7}{0}

    \myedge{t6}{t2}{-5}
    \myedge{t7}{t2}{-5}
    \myedge{t4}{t7}{5}
    \myedge{t3}{t9}{-5}

    \node[draw=none] at (9,-23) {Transductive Learning};

\end{tikzpicture}
    \caption{The two learning paradigms in Collective Classification. The colors represent the classes of the nodes. White nodes are nodes of the Test Set. In inductive learning, Training and Test sets are two distinct graphs: the network has to learn only from labeled data; in transductive learning, there is a unique graph with both training and test nodes: the network can make use of information coming from the Test Set at training time by considering the additional relations.}
    \label{fig:collective}
\end{figure}


The tests were performed with the goal of obtaining a comparison with other important methods of Neural-Symbolic integration. For this reason, we followed the evaluation methodology of~\citep{RNM}, where the experiments have been carried out on Citeseer dataset~\citep{citeseer} using SBR and RNM. As in~\citep{RNM}, the experiments have been performed on both inductive and transductive paradigms\footnote{Source code of the experiments available at \href{https://github.com/rmazzier/KENN-Citeseer-Experiments}{https://github.com/rmazzier/KENN-Citeseer-Experiments}}. 




\subsection{Citeseer Dataset}
The Citeseer dataset~\citep{citeseer} used in the evaluation is a citation network: the graph's nodes represent documents and the edges represent citations. The nodes' features are bag-of-words vectors, where an entry is zero if the corresponding word of the dictionary is absent in the document, and one if it is present. The classes to be predicted represent possible topics for a document.

The dataset contains 4732 nodes that must be classified in 6 different classes: AG, AI, DB, IR, ML and HCI.
The classification is obtained from the 3703 features of the nodes, with the addition of the information coming from the citations.

\subsection{The Prior Knowledge}

The Prior Knowledge codifies the idea that papers cite works that are related to them. This implies that the topic of a paper is often the same as the paper it cites. The clause
$$
\forall x \ \forall y \ \lnot T(x) \lor \lnot Cite(x,y) \lor T(y)
$$
is instantiated multiple times by substituting the topic $T$ with all the six classes.





\subsection{Experimental setup}
We used as base NN a dense network with three hidden layers, each with 50 hidden nodes and RELU activation function. This setting is the same as~\citep{RNM}, meaning that the results of \kenn\ are directly comparable with the results of SBR and RNM provided by such work since also the used Knowledge is the same.

Notice that the base NN takes as input only the features of the nodes and it does not take into account relations. For this reason, by adding the knowledge, we could expect improvements as compared to the base NN. Notice also that for the base NN model the two paradigms (inductive and transductive) are equivalent: the only
difference between the two is the presence or absence of relations between nodes in the training and test sets, and such relations are not taken into account by the NN. For this reason, we used the same initialization of the NN weights in both inductive and transductive learning (using the same seed). Indeed, the predictions of the NN are the same in the two cases, while models that make use of the knowledge perform much better on the transductive case since they have more information at their disposal.
Indeed, in KENN the relations are considered only on the level of the KE. This is true also for RNM and SBR. In the case of SBR, the difference is that it does not change the basic neural network structure, meaning that even at inference time the learned model does not take into account the citations. To obviate these problems, SBR optimizes the satisfaction of the constraints defined by the knowledge even at test time. This is a strong point in favor of \kenn\ since among the three methods is the only one that does not require to solve an optimization problem during inference. Indeed, one of the main advantages of \kenn\ is scalability.

Notice also that the citations are known a priori (the edges of the graph are given as inputs). Therefore, when applying KE, it is possible to focus only on pairs of documents for which the $Cite(x,y)$ predicate is known to be true. Indeed, all the grounded clauses in the Prior Knowledge are automatically satisfied if $Cite(x,y)$ is false for the specific grounding. This means that clauses are always satisfied when the pair of documents do not cite one another and that KE would not apply any changes for such pairs of objects. For this reason, $s_i^x$, $s_i^y$ and $\B$ were generated using only pairs of objects for which $Cite(x,y)$ is true, reducing both train and test time. 

Finally, since each paper can not be classified with multiple classes, the activation used was the $\softmax$ function. Notice that the results obtained in Section~\ref{preact} were proved with the assumption that the activation is a logistic function. Therefore, there is not theoretic evidence that this approach can work. However, the results obtained give us empirical evidence that \kenn\ can improve the base NN predictions even with $\softmax$ activation.

The training set dimension is changed multiple times to evaluate the efficacy of the three methods on the varying of training data. More precisely, tests have been conducted by selecting 10\%, 25\%, 50\%, 75\% and 90\% of the dataset for training. 

For each of these values the training and evaluation was performed multiple times, each of which with a different split of the dataset in training and test data. More in details, at each run the training set is created by selecting random nodes of the graph, with the contraints that the dataset must be balanced (i.e., each of the 6 classes have the same amount of nodes in the training data).

In~\citep{RNM}, the number of runs were ten for each percentage of training data and the final results were obtained as the mean of accuracy in those ten runs. This was justified by the fact that the choice of nodes in the Training Set has a strong effect on the learned model, and as a consequence on the test accuracy. However, during our experiments, we noticed that the standard deviations of both NN and KENN's results were rather high. As a consequence, when running multiple times the 10 iterations, the results were much different every time. For this reason, differently from~\citep{RNM}, for each training dimension we performed 500 runs instead of 10.


To evaluate the relevance of our results we calculated the p-values for each percentage of training data, where the considered null Hypothesis correponds to the supposition that the distribution of accuracies of the NN is the same as KENN. In other words, we calculated the probability of obtaining our results under the assumption that the additional knowledge injected by KENN is not really helpful. Moreover, for each experiment, we calculated a 95\% confidence interval on the improvements obtained by KENN.


%


\subsection{Results on Inductive Learning}
\label{sec:results_collective_inductive}
Table~\ref{inductive_results} shows the results of \kenn\ on the inductive version of the training. 

\begin{table}[H]
    \caption{Results in terms of accuracy on the inductive variant of the task ordered by the amount of data used for training. The first three columns contains the results reported in~\citep{RNM}.}
    \centering
    \begin{tabular}{c|lll|lll}
        \% training & NN & SBR & RNM & NN & \kenn\\
        \hline
        \rule{0pt}{3ex}\hspace{-0.12cm}
        10 & 0.645 & 0.650 & {\bfseries 0.685} & 0.544 & 0.601 \\
        &&  (+0.005) & (+0.040) && {\bfseries (+0.048)} \\
        \rule{0pt}{3ex}\hspace{-0.12cm}
        25 & 0.674 & 0.682 & {\bfseries 0.709} & 0.629 & 0.671 \\
        &&  (+0.008) & (+0.035) && {\bfseries (+0.041)}\\
        \rule{0pt}{3ex}\hspace{-0.12cm}
        50 & 0.707 & 0.712 & {\bfseries 0.726} & 0.680 & 0.714 \\
        &&  (+0.005) & (+0.019) && {\bfseries (+0.034)} \\
        \rule{0pt}{3ex}\hspace{-0.12cm}
        75 & 0.717 & 0.719 & 0.726 & 0.733 & {\bfseries 0.754} \\
        &&  (+0.002) & (+0.009) && {\bfseries (+0.021)} \\
        \rule{0pt}{3ex}\hspace{-0.12cm}
        90 & 0.723 & 0.726 & 0.732 & 0.759 & {\bfseries 0.768} \\
        &&  (+0.003) & (+0.009) && {\bfseries (+0.010)} \\
        \hline
        \hline
    \end{tabular}
    \label{inductive_results}
\end{table}

Notice that, although the final results are calculated as the mean of multiple runs, there are still some fluctuations and it is not possible to perfectly replicate the results obtained by~\citep{RNM} for the NN. For this reason, in the results table the NN has two columns, one with the results reported in~\citep{RNM}, another with the results obtained in our experiments. Since the results of the NN are slightly different in these two cases, the considered metric is the improvement over the base NN.

\begin{figure*}
    \makebox[\textwidth][c]{
        \includegraphics[scale=0.65]{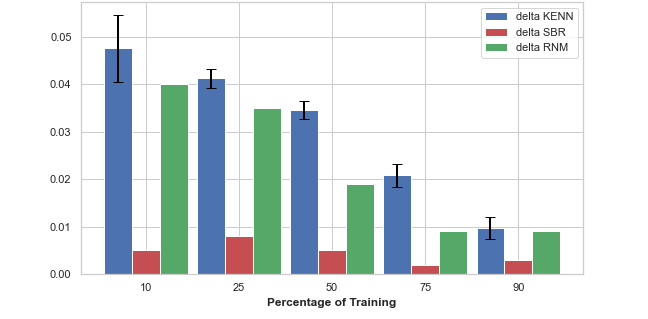}
    }
    \caption{Comparison between KENN, SBR and RNM in terms of improvements over the base NN on inductive learning. Black lines represent 95\% confidence intervals}
    \label{fig:inductive_comparison}
\end{figure*}

Figure~\ref{fig:inductive_comparison} shows the different improvements with respect to the NN model of the three methods at the varying of the percentage of training data. Moreover, the plot shows the confidence intervals for KENN improvements.

Notice that, since we don't have all the data of the RNM and SBR experiments, it is not possible to provide the same analysis also for these methods. Indeed, to calculate p-values and confidence intervals the standard deviation is required. For this reason, while we can be fairly confident about the ability of KENN to incorporate the knowledge into the base NN, the comparison with the other two methods is not as statistically relevant. Additionally, since the number of runs performed in those experiments are only ten for each percentage of training data (as opposed to 500 of the KENN experiments), we could expect much bigger confidence intervals for SBR and RNM improvements. However, we believe that this comparison can still provide us some useful information, in particular when comparing SBR with KENN and RNM. Indeed, the difference between SBR and the other two methods seems too large to depend only on the random choice of the nodes in the Training Set. Notice also that these results are in line with previous results obtained on VRD dataset, where another regularization approach (LTN) was compared with KENN~\citep{kenn}. Indeed, the results obtained in both VRD and Citeseer suggest better performances of model based approaches as compared to the one based on regularization. This could be explained by the ability of model based methods to learn clause weights and, as a consequence, to adapt to partially unsatisfied contraints.


Moreover, the behavior of KENN and RNM is consistent with the simple intuition that, when the training data is scarce, the usage of knowledge should bring higher benefits. That is because the knowledge becomes redundant when the amount of data is large enough and the NN can implicitly learn the given rules directly from the Training Set. Indeed, the two methods produce a large improvements in the smaller datasets and such improvements decrease when adding more samples in the Training Set, converging to similar results when the percentage of training is 90\%.

\begin{figure*}
    \makebox[\textwidth][c]{
        \includegraphics[scale=0.55]{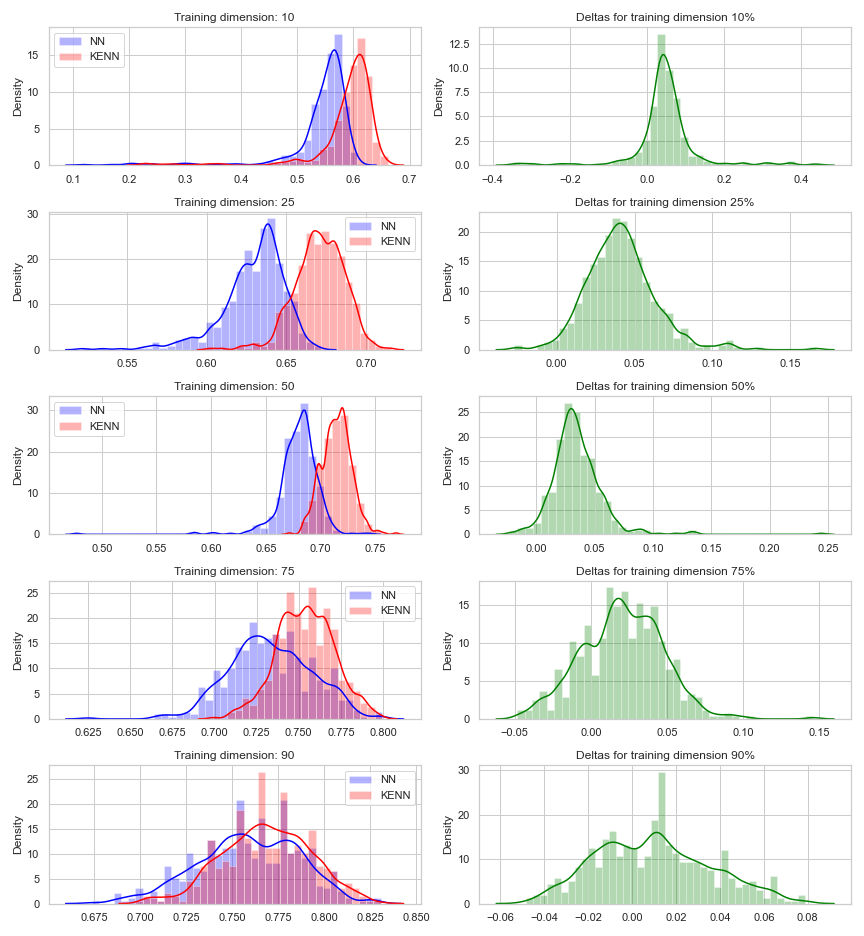}
    }
    \caption{Left: distributions of accuracies of the NN and KENN on 500 runs of Inductive Learning; Right: distributions of the improvements in accuracy produced by the KE.}
    \label{fig:inductive_histograms}
\end{figure*}

Finally, Figure~\ref{fig:inductive_histograms} shows the distribution of the accuracies achieved by the base NN and KENN (on the left side) as well as the distribution of the imporvements obtained by the injection of the logical rules (on the right). Based on these plots, we assumed these distributions to be normal, and we calculated the p-values for each training dimension. Since the number of runs, as well as the obtained improvements, are quite high, the resulting p-values are always close to zero. More in details, the p-values are all in the range from 6.6e-183 (for 25\% of the training data) to 2.9e-09 (in the case of 90\%). For this reason we can safely reject the null Hypothesis. In other words, we are very confident that the improvements given by KENN do not depend on the random choices of the splits of the dataset.

\subsection{Results on Transductive Learning}
\label{sec:results_collective_transductive}
Table~\ref{transductive_results} shows the results of the various models on the transductive learning task. Note that, as mentioned previously, the results of the NN provided by~\citep{RNM} are similar to the case of inductive learning. In our case the results are identical. This is because the seed was chosen to be the same in both inductive and transductive learning. This allow us to have an even better comparison of the behaviours of KENN in inductive and transductive learning. Indeed, while the neural network predictions are exactly the same in the two cases, the improvements obtained by adding the knowledge (in all three methods) are much higher when using the citations between training and test nodes.

\begin{table}
    \caption{Results in terms of accuracy on the transductive variant of the task ordered by the amount of data used for training. The first three columns contains the results reported in~\citep{RNM}.}
    \centering
    \begin{tabular}{c|lll|lll}
        \% training & NN & SBR & RNM & NN & \kenn \\
        \hline
        \rule{0pt}{3ex}\hspace{-0.12cm}
        10 & 0.640 & 0.703 & {\bfseries 0.708} & 0.544 & 0.652\\
        &&  (+0.063) & (+0.068) && {\bfseries (+0.108)}\\
        \rule{0pt}{3ex}\hspace{-0.12cm}
        25 & 0.667 & 0.729 & {\bfseries 0.735} & 0.629 & 0.702 \\
        &&  (+0.062) & (+0.068) && {\bfseries (+0.073)} \\
        \rule{0pt}{3ex}\hspace{-0.12cm}
        50 & 0.695 & 0.747 & {\bfseries 0.753} & 0.680 & 0.744 \\
        &&  (+0.052) & +0.058) && {\bfseries (+0.065)} \\
        \rule{0pt}{3ex}\hspace{-0.12cm}
        75 & 0.708 & 0.764 & 0.766 & 0.733 & {\bfseries 0.788}\\
        &&  (+0.056) & {\bfseries (+0.058)} && (+0.055)\\
        \rule{0pt}{3ex}\hspace{-0.12cm}
        90 & 0.726 & 0.780 & 0.780 & 0.759 & {\bfseries 0.808}\\
        &&  {\bfseries (+0.054)} & {\bfseries (+0.054)} && (+0.049)\\
        \hline
        \hline
    \end{tabular}
    \label{transductive_results}
\end{table}

By looking at Table~\ref{transductive_results} we can notice that, differently from the inductive case, the three methods have similar effects on the accuracies, with the only exception of the 10\% split of the data, where KENN produces a much higher improvement over the base NN. Notice however that such improvement could depend on the much smaller accuracies produces by the NN of our experiments as opposed to the one obtained by~\citep{RNM}.

\begin{figure*}
    \makebox[\textwidth][c]{
        \includegraphics[scale=0.55]{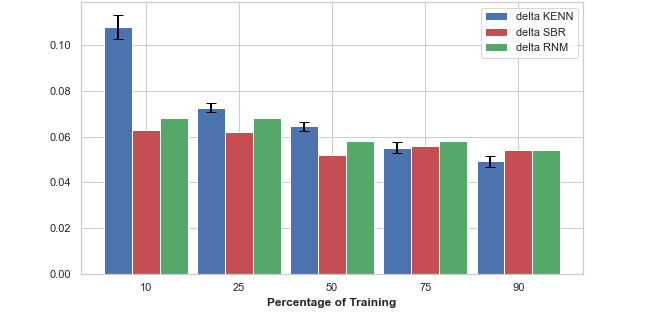}
    }
    \caption{Comparison between KENN, SBR and RNM in terms of improvements over the base NN on transductive learning. Black lines represent 95\% confidence intervals}
    \label{fig:transductive_comparison}
\end{figure*}

As before, Figure~\ref{fig:transductive_comparison} shows a comparison of the improvements in accuracies induced by the three methods. By looking at the confidence intervals, and considering that SBR and RNM's resuts are based on 10 iterations, we can not come to meaningfull conclusions on the comparison of the three methods. However, as before, we can safely conclude that KENN is sensibly improving the performances of the base NN (again, the p-values obtained are close to zero, and the confidence intervals are even smaller than the ones in the inductive case). This can be also noticed from the histograms in Figure~\ref{fig:transductive_histograms}, where the difference between KENN and the NN can easily be seen with naked eyes.

\begin{figure*}
    \makebox[\textwidth][c]{
        \includegraphics[scale=0.55]{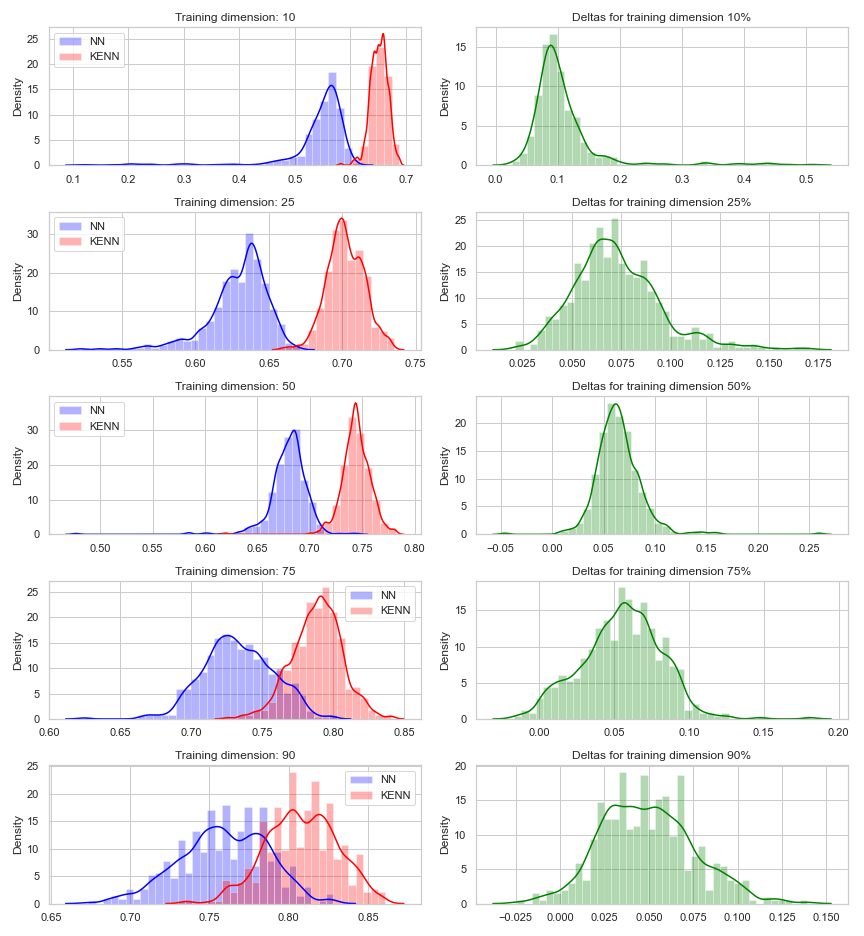}
    }
    \caption{Left: distributions of accuracies of the NN and KENN on 500 runs of Transductive Learning; Right: distributions of the improvements in accuracy produced by the KE.}
    \label{fig:transductive_histograms}
\end{figure*}

As mentioned previously, the three methods seem to have similar behaviours. In particular, we can not see anymore the large difference between KENN/RNM and SBR. To explain this behaviour, we can reason about the properties of transductive learning, where all the methods have at their own disposal informations about relations between training and test nodes. Like for inductive learning, the loss is calculated from the predictions on the training set nodes. However, if clause weights are high enough, the loss depends indirectly also on the values of the test nodes.
To explain this property, let's look at Figure~\ref{fig:bias_kenn}. As before, white nodes represent nodes of the test set. Suppose color red represents the class AI and blue ML. The value of the loss function does not depend directly on the value of node 9, since it is not in the training set. However, if the clause weight $w_{AI}$ of clause $c_{AI}: \forall x \ \forall y \ AI(x) \land Cite(x,y) \to AI(y)$ is high, then setting node 9 to class AI would force node 3 to be classified as AI as well. On the contrary, if we classify node 9 with the class ML, then node 3 would also be classified as ML (supposing high value also for $w_{ML}$).

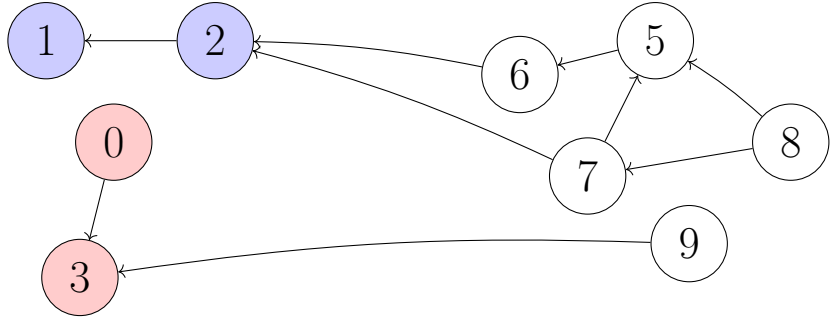
\begin{figure}[H]
	\centering
\newcommand{\mynode}[4]{
    \node[draw, circle, fill=#4, minimum width=1cm] (O#1) at (#2, #3) {#1};
}

\newcommand{\myedge}[3]{
    \draw[color=black, ->] (O#1) to [bend left=#3] (O#2);
}

\begin{tikzpicture}[scale=0.45, node distance=1mm, align=left]
    \tikzset{every node}=[font=\fontsize{16}{0}\selectfont];

    \mynode{0}{0}{-15}{red!20}
    \mynode{1}{-2}{-12}{blue!20}
    \mynode{2}{3}{-12}{blue!20}
    \mynode{3}{-1}{-19}{red!20}
    
    \mynode{5}{16}{-12}{none}
    \mynode{6}{12}{-13}{none}
    \mynode{7}{14}{-16}{none}
    \mynode{8}{20}{-15}{none}
    \mynode{9}{17}{-18}{none}

    \myedge{2}{1}{0}
    \myedge{0}{3}{0}

    \myedge{5}{6}{0}
    \myedge{8}{7}{0}
    \myedge{7}{5}{0}
    \myedge{8}{5}{-5}

    \myedge{6}{2}{-5}
    \myedge{7}{2}{-5}
    \myedge{9}{3}{-5}

\end{tikzpicture}
    \caption{An extreme example of a graph in a transductive learning scenario: the white nodes are available during training, but not their classes. In this simple scenario, it is easy for the learning algorithm to make \kenn\ fit the data by forcing the satisfaction of the constraints and by classifying the test set nodes accordingly.}
    \label{fig:bias_kenn}
\end{figure}

Summarizing, in the simple scenario of Figure \ref{fig:bias_kenn}, it is convenient for the model to learn a high value for $w_{AI}$ and $w_{ML}$. By predicting nodes 9 and 8 as AI and ML respectively, the model will force also the training nodes to be correctly classified. As a consequence, the loss would be reduced and this kind of solution will be selected by the training process. In this case, the KE produces a sort of regularization similar to the one of LTN and SBR, since it rewards solutions that make the constraints satisfied. 

\section{Clauses weights learning}
\label{sec:scatter}

We analyzed the efficacy of \kenn\ on learning the clause weights by analyzing their values after training and comparing them with the level of satisfaction of the constraints on the training data. 
More in details, given a specific topic $k$\footnote{To simplify the notation, here we represent the topics as integers} and the corresponding clause ($\lnot k(x) \lor \lnot Cite(x,y) \lor k(y)$), we define the compliance of the Training Set with the constraint as the fraction of time in which papers of topic $k$ are cited by papers that are also of that topic.

More in details, lets define $G = (V, E, T)$ as the graph containing the ground truth labels of the training set. Here $V$ is the set of nodes (papers), $E$ the sets of edges (citations), and $T: V \to [1,k]$ a function that maps each paper to the correct topic. Given a topic $k$, we define the clause compliance of $k$ as:
$$
C(G, k) = \frac{\sum_{v \in T_k} \sum_{u \in \mathcal{N}(v)} \mathbf{1}(u \in T_k)}
{\sum_{v \in T_k} |\mathcal{N}(v)|}
$$
where $T_k = \{ v | v \in V \land T(v) = k \}$ is the set of nodes of topic $k$, $\mathcal{N}(v)$ is the function that returns the neighbours of node $v \in V$ and $\mathbf{1}(u \in T_k)$ is 1 if $u \in T_k$, 0 otherwise.

Figure~\ref{fig:scatter_plot} shows a series of scatter plots (one for each percentage of training data) representing the correlation between learned clause weights and training data compliance with the clauses. Each point corresponds to a topic in a specific run and its coordinates are given by the compliance value for the topic ($y$ axis) and the learned clause weight ($x$ axis). In total there are 3000 points, each of which corresponds to a specific topic and run in the inductive learning scenario.

The plots suggest a correlation between the two values, in particular with high percentages of training data. This means that, in this specific dataset, \kenn\ seems to be capable of learning the clause weights from the data at its disposal. 

    Moreover, we can notice an increase in variance when the constraints are only partially satisfied. For instance, the constraint applied on topic AI is in general satisfied half of the times and the values of the corresponding clause weight is more sparse as compared with other topics. This behaviour could be explained by the random initialization of the base NN parameters since the impact of a clause depends on the initial predictions of the NN. Indeed, as seen in Section~\ref{preact}, the knowledge is added at the level of the preactivations and it could produce bigger changes when the preactivations are close to zero (see Figure~\ref{fig:logistic} for more details). 

    This means that the CE produces different degree of changes to each node, and the random initialization of the base NN affects the choice of which nodes receive higher modifications. If the ones that are highly modified by the CE correspond to the nodes for which the constraints is satisfied, then the impact of the CE on the final prediction is positive. On the contrary, if the CE produce small changes when the constraint is satisifed, and big changes when it is not, then the CE produces a degradation of the performances and the corresponding clause weight is reduced. In other words, when the constraint is only partially satisfied, the effect of the associated CE is more randomic and, as a consequence, the clause weight distribution is more sparse.




\begin{figure*}
    \makebox[\textwidth][c]{
        \includegraphics[scale=0.45]{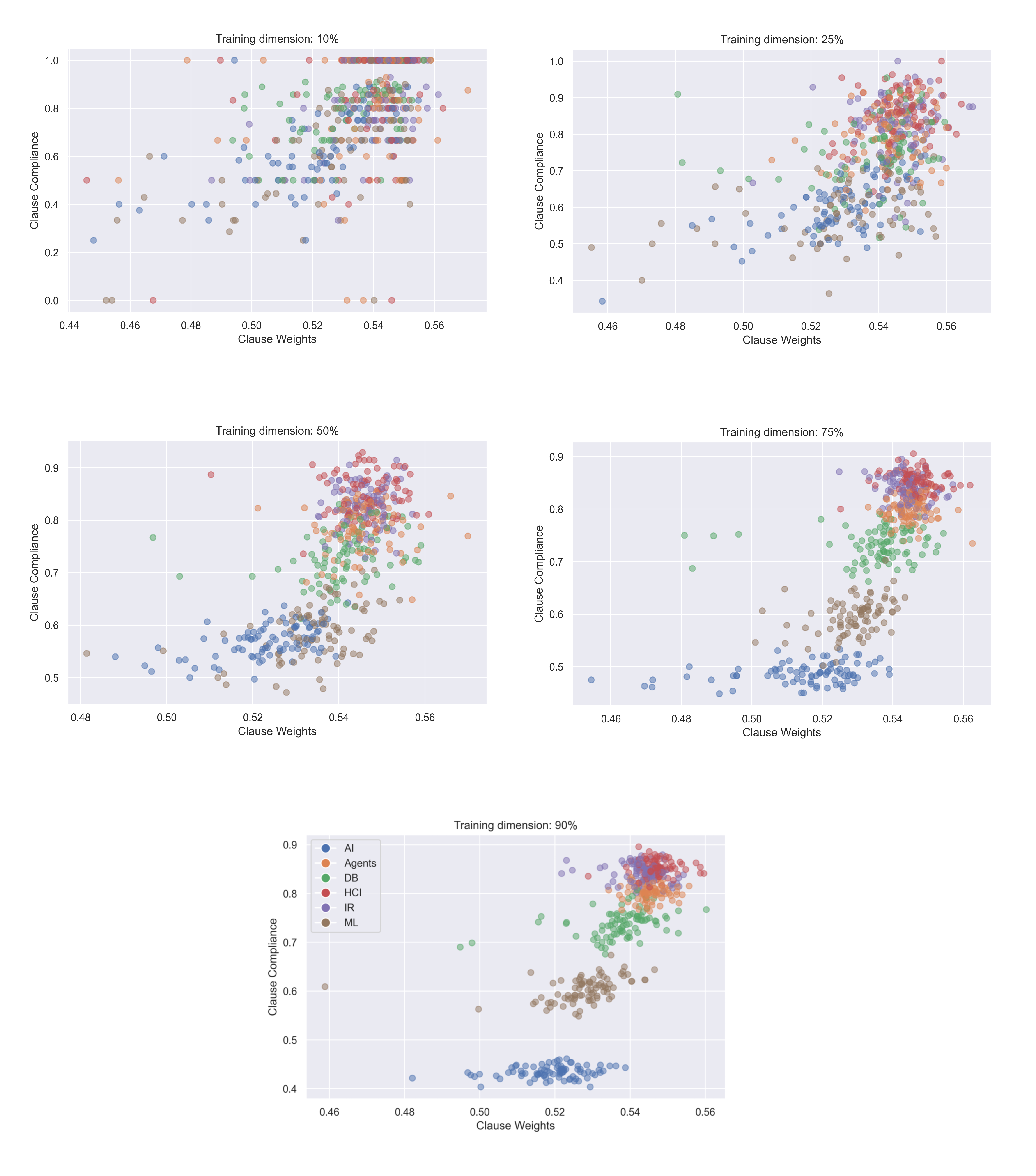}
    }
    \caption{Correlation between learned clause weights and clauses satisfaction on Training Set for the five percentages of training data}
    \label{fig:scatter_plot}
\end{figure*}

\section{Conclusions and future work}\label{sec:conclusions}

We proposed an extension of \kenn, a method for injecting Prior Knowledge expressed as logical clauses into a neural network model, merging in this way the learning capabilities of Neural Networks with the expressivity of First Order Logic. \kenn\ can be used in combination with many different types of neural networks and differs from other methods on the way logic is injected, that is by adding a new differentiable layer (the Knowledge Enhancer) at the top of a NN which provides predictions for the atomic formulas of the logic clauses.

Moreover, it has been introduced the concept of t-conorm boost function (TBF), which is a function that modifies the truth values of the literals of a clause to increase its satisfaction in terms of a fuzzy logic t-conorm. Furthermore, we introduced the concept of minimality of a TBF and provided a formal proof of the minimality of a specific function $\delta^f$ in respect to the G\"{o}del t-conorm. A soft approximation of $\delta^f$ is used inside the Clause Enhancer, a submodule of the KE.

In addition, the KE contains clause weights, which are learnable parameters that represent the strength of the clauses. At the best of our knowledge, \kenn\ and RNM are the only approaches that can merge Neural Networks models and logical Prior Knowledge while learning the clause weights. The two methods work similarly, in the sense that both the methods use a neural network to predict the initial classification which is then updated based on the knowledge. The difference between the two can be summarized in the different choices made for this second step: \kenn\ integrate it directly into the base NN as a differentiable function but it considers each rule separately and integrates the results via a linear combination; RNM instead take into account the entire knowledge at the same time, but it requires to solve an optimization problem (MAP estimation) at inference time and at each training step.

The ability to learn clause weights makes \kenn\ suitable for tasks where the provided knowledge is not totally reliable or when some of the clauses are not hard constraints and it is not known a priori the strength of the constraints encoded in the clauses. This has proven to be a major advantage in previous experiments with VRD Dataset, where the results obtained by \kenn\ outperformed LTN using the same type of knowledge. 
Moreover, the experiments on Citeseer provide other evidence in favour of approaches that integrate the knowledge into the model, since both \kenn\ and RNM obtained higher results than SBR. Finally, the ability of \kenn\ of learning the clause weights have been analyzed in the context of the Citeseer dataset, where a correlation between learned weights and clause satisfaction has been found. However, a more in depth analysis on the learned clause weights is neccessary to better understand such correlation, in particular when working with different datasets and with less satisfied clauses.
Finally, notice that one of the advantages of learning the clause weights is that in principle it should be possible to produce new knowledge by generating random clauses and let \kenn\ validate them by learning the weights. This approach could be the focus of further experiments.

While \kenn\ proved to be able to effectively add the knowledge into the base NN with good results in multiple datasets, the flexibility in terms of usable knowledge is lower as compared to the other methods, in particular because of the absence of the existential quantifier. This could be the goal of future developments. Notice that, if such a change is done under the Close World Assumption, this should be a quite trivial extension to develop. Another missing feature is the possibility to represent functions. In this case, however, it is not a straightforward task. Further investigations are needed in this direction.

Another important contribution is in the solution found to manage relational data: the KE has been developed to work on a matrix where each column represents a unary predicate and the rows the objects of the domain. This means that the KE does not work with relational data. On the contrary, it works under the i.i.d. assumption. Instead of developing a version of the KE for relational data, the graph structure is managed by a pre elaboration step which provides the KE with a matrix in the required format, and a post elaboration step that aggregates the outputs generated by the KE to obtain the final results which are compliant with the relational structure of the domain. The main contribution is on the way these pre and post elaborations steps work: they are represented as queries of a relational database. Indeed, the data structure used to store a graph is very similar to the one of a relational database. However, in this case, such database is implemented in TensorFlow and the queries are implemented as differentiable operators. This strategy was used only in the context of \kenn, but it could be the focus of further investigation. Indeed the idea of a differentiable relational database could provide a very convenient way to create Neural Network architectures for relational domains.

\section{Acknowledgements}
We would like to thank Riccardo Mazzieri for his help on conducting the experiments on Citeseer. Many thanks also to Artur d'Avila Garcez, Alessandro Sperduti, Marco Gori, and Tommaso Carraro for useful suggestions and feedbacks.

\section{Funding}
This work has been partially supported by the EU H2020-ICT-2018-2 RIA project AI4EU  (grant agreement 825619).

\bibliography{references} 

\begin{thebibliography}{47}
\providecommand{\natexlab}[1]{#1}
\providecommand{\url}[1]{\texttt{#1}}
\expandafter\ifx\csname urlstyle\endcsname\relax
  \providecommand{\doi}[1]{doi: #1}\else
  \providecommand{\doi}{doi: \begingroup \urlstyle{rm}\Url}\fi

\bibitem[Bach et~al.(2017)Bach, Broecheler, Huang, and Getoor]{PSL}
Stephen~H. Bach, Matthias Broecheler, Bert Huang, and Lise Getoor.
\newblock Hinge-loss {M}arkov random fields and probabilistic soft logic.
\newblock \emph{Journal of Machine Learning Research}, 18\penalty0
  (109):\penalty0 1--67, 2017.

\bibitem[Bahdanau et~al.(2014)Bahdanau, Cho, and Bengio]{machine_translation}
Dzmitry Bahdanau, Kyunghyun Cho, and Yoshua Bengio.
\newblock Neural machine translation by jointly learning to align and
  translate.
\newblock \emph{arXiv preprint arXiv:1409.0473}, 2014.

\bibitem[Besold et~al.(2017)Besold, d'Avila Garcez, Bader, Bowman, Domingos,
  Hitzler, K{\"{u}}hnberger, Lamb, Lowd, Lima, de~Penning, Pinkas, Poon, and
  Zaverucha]{neural_symbolic}
Tarek~R. Besold, Artur~S. d'Avila Garcez, Sebastian Bader, Howard Bowman,
  Pedro~M. Domingos, Pascal Hitzler, Kai{-}Uwe K{\"{u}}hnberger, Lu{\'{\i}}s~C.
  Lamb, Daniel Lowd, Priscila Machado~Vieira Lima, Leo de~Penning, Gadi Pinkas,
  Hoifung Poon, and Gerson Zaverucha.
\newblock Neural-symbolic learning and reasoning: {A} survey and
  interpretation.
\newblock \emph{CoRR}, abs/1711.03902, 2017.
\newblock URL \url{http://arxiv.org/abs/1711.03902}.

\bibitem[Bianchi and Hitzler(2019)]{LTN_reasoning}
Federico Bianchi and Pascal Hitzler.
\newblock On the capabilities of logic tensor networks for deductive reasoning.
\newblock In \emph{AAAI Spring Symposium: Combining Machine Learning with
  Knowledge Engineering}, 2019.

\bibitem[Campero et~al.(2018)Campero, Pareja, Klinger, Tenenbaum, and
  Riedel]{rule_induction}
Andres Campero, Aldo Pareja, Tim Klinger, Josh Tenenbaum, and Sebastian Riedel.
\newblock Logical rule induction and theory learning using neural theorem
  proving.
\newblock \emph{arXiv preprint arXiv:1809.02193}, 2018.

\bibitem[Cohen(2016)]{tensorlog}
William~W Cohen.
\newblock Tensorlog: A differentiable deductive database.
\newblock \emph{arXiv preprint arXiv:1605.06523}, 2016.

\bibitem[Daniele and Serafini(2019)]{kenn}
Alessandro Daniele and Luciano Serafini.
\newblock Knowledge enhanced neural networks.
\newblock In Abhaya~C. Nayak and Alok Sharma, editors, \emph{PRICAI 2019:
  Trends in Artificial Intelligence}, pages 542--554, Cham, 2019. Springer
  International Publishing.
\newblock ISBN 978-3-030-29908-8.

\bibitem[de~Jong and Sha(2019)]{NTP_do_not}
Michiel de~Jong and Fei Sha.
\newblock Neural theorem provers do not learn rules without exploration.
\newblock \emph{arXiv preprint arXiv:1906.06805}, 2019.

\bibitem[De~Raedt et~al.(2007)De~Raedt, Kimmig, and Toivonen]{problog}
Luc De~Raedt, Angelika Kimmig, and Hannu Toivonen.
\newblock Problog: A probabilistic prolog and its application in link
  discovery.
\newblock In \emph{IJCAI}, volume~7, pages 2462--2467. Hyderabad, 2007.

\bibitem[Detassis et~al.(2020)Detassis, Lombardi, and Milano]{old_dog}
Fabrizio Detassis, Michele Lombardi, and Michela Milano.
\newblock Teaching the old dog new tricks: Supervised learning with
  constraints.
\newblock \emph{arXiv preprint arXiv:2002.10766}, 2020.

\bibitem[Diligenti et~al.(2017)Diligenti, Gori, and Sacc{\`{a}}]{SBR}
Michelangelo Diligenti, Marco Gori, and Claudio Sacc{\`{a}}.
\newblock Semantic-based regularization for learning and inference.
\newblock \emph{Artif. Intell.}, 244:\penalty0 143--165, 2017.

\bibitem[Donadello(2018)]{IvanThesis}
Ivan Donadello.
\newblock \emph{Semantic Image Interpretation - Integration of Numerical Data
  and Logical Knowledge for Cognitive Vision}.
\newblock PhD thesis, Trento Univ., Italy, 2018.

\bibitem[Donadello et~al.(2017)Donadello, Serafini, and d'Avila
  Garcez]{DonadelloSG17}
Ivan Donadello, Luciano Serafini, and Artur~S. d'Avila Garcez.
\newblock Logic tensor networks for semantic image interpretation.
\newblock In Carles Sierra, editor, \emph{Proceedings of the Twenty-Sixth
  International Joint Conference on Artificial Intelligence, {IJCAI} 2017,
  Melbourne, Australia, August 19-25, 2017}, pages 1596--1602. ijcai.org, 2017.
\newblock ISBN 978-0-9992411-0-3.
\newblock \doi{10.24963/ijcai.2017/221}.
\newblock URL \url{https://doi.org/10.24963/ijcai.2017/221}.

\bibitem[Dong et~al.(2019)Dong, Mao, Lin, Wang, Li, and Zhou]{NLM}
Honghua Dong, Jiayuan Mao, Tian Lin, Chong Wang, Lihong Li, and Denny Zhou.
\newblock Neural logic machines.
\newblock \emph{arXiv preprint arXiv:1904.11694}, 2019.

\bibitem[Evans and Grefenstette(2018)]{dilp}
Richard Evans and Edward Grefenstette.
\newblock Learning explanatory rules from noisy data.
\newblock \emph{Journal of Artificial Intelligence Research}, 61:\penalty0
  1--64, 2018.

\bibitem[Fischer et~al.(2019)Fischer, Balunovic, Drachsler-Cohen, Gehr, Zhang,
  and Vechev]{dl2}
Marc Fischer, Mislav Balunovic, Dana Drachsler-Cohen, Timon Gehr, Ce~Zhang, and
  Martin Vechev.
\newblock Dl2: Training and querying neural networks with logic.
\newblock In \emph{International Conference on Machine Learning}, pages
  1931--1941, 2019.

\bibitem[Fran{\c{c}}a et~al.(2014)Fran{\c{c}}a, Zaverucha, and Garcez]{CILP++}
Manoel~VM Fran{\c{c}}a, Gerson Zaverucha, and Artur S~d’Avila Garcez.
\newblock Fast relational learning using bottom clause propositionalization
  with artificial neural networks.
\newblock \emph{Machine learning}, 94\penalty0 (1):\penalty0 81--104, 2014.

\bibitem[Garcez et~al.(2019)Garcez, Gori, Lamb, Serafini, Spranger, and
  Tran]{garcez_nesy}
Artur~d'Avila Garcez, Marco Gori, Luis~C Lamb, Luciano Serafini, Michael
  Spranger, and Son~N Tran.
\newblock Neural-symbolic computing: An effective methodology for principled
  integration of machine learning and reasoning.
\newblock \emph{arXiv preprint arXiv:1905.06088}, 2019.

\bibitem[Garcez and Zaverucha(1999)]{CILP}
Artur S~Avila Garcez and Gerson Zaverucha.
\newblock The connectionist inductive learning and logic programming system.
\newblock \emph{Applied Intelligence}, 11\penalty0 (1):\penalty0 59--77, 1999.

\bibitem[Haykin(1999)]{XOR2}
S~Haykin.
\newblock Neural networks, a comprehensive foundation second edition by
  prentice-hall, 1999.

\bibitem[Hinton et~al.(2012)Hinton, Deng, Yu, Dahl, Mohamed, Jaitly, Senior,
  Vanhoucke, Nguyen, Kingsbury, et~al.]{speech_recognition}
Geoffrey Hinton, Li~Deng, Dong Yu, George Dahl, Abdel-rahman Mohamed, Navdeep
  Jaitly, Andrew Senior, Vincent Vanhoucke, Patrick Nguyen, Brian Kingsbury,
  et~al.
\newblock Deep neural networks for acoustic modeling in speech recognition.
\newblock \emph{IEEE Signal processing magazine}, 29, 2012.

\bibitem[Hu et~al.(2016)Hu, Ma, Liu, Hovy, and Xing]{harnessing}
Zhiting Hu, Xuezhe Ma, Zhengzhong Liu, Eduard~H. Hovy, and Eric~P. Xing.
\newblock Harnessing deep neural networks with logic rules.
\newblock In \emph{Proceedings of the 54th Annual Meeting of the Association
  for Computational Linguistics, {ACL} 2016, August 7-12, 2016, Berlin,
  Germany, Volume 1: Long Papers}. The Association for Computer Linguistics,
  2016.
\newblock ISBN 978-1-945626-00-5.
\newblock URL \url{http://aclweb.org/anthology/P/P16/P16-1228.pdf}.

\bibitem[Jiang and Luo(2019)]{dilp_reinforcment}
Zhengyao Jiang and Shan Luo.
\newblock Neural logic reinforcement learning.
\newblock \emph{arXiv preprint arXiv:1904.10729}, 2019.

\bibitem[Kaur et~al.(2019)Kaur, Kunapuli, Joshi, Kersting, and
  Natarajan]{NN_relational}
Navdeep Kaur, Gautam Kunapuli, Saket Joshi, Kristian Kersting, and Sriraam
  Natarajan.
\newblock Neural networks for relational data.
\newblock In \emph{International Conference on Inductive Logic Programming},
  pages 62--71. Springer, 2019.

\bibitem[Koller et~al.(2007)Koller, Friedman, D{\v{z}}eroski, Sutton, McCallum,
  Pfeffer, Abbeel, Wong, Heckerman, Meek, et~al.]{introduction_SRL}
Daphne Koller, Nir Friedman, Sa{\v{s}}o D{\v{z}}eroski, Charles Sutton, Andrew
  McCallum, Avi Pfeffer, Pieter Abbeel, Ming-Fai Wong, David Heckerman, Chris
  Meek, et~al.
\newblock \emph{Introduction to statistical relational learning}.
\newblock MIT press, 2007.

\bibitem[Krizhevsky et~al.(2012)Krizhevsky, Sutskever, and
  Hinton]{image_classification}
Alex Krizhevsky, Ilya Sutskever, and Geoffrey~E. Hinton.
\newblock Imagenet classification with deep convolutional neural networks.
\newblock In \emph{Proceedings of the 25th International Conference on Neural
  Information Processing Systems - Volume 1}, NIPS'12, pages 1097--1105, USA,
  2012. Curran Associates Inc.
\newblock URL \url{http://dl.acm.org/citation.cfm?id=2999134.2999257}.

\bibitem[Li and Srikumar(2019)]{augmenting}
Tao Li and Vivek Srikumar.
\newblock Augmenting neural networks with first-order logic.
\newblock In \emph{Proceedings of the 57th Annual Meeting of the Association
  for Computational Linguistics}, pages 292--302, Florence, Italy, July 2019.
  Association for Computational Linguistics.
\newblock \doi{10.18653/v1/P19-1028}.
\newblock URL \url{https://www.aclweb.org/anthology/P19-1028}.

\bibitem[Lu et~al.(2016)Lu, Krishna, Bernstein, and Li]{visual2}
Cewu Lu, Ranjay Krishna, Michael~S. Bernstein, and Fei{-}Fei Li.
\newblock Visual relationship detection with language priors.
\newblock In \emph{{ECCV} {(1)}}, volume 9905 of \emph{Lecture Notes in
  Computer Science}, pages 852--869. Springer, 2016.

\bibitem[Lu and Getoor(2003)]{citeseer}
Qing Lu and Lise Getoor.
\newblock Link-based classification.
\newblock In \emph{Proceedings of the Twentieth International Conference on
  International Conference on Machine Learning}, ICML’03, page 496–503.
  AAAI Press, 2003.
\newblock ISBN 1577351894.

\bibitem[Manhaeve et~al.(2018)Manhaeve, Dumancic, Kimmig, Demeester, and
  De~Raedt]{deepproblog}
Robin Manhaeve, Sebastijan Dumancic, Angelika Kimmig, Thomas Demeester, and Luc
  De~Raedt.
\newblock Deepproblog: Neural probabilistic logic programming.
\newblock In \emph{Advances in Neural Information Processing Systems}, pages
  3749--3759, 2018.

\bibitem[Marra et~al.(2020)Marra, Diligenti, Giannini, Gori, and Maggini]{RNM}
Giuseppe Marra, Michelangelo Diligenti, Francesco Giannini, Marco Gori, and
  Marco Maggini.
\newblock Relational neural machines.
\newblock \emph{arXiv preprint arXiv:2002.02193}, 2020.

\bibitem[Minervini and Riedel(2018)]{adversarially}
Pasquale Minervini and Sebastian Riedel.
\newblock Adversarially regularising neural nli models to integrate logical
  background knowledge.
\newblock \emph{arXiv preprint arXiv:1808.08609}, 2018.

\bibitem[Pearl(2014)]{MRF}
Judea Pearl.
\newblock \emph{Probabilistic reasoning in intelligent systems: networks of
  plausible inference}.
\newblock Elsevier, 2014.

\bibitem[Reimann and Schwung(2019)]{NLRL}
Jan~Niclas Reimann and Andreas Schwung.
\newblock Neural logic rule layers.
\newblock \emph{arXiv preprint arXiv:1907.00878}, 2019.

\bibitem[Richardson and Domingos(2006)]{MLN}
Matthew Richardson and Pedro Domingos.
\newblock Markov logic networks.
\newblock \emph{Mach. Learn.}, 62\penalty0 (1-2):\penalty0 107--136, February
  2006.
\newblock ISSN 0885-6125.

\bibitem[Rockt{\"a}schel and Riedel(2016)]{NTP1}
Tim Rockt{\"a}schel and Sebastian Riedel.
\newblock Learning knowledge base inference with neural theorem provers.
\newblock In \emph{Proceedings of the 5th Workshop on Automated Knowledge Base
  Construction}, pages 45--50, 2016.

\bibitem[Rockt{\"a}schel and Riedel(2017)]{NTP2}
Tim Rockt{\"a}schel and Sebastian Riedel.
\newblock End-to-end differentiable proving.
\newblock In \emph{Advances in Neural Information Processing Systems}, pages
  3788--3800, 2017.

\bibitem[Rockt{\"a}schel et~al.(2014)Rockt{\"a}schel, Bosnjak, Singh, and
  Riedel]{low_dimensional}
Tim Rockt{\"a}schel, Matko Bosnjak, Sameer Singh, and Sebastian Riedel.
\newblock Low-dimensional embeddings of logic.
\newblock In \emph{Proceedings of the ACL 2014 Workshop on Semantic Parsing},
  pages 45--49, 2014.

\bibitem[Sen et~al.(2008)Sen, Namata, Bilgic, Getoor, Gallagher, and
  Eliassi-Rad]{collective_classification}
Prithviraj Sen, Galileo~Mark Namata, Mustafa Bilgic, Lise Getoor, Brian
  Gallagher, and Tina Eliassi-Rad.
\newblock Collective classification in network data.
\newblock \emph{AI Magazine}, 29\penalty0 (3):\penalty0 93--106, 2008.

\bibitem[Serafini and d'Avila Garcez(2016)]{LTN}
Luciano Serafini and Artur~S. d'Avila Garcez.
\newblock Logic tensor networks: Deep learning and logical reasoning from data
  and knowledge.
\newblock \emph{CoRR}, abs/1606.04422, 2016.

\bibitem[Touretzky and Pomerleau(1989)]{XOR1}
David~S Touretzky and Dean~A Pomerleau.
\newblock What’s hidden in the hidden layers.
\newblock \emph{Byte}, 14\penalty0 (8):\penalty0 227--233, 1989.

\bibitem[Towell and Shavlik(1994)]{KBANN}
Geoffrey~G. Towell and Jude~W. Shavlik.
\newblock Knowledge-based artificial neural networks.
\newblock \emph{Artif. Intell.}, 70\penalty0 (1-2):\penalty0 119--165, October
  1994.
\newblock ISSN 0004-3702.
\newblock \doi{10.1016/0004-3702(94)90105-8}.
\newblock URL \url{http://dx.doi.org/10.1016/0004-3702(94)90105-8}.

\bibitem[Van~Krieken et~al.(2019)Van~Krieken, Acar, and
  Van~Harmelen]{semi_supervised}
Emile Van~Krieken, Erman Acar, and Frank Van~Harmelen.
\newblock Semi-supervised learning using differentiable reasoning.
\newblock \emph{arXiv preprint arXiv:1908.04700}, 2019.

\bibitem[Wang and Domingos(2008)]{HMLN}
Jue Wang and Pedro Domingos.
\newblock Hybrid markov logic networks.
\newblock In \emph{Proceedings of the 23rd National Conference on Artificial
  Intelligence}, volume~2 of \emph{AAAI'08}, pages 1106--1111. AAAI Press,
  2008.
\newblock ISBN 978-1-57735-368-3.

\bibitem[Wang et~al.(2019)Wang, Donti, Wilder, and Kolter]{satnet}
Po-Wei Wang, Priya~L Donti, Bryan Wilder, and Zico Kolter.
\newblock Satnet: Bridging deep learning and logical reasoning using a
  differentiable satisfiability solver.
\newblock \emph{arXiv preprint arXiv:1905.12149}, 2019.

\bibitem[Xu et~al.(2018)Xu, Zhang, Friedman, Liang, and Van~den
  Broeck]{semantic_loss}
Jingyi Xu, Zilu Zhang, Tal Friedman, Yitao Liang, and Guy Van~den Broeck.
\newblock A semantic loss function for deep learning with symbolic knowledge.
\newblock In Jennifer Dy and Andreas Krause, editors, \emph{Proceedings of the
  35th International Conference on Machine Learning}, volume~80 of
  \emph{Proceedings of Machine Learning Research}, pages 5502--5511,
  Stockholmsmässan, Stockholm Sweden, 10--15 Jul 2018. PMLR.
\newblock URL \url{http://proceedings.mlr.press/v80/xu18h.html}.

\bibitem[Yang et~al.(2017)Yang, Yang, and Cohen]{tensorlog_ILP}
Fan Yang, Zhilin Yang, and William~W Cohen.
\newblock Differentiable learning of logical rules for knowledge base
  reasoning.
\newblock In \emph{Advances in Neural Information Processing Systems}, pages
  2319--2328, 2017.

\end{thebibliography}

\end{document}